\documentclass[nohyperref]{article}

\usepackage{microtype}
\usepackage{graphicx}
\usepackage{subcaption}
\usepackage{bbm}
\usepackage{titletoc}

\usepackage{mymacros}

\usetikzlibrary{automata, positioning, arrows}
\usetikzlibrary{shapes.geometric}

\usepackage{hyperref}

\usepackage[accepted]{icml2023}

\usepackage{natbib}
\usepackage{enumitem}

\usepackage[capitalize,noabbrev]{cleveref}

\icmltitlerunning{Towards Theoretical Understanding of Inverse Reinforcement Learning}

\begin{document}

\setlength{\abovedisplayskip}{4pt}
\setlength{\belowdisplayskip}{4pt}
\setlength{\textfloatsep}{10pt}

\twocolumn[
%\icmltitle{Towards a Theoretical Understanding of Inverse Reinforcement Learning}
\icmltitle{Towards Theoretical Understanding of Inverse Reinforcement Learning}



\icmlsetsymbol{equal}{*}

\begin{icmlauthorlist}
\icmlauthor{Alberto Maria Metelli}{yyy}
\icmlauthor{Filippo Lazzati}{yyy}
\icmlauthor{Marcello Restelli}{yyy}
\end{icmlauthorlist}

\icmlaffiliation{yyy}{Politecnico di Milano, 32, Piazza Leonardo da Vinci, Milan, Italy}

\icmlcorrespondingauthor{Alberto Maria Metelli}{albertomaria.metelli@polimi.it}

\icmlkeywords{Machine Learning, ICML}

\vskip 0.3in
]



\printAffiliationsAndNotice{\icmlEqualContribution} % otherwise use the standard text.

\allowdisplaybreaks[4]

\begin{abstract}
\emph{Inverse reinforcement learning} (IRL) denotes a powerful family of algorithms for recovering a reward function justifying the behavior demonstrated by an expert agent. A well-known limitation of IRL is the \emph{ambiguity} in the choice of the reward function, due to the existence of multiple rewards that explain the observed behavior. This limitation has been recently circumvented by formulating IRL as the problem of estimating the \emph{feasible reward set}, \ie the region of the rewards compatible with the expert's behavior. In this paper, we make a step towards closing the theory gap of IRL in the case of finite-horizon problems with a generative model. We start by formally introducing the problem of estimating the feasible reward set, the corresponding PAC requirement, and discussing the properties of particular classes of rewards. Then, we provide the first minimax lower bound on the sample complexity for the problem of estimating the feasible reward set of order ${\Omega}\left( \frac{H^3SA}{\epsilon^2} \left( \log \left(\frac{1}{\delta}\right) + S \right)\right)$, being $S$ and $A$ the number of states and actions respectively, $H$ the horizon, $\epsilon$ the desired accuracy, and $\delta$ the confidence. We analyze the sample complexity of a uniform sampling strategy (\texttt{US-IRL}), proving a matching upper bound up to logarithmic factors. Finally, we outline several open questions in IRL and propose future research directions.
%Given expert demonstrations, we know that there exist multiple reward functions compatible with them. How many samples do we need for estimating all of them? In this paper, we provide a matching bound of $\Theta\bigl(\frac{|\mathcal{S}||\mathcal{A}|}{(1-\gamma)^2 \epsilon^2}\ln\frac{1}{\delta}\bigl)$ for the sample complexity of estimating the feasible set in the episodic fixed-horizon case with a generative model. Moreover, if we use a forward sampling model, we prove a lower bound of $\Omega\bigl(\frac{|\mathcal{S}||\mathcal{A}|H}{\epsilon^2}\ln\frac{1}{\delta}\bigl)$ episodes and we propose an algorithm POLIRL with almost-matching bound of $O\bigl(\frac{|\mathcal{S}|^2|\mathcal{A}|^2H}{\epsilon^2}\ln\frac{1}{\delta}\bigl)$ episodes.
\end{abstract}

\section{Introduction}
\emph{Inverse reinforcement learning} (IRL) aims at efficiently learning a desired behavior by observing an \emph{expert} agent and inferring their intent encoded in a \emph{reward function} (refer to~\citet{OsaPNBA018, AroraD21, AdamsCB22} for recent surveys on IRL). This abstract setting, that diverges from standard \emph{reinforcement learning}~\citep[RL,][]{sutton_barto}, as the reward function has to be learned, arises in a large variety of real-world tasks. In particular, in a \emph{human-in-the-loop}~\cite{WuXSZM022} scenario, when the expert is represented by a human solving a task, an explicit specification of the reward function representing the human's goal is often unavailable. Experience suggests that humans are uncomfortable when asked to describe their intent and, thus, the underlying reward; while they are much more comfortable providing demonstrations of what is believed to be the right behavior. Indeed, human behavior is usually the product of many, possibly conflicting, objectives.\footnote{In RL, the Sutton's hypothesis \cite{sutton_barto} conjectures that a scalar reward is an adequate notion of goal.}
Succeeding in retrieving a representation of the expert's reward has notable implications. First, we obtain explicit information for understanding the motivations behind the expert's choices (\emph{interpretability}). Second, the reward can be employed in RL to train artificial agents, under shifts in the features of the underlying system (\emph{transferability}).

Since the beginning, the community recognized that the IRL problem is, per se, \emph{ill-posed}, as multiple reward functions are compatible with the expert's behavior~\cite{ng200algorithms}. This ambiguity was heterogeneously addressed by the algorithmic proposals that have followed over the years, which realized in several selection criteria, including maximum margin~\cite{RatliffBZ06}, maximum entropy~\cite{zeng2022maximumlikelihood}, minimum Hessian eigenvalue~\cite{MetelliPR17}. Some of these approaches come with theoretical guarantees on the sample complexity, although according to different performance indexes~\citep[\eg][]{abbeel2004apprenticeship, syed2007game, PirottaR16}.

A promising line of research that aspires to overcome the ambiguity issue has been recently investigated in~\cite{metelli2021provably, lindner2022active}. These works focus on estimating \emph{all} the reward functions compatible with the expert's demonstrated behavior, namely the \emph{feasible rewards}. Remarkably, this viewpoint that focuses on the \emph{feasible reward set}, rather than on \emph{one} reward obtained with a specific selection criterion, as previous works did, circumvents the ambiguity problem, postponing the reward selection and pointing to the expert's intent. Although these works provide sample complexity guarantees in different settings, a rigorous understanding of the inherent complexity of the IRL problem is currently lacking. 

\textbf{Contributions}~~In this paper, we aim at taking a step toward the theoretical understanding of the IRL problem. As in~\cite{metelli2021provably, lindner2022active}, we consider the problem of estimating the feasible reward set. We focus on a \emph{generative model} setting, where the agent can query the environment and the expert in any state, and consider finite-horizon decision problems. The contributions of the paper can be summarized as follows.
\begin{itemize}[noitemsep, leftmargin=*, topsep=0pt]
	\item We propose a novel framework to evaluate the accuracy in recovering the feasible reward set, based on the \emph{Hausdorff metric}~\cite{rockafellar2009variational}. This tool generalizes existing performance indexes. Furthermore, we show that the feasible reward set enjoys a desirable Lipschitz continuity property \wrt the IRL problem (Section~\ref{sec:regul}).
	\item We devise a PAC (Probability Approximately Correct) framework for estimating the feasible reward set, providing the definition of $(\epsilon,\delta)$-PAC IRL algorithm. Then, we investigate the relationships between several performance indexes based on the Hausdorff metric (Section~\ref{sec:pacFw}).
	\item We conceive, based on the provided PAC requirements introduced, a novel sample complexity \emph{lower bound} of order ${\Omega} \left( \frac{H^3SA}{\epsilon^2} \left( \log\left(\frac{1}{\delta}\right) + S \right)\right)$. This represents the most significant contribution and, to the best of our knowledge, it is the first lower bound that values the importance of the relevant features of the IRL problem. From a technical perspective, the lower bound construction merges new proof ideas with reworks of existing techniques (Section~\ref{sec:lb}).
	\item We analyze a uniform sampling exploration strategy (UniformSampling-IRL, \texttt{US-IRL}) showing that, in the generative model setting, it matches the lower bound up to logarithmic factors (Section~\ref{sec:algorithm}). 
\end{itemize}
The complete proofs of the results presented in the main paper are reported in Appendix~\ref{apx:proofs}.

\section{Preliminaries}
In this section, we provide the background that will be employed in the subsequent sections.

\textbf{Mathematical Background}~~
Let $a,b \in \Nat$ with $a \le b$, we denote with $\dsb{a,b} \coloneqq \{a,\dots,b\}$ and with $\dsb{a} \coloneqq \dsb{1,a}$. Let $\Xs$ be a set, we denote with $\Delta^{\Xs}$ the set of probability measures over $\Xs$. Let $\Ys$ be a set, we denote with $\Delta_{\Ys}^{\Xs}$ the set of functions with signature $\Ys \rightarrow \Delta^{\Xs}$.
Let $(\Xs,d)$ be a (pre)metric space, where $\Xs$ is a set and $d : \Xs\times \Xs \rightarrow [0,+\infty]$ is a (pre)metric.\footnote{A \emph{premetric} $d$ satisfies the axioms: $d(x,x') \ge 0$ and $d(x,x) = 0$ for all $x,x' \in \Xs$. Any \emph{metric} is clearly a premetric.} Let $\Ys,\Ys' \subseteq \Xs$ be non-empty sets, we define the \emph{Hausdorff (pre)metric}~\cite{rockafellar2009variational} $\mathcal{H}_d: 2^{\Xs} \times 2^{\Xs} \rightarrow [0,+\infty]^{}$ between $\Ys$ and $\Ys'$ induced by the (pre)metric $d$ as follows:
{\thinmuskip=1mu
\medmuskip=1mu
\thickmuskip=1mu
\begin{equation}\label{eq:hausdorff}\resizebox{.9\linewidth}{!}{$\displaystyle
	\mathcal{H}_d(\Ys,\Ys') \coloneqq \max \left\{  \sup_{y \in \Ys} \inf_{y' \in \Ys'} d(y,y'),  \sup_{y' \in \Ys'} \inf_{y \in \Ys} d(y,y') \right\}.$}
\end{equation}}%

\textbf{Markov Decision Processes without Reward}~~
A time-inhomogeneous finite-horizon \emph{Markov decision process without reward} (\MDPRtext) is defined as a 4-tuple $\MDPR = (\Ss,\As,\pb,H)$ where $\mathcal{S}$ is a finite state space ($S = |\Ss|$), $\mathcal{A}$ is a finite action space ($A = |\As|$), $\pb=(p_h)_{h \in \dsb{H}}$ is the transition model where for every stage $h \in \dsb{H}$ we have $p_h \in\Delta_{\mathcal{S}\times\mathcal{A}}^\mathcal{S}$, and $H \in \Nat$ is the horizon. An \MDPRtext is time-homogeneous if, for every stage $h \in \dsb{H-1}$, we have $p_{h} = p_{h+1}$ a.s.; in such a case, we denote the transition model with the symbol $p$ only. A time-inhomogeneous reward function is defined as $\rb = (r_h)_{h \in \dsb{H}}$, where for every stage $h \in \dsb{H}$ we have $r_h: \SAs\rightarrow [-1,1]$.\footnote{For the sake of simplicity and w.l.o.g., we restrict to reward functions bounded by $1$ in absolute value.} A \emph{Markov decision process}~\citep[MDP,][]{puterman2014markov} is obtained by pairing an \MDPRtext $\mathcal{M}$ with a reward function $r$. The agent's behavior is modeled with a time-inhomogeneous policy $\pib=(\pi_h)_{h \in \dsb{H}}$ where for every stage $h \in \dsb{H}$, we have $\pi_h \in \Delta_{\Ss}^{\As}$. Let $f \in \Reals^{\Ss}$ and $g \in \Reals^{\SAs}$, we denote with $p_h f(s,a) =  \sum_{s' \in \Ss} p_h(s'|s,a) f(s')$  and with  $\pi_h g(s) = \sum_{a \in \As} \pi_h(a|s) g(s,a)$ the expectation operators \wrt the transition model and the policy, respectively.

\textbf{Value Functions and Optimality}~~
Given an \MDPRtext $\MDPR$, a policy $\pib$, and a reward function $r$, the \emph{Q-function} $\qb^{\pi}(\cdot;r) = (Q_h^{\pi}(\cdot;r))_{h \in \dsb{H}}$ induced by $r$ represents the expected sum of rewards collected starting from $(s,a,h) \in \SAHs$ and following policy $\pi$ thereafter: 
\begin{align*}
	Q_h^{\pi}(s,a;r) \coloneqq \E_{(\mathcal{M}, \pib)} \left[ \sum_{l=h}^H r_l(s_l,a_l) | s_h=s, a_h=a \right],
\end{align*}
where $\E_{(\mathcal{M}, \pib)}$ denotes the expectation \wrt $\mathcal{M}$ and $\pib$, \ie $a_{h} \sim \pi_h(\cdot|s_h)$ and $s_{h+1} \sim p_h(\cdot|s_h,a_h)$ for every stage $h \in \dsb{h,H}$. The Q-function fulfills the Bellman equations~\cite{puterman2014markov} for every $(s,a,h) \in \SAs\times \dsb{H}$:
\begin{align*}
& Q_h^{\pi}(s,a;r) = r_h(s,a) + p_h V_{h+1}^{\pi}(s,a;r), \\
& V_{h}^{\pi}(s;r) = \pi_h Q_h^{\pi}(s;r) \;\;\;\; \text{ and } \;\;\;\;  V_{H+1}^{\pi}(s;r) = 0,
\end{align*}
where $\vb^{\pi}(\cdot;r) = (V_h^{\pi}(\cdot;r))_{h \in \dsb{H}}$ is the \emph{V-function}. The \emph{advantage function} $A_h^{\pi}(s,a;r) = Q_h^{\pib}(s,a;r) - V_h^{\pi}(s;r)$ represents the relative gain of playing action $a \in \As$ rather than following policy $\pi$ in the state-stage pair $(s,h)$. A policy $\pi^*$ is \emph{optimal} if it has non-positive advantage everywhere, \ie $A_h^{\pi^*}(s,a;r) \le 0$ for every $(s,a,h) \in \SAs\times\dsb{H}$. The Q- and V-functions of an optimal policy are denoted with $Q^*_h(s,a;r)$ and $V^*_h(s;r)$.

\textbf{Inverse Reinforcement Learning}~~
An \emph{inverse reinforcement learning} problem~\citep[IRL,][]{ng200algorithms} is defined as a pair $(\mathcal{M},\pi^E)$, where $\mathcal{M}$ is an \MDPRtext  and $\pi^E $ is an \emph{expert's policy}. Informally, solving an IRL problem consists in finding a reward function $(r_h)_{h \in \dsb{H}}$ making $\pi^E$ optimal for the \MDPRtext $\mathcal{M} $ paired with reward function $\rb$. Any reward function fulfilling this condition is called \emph{feasible} and the set of all such reward functions is called \emph{feasible reward set}~\cite{metelli2021provably, lindner2022active}, defined as:
%
%In~\cite{metelli2021provably, lindner2022active}, any reward function fulfilling this condition is called \emph{feasible} and the set of all such reward functions is called \emph{feasible reward set} that can be formally defined as:
{\thinmuskip=2mu
\medmuskip=2mu
\thickmuskip=2mu
\begin{equation}\label{eq:rSet}
\resizebox{.9\linewidth}{!}{$\displaystyle
\begin{aligned}
	\rSet_{(\mathcal{M},\pi^E)} & \coloneqq \Big\{  (r_h)_{h \in \dsb{H}} \, \Big\rvert\, \forall h \in \dsb{H} \,:\,  r_h : \SAs \rightarrow [-1,1]  \\
	& \wedge  \forall (s,a,h) \in \SAHs \,:\, A_h^{\pi^E}(s,a;r) \le 0 \Big\}. 
\end{aligned}$
}
\end{equation}
}%
We will omit the subscript $(\mathcal{M},\pi^E)$ whenever clear from the context.% and write $\mathcal{R}$.

\textbf{Empirical MDP and Empirical Expert's Policy}~~
{\thinmuskip=1mu
\medmuskip=1mu
\thickmuskip=1mu
Let $D = \{(s_l,a_l,h_l,s_l',a_l^E)\}_{l \in \dsb{t}}$ be a dataset of $t \in \Nat$ tuples, where for every $l \in \dsb{t}$, we have $s_{l}' \sim p_{h_l}(\cdot|s_l,a_l)$ and  $a_l^E \sim \pi^E_{h_l}(\cdot|s_l)$. We introduce the counts for every $(s,a,h) \in \SAHs$: $n_h^t(s,a,s') \coloneqq  \sum_{l=1}^t \indic\{(s_l,a_l,h_l,s_l') = (s,a,h,s')\}$, $n_h^t(s,a) \coloneqq \sum_{s' \in \Ss} n_h^t(s,a,s')$, $n_h^t(s) \coloneqq \sum_{a \in \As} n_h^t(s,a)$, and $n_h^{t,E}(s,a) \coloneqq \sum_{l =1}^t   \indic\{(s_l,a_l^E) = (s,a)\}$. These quantities allow  defining the \emph{empirical transition model} $\widehat{p}^t = (\widehat{p}_h^t)_{h \in \dsb{H}}$ and \emph{empirical expert's policy} $\widehat{\pi}^{t,E} = (\pi_h^{t,E})_{h \in \dsb{H}}$  as follows:}%
\begin{equation}\label{eq:empiricalMDP}
\begin{aligned}
&\widehat{p}_h^t(s'|s,a) \coloneqq \begin{cases}
									\frac{n_h^t(s,a,s')}{n_h^t(s,a)} & \text{if } n_h^t(s,a) > 0 \\
									\frac{1}{S} & \text{otherwise}								\end{cases},\\
&\widehat{\pi}^{E,t}_h(a|s) \coloneqq \begin{cases}
	\frac{n_h^{E,t}(s,a)}{n_h^t(s)} & \text{if } n_h^t(s)> 0 \\
	\frac{1}{A} & \text{otherwise}
\end{cases}.
\end{aligned}
\end{equation}
In the time-homogeneous case, we simply merge the samples collected at different stages $h \in \dsb{H}$.
We denote with $(\widehat{\mathcal{M}}^t,\widehat{\pi}^{E,t})$ the \emph{empirical IRL} problem, where $\widehat{\mathcal{M}}^t = (\Ss,\As,\widehat{p}^t,H)$ the empirical \MDPRtext induced by $\widehat{p}^t$. Finally, we denote with $\widehat{\mathcal{R}}^t \coloneqq \mathcal{R}_{(\widehat{\mathcal{M}}^t,\widehat{\pi}^{E,t})}$ the feasible reward set induced $(\widehat{\mathcal{M}}^t,\widehat{\pi}^{E,t})$. We will omit the superscript $t$, whenever clear from the context and write $\widehat{\mathcal{R}}$.

% We will denote with $\widehat{Q}_{h}^{\pi}(s,a)$, $\widehat{V}_{h}^{\pi}(s)$, and $\widehat{A}_{h}^{\pi}(s,a)$ the value functions in the empirical MDP  $\widehat{\mathcal{M}}$.

\section{Lipschitz Framework for IRL}\label{sec:regul}
In this section, we analyze the regularity properties of the feasible reward set in terms of the Lipschitz continuity \wrt the IRL problem. To make the idea more concrete, suppose that $\mathcal{R}$ is the feasible reward set obtained from the IRL problem $(\mathcal{M},\pi^E)$ and that $\widehat{\mathcal{R}}$ is obtained with a different IRL problem $(\widehat{\mathcal{M}},\widehat{\pi}^E)$, which we can think to as an empirical version of $(\mathcal{M},\pi^E)$, with an estimated transition model $\widehat{p}$ replacing the true model $p$. Intuitively, to have any learning guarantee, \quotes{similar} IRL problems (${p} \approx \widehat{p}$ and $\pi^E \approx \widehat{\pi}^E$) should lead to \quotes{similar} feasible reward sets (${\mathcal{R}} \approx \widehat{\mathcal{R}}$).\footnote{If not, any arbitrary accurate estimate $(\widehat{p},\widehat{\pi}^E)$ of $(p,\pi^E)$, may induce feasible sets $\widehat{\mathcal{R}}$ and $\mathcal{R}$ with finite non-zero dissimilarity.} 

To formally define a Lipschitz framework, we need to select a (pre)metric for evaluating dissimilarities between feasible reward sets and IRL problems. While we defer the presentation of the (pre)metric for the IRL problems to Section~\ref{sec:reg}, where it will emerge naturally, for the feasible reward sets, we employ the \emph{Hausdorff (pre)metric} $\mathcal{H}_d(\mathcal{R},\widehat{\mathcal{R}})$ (Equation~\ref{eq:hausdorff}), induced by a (pre)metric $d(r,\widehat{r}) $ used to evaluate the dissimilarity between individual reward functions $r \in \mathcal{R}$ and $\widehat{r} \in \widehat{\mathcal{R}}$.
%
%. Let $\mathcal{R}$ and $\widehat{\mathcal{R}}$ be two feasible reward sets and let $r \in \mathcal{R}$ and $\widehat{r} \in \widehat{\mathcal{R}}$, we consider a (pre)metric $d(r,\widehat{r}) $ used to evaluate the dissimilarity between individual reward functions. The dissimilarity between the feasible reward sets $\mathcal{R}$ and $\widehat{\mathcal{R}}$ is given by the Hausdorff (pre)metric $\mathcal{H}_d(\mathcal{R},\widehat{\mathcal{R}})$ induced by $d$. 
With this choice, two feasible reward sets are \quotes{similar} if every reward $r \in \mathcal{R}$ is \quotes{similar} to some reward $\widehat{r} \in \widehat{\mathcal{R}}$ in terms of the (pre)metric $d$. In the next sections, we employ as $d$ the metric induced by the $L_\infty$-norm between the reward functions $r \in \mathcal{R}$ and $\widehat{r} \in \widehat{\mathcal{R}}$:\footnote{We discuss other choices of $d$ in Section~\ref{sec:pacFw}.}
\begin{align}\label{eq:metricGen}
	d^{\text{G}}(r,\widehat{r}) \coloneqq \max_{(s,a,h) \in \SAHs} \left| r_h(s,a) - \widehat{r}_h(s,a) \right|,
\end{align}
%
%
% Being sets, we decide to use the \emph{Hausdorff distance} to
%
%To this end, we consider a (pre)metric $d(r,\widehat{r}) $ used to evaluate the dissimilarity between two individual reward functions $r$ and $\widehat{r}$ and let $\mathcal{H}_d(\mathcal{R},\widehat{\mathcal{R}})$ be the Hasudorff (pre)metric, induced by $d$, that quantifies the dissimilarity between two feasible reward sets $\mathcal{R}$ and $\widehat{\mathcal{R}}$. 
%
%$\widehat{\mathcal{R}}$ as a function of the IRL problem $(\mathcal{M},\pi^E)$, \ie whenever ${p} \approx \widehat{p}$ and $\pi^E \approx \widehat{\pi}^E$, it must follow that $\widehat{\mathcal{R}} \approx \widehat{\mathcal{R}}$.\footnote{If this is not the case, any arbitrary accurate estimate $\widehat{p}$ and $\widehat{\pi}^E$ of $p$ and $\pi^E$, might induce feasible reward sets $\widehat{\mathcal{R}}$ and $\mathcal{R}$ with non vanishing dissimilarity.}
where $\text{G}$ stands for \quotes{generative}. In Section~\ref{sec:reg}, we prove that the Lipschitz continuity is fulfilled when no restrictions on the reward function are enforced (besides boundedness in $[-1,1]$). 
%and we formalize the (pre)metric for the IRL problem. 
Then, in Section~\ref{sec:nonReg}, we show that, when further restrictions on the viable rewards are required (\eg state-only reward), such a regularity property no longer holds.

\subsection{Lipschitz Continuous Feasible Reward Sets}\label{sec:reg}
In order to prove the Lipschitz continuity property, we use the \emph{explicit} form of the feasible reward sets introduced in~\cite{metelli2021provably} and extended by~\cite{lindner2022active} for the finite-horizon case, that we report below.

\begin{restatable}[Lemma 4 of \citet{lindner2022active}]{lemma}{explicitRew}\label{lemma:explicitRew}
A reward function $r = (r_h)_{h \in \dsb{H}}$ is feasible for the IRL problem $(\mathcal{M},\pi^E)$ if and only if there exist two functions $(A_h,V_h)_{h \in \dsb{H}}$ where for every $h \in \dsb{H}$ we have $A_h : \SAs \rightarrow \mathbb{R}_{\ge 0}$, $V_h : \SAs \rightarrow \Reals$, and $V_{H+1}=0$, such that for every $(s,a,h) \in \SAHs$ it holds that:
{\thinmuskip=2mu
\medmuskip=2mu
\thickmuskip=2mu
\begin{align*}
	r_h(s,a) = -A_h(s,a) \indic_{\{\pi^E_h(a|s) = 0\}} + V_h(s) - p_h V_{h+1}(s,a).
\end{align*}
}%
Furthermore, if $|r_h(s,a)| \le 1$, if follows that $|V_h(s)| \le H-h+1$ and $A_h(s,a) \le H-h+1$.
\end{restatable}

A form of regularity of the feasible reward set was already studied in Theorem of 3.1 of \citet{metelli2021provably} and in Theorem 5 of \citet{lindner2022active}, providing an \emph{error propagation} analysis. These results are based on showing the existence of a \emph{particular} reward $\widetilde{r}$ feasible for the IRL problem $(\widehat{\mathcal{M}}, \widehat{\pi}^E)$, whose distance from the original reward function $r \in \mathcal{R}$ is bounded by a dissimilarity term between $(\mathcal{M},\pi^E)$ and $(\widehat{\mathcal{M}},\widehat{\pi}^E)$. Unfortunately, such a reward $\widetilde{r}$ is not guaranteed to be bounded in $[-1,1]$ even when the original reward $r$ is (and, thus, it might be $\widetilde{r} \not\in \widehat{R}$ according to Equation~\ref{eq:rSet}).\footnote{We illustrate in Fact~\ref{fact:largeRew} an example of this phenomenon.} In Lemma~\ref{lemma:errorProp}, with a modified construction, we show the existence of another \emph{particular} feasible reward $\widehat{r}$  bounded in $[-1,1]$ (and, thus, $\widehat{r} \in \widehat{\mathcal{R}}$). From this, the Lipschitz continuity of the feasible reward sets follows.

\begin{restatable}[Lipschitz Continuity]{thr}{errorPropA}\label{thr:errorProp2}
Let $\mathcal{R}$ and $\widehat{\mathcal{R}}$ be the feasible reward sets of the IRL problems $(\mathcal{M},\pi^E)$ and $(\widehat{\mathcal{M}},\widehat{\pi}^E)$. Then, it holds that:\footnote{This implies the standard Lipschitz continuity, by simply bounding $ \frac{2 \rho^{\text{G}}((\mathcal{M},\pi^E),(\widehat{\mathcal{M}},\widehat{\pi}^E))}{1+\rho^{\text{G}}((\mathcal{M},\pi^E),(\widehat{\mathcal{M}},\widehat{\pi}^E))} \le  2 \rho^{\text{G}}((\mathcal{M},\pi^E),(\widehat{\mathcal{M}},\widehat{\pi}^E))$.} 
\begin{align}\label{eq:LipHaus}
	\mathcal{H}_{d^{\text{G}}}(\mathcal{R},\widehat{\mathcal{R}}) \le \frac{2 \rho^{\text{G}}((\mathcal{M},\pi^E),(\widehat{\mathcal{M}},\widehat{\pi}^E))}{1+\rho^{\text{G}}((\mathcal{M},\pi^E),(\widehat{\mathcal{M}},\widehat{\pi}^E))},
\end{align}
where $\rho^{\text{G}}(\cdot,\cdot)$ is a (pre)metric between IRL problems, defined as:
{\thinmuskip=1mu
\medmuskip=1mu
\thickmuskip=1mu
\begin{equation*}\resizebox{8cm}{!}{$\displaystyle
\begin{aligned}
	\rho^{\text{G}}& ((\mathcal{M},\pi^E),(\widehat{\mathcal{M}},\widehat{\pi}^E))  \coloneqq  \max_{(s,a,h)\in \SAHs} (H-h+1) \\
	& \times \left( \left| \indic_{\{\pi^E_h(a|s) = 0\}} - \indic_{\{\widehat{\pi}^E_h(a|s) = 0\}} \right| + \left\| p_h(\cdot|s,a) - \widehat{p}_h(\cdot|s,a) \right\|_1  \right).
\end{aligned}$}
\end{equation*}}%
\end{restatable}

Some observations are in order. First, the function $\rho^{\text{G}}$ is indeed a (pre)metric since it is non-negative and takes value $0$ when the IRL problems coincide. Second, as supported by intuition, $\rho^{\text{G}}$  is composed of two terms related to the estimation of the expert's policy and of the transition model. While for the transition model, the dissimilarity is formalized by the $L_1$-norm distance $\left\| p_h(\cdot|s,a) - \widehat{p}_h(\cdot|s,a) \right\|_1 $, for the policy, the resulting term deserves some comments. Indeed, the dissimilarity $| \indic_{\{\pi^E_h(a|s) = 0\}} - \indic_{\{\widehat{\pi}^E_h(a|s) = 0\}} |$ highlights that what matters is \emph{whether an action $a \in \As$ is played by the expert} and not the corresponding probability $\pi^E_h(a|s)$. Indeed, the expert's policy plays an action (with any non-zero probability) only if it is an optimal action.

 \subsection{Non-Lipschitz Continuous Feasible Reward Sets}\label{sec:nonReg}
In this section, we illustrate three cases of feasible reward sets restrictions that turn out not to fulfill the condition of Theorem~\ref{thr:errorProp2}. These examples consider three conditions commonly enforced in the literature: state-only reward function $r_h(s)$ (Example~\ref{ex:stateOnly}), time-homogeneous reward function $r(s,a)$ (Example~\ref{ex:timeHomo}), and $\beta$-margin reward function (Example~\ref{ex:betaSep}). We present counter-examples in which in front of $\epsilon$-close transition models, the induced feasible sets are far apart by a constant independent of $\epsilon$. For space reasons, we report the complete derivation in Appendix~\ref{apx:examples}.

\begin{figure}
\begin{subfigure}{\linewidth}
\centering
\begin{tikzpicture}[node distance=4cm]
\node[state] (s0) {$s_{0}$};
\node[draw=none, below right= .3cm and 1cm of s0, fill=black] (s0+) {};
\node[draw=none, above right= .3cm and 1cm of s0, fill=black] (s0-) {};
\node[state, below right= .6cm and 3cm of s0] (s+) {$s_{-}$};
\node[state, above right= .6cm and 3cm of s0] (s-) {$s_{+}$};

\draw (s0) edge[-, solid, below] node{$a_1$} (s0+);
\draw (s0) edge[-, solid, above] node{$a_2$} (s0-);

\draw (s0+) edge[->, solid, below] node{} (s+);
\draw (s0-) edge[->, solid, above] node{$\scriptstyle 1/2$} (s-);
\draw (s0+) edge[->, solid, below] node[pos=0.8]{} (s-);
\draw (s0-) edge[->, solid, above right] node[pos=0.75]{$\scriptstyle 1/2$} (s+);

\draw (s+) edge[->, loop right, solid, double] node{$\scriptstyle 1$ } (s0);
\draw (s-) edge[->, loop right, solid, double] node{$\scriptstyle 1$ } (s0);
\end{tikzpicture}
\caption{}\label{fig:counterEx1}
\end{subfigure}
\begin{subfigure}{\linewidth}
\centering
\begin{tikzpicture}[node distance=4cm]
\node[state] (s0) {$s_{0}$};
\node[draw=none, below right= .3cm and 1.5cm of s0, fill=black] (s0+) {};
\node[draw=none, above right= .3cm and 1.5cm of s0, fill=black] (s0-) {};
\node[state, right= 3cm of s0] (s+) {$s_{1}$};
%\node[state, above right= .6cm and 3cm of s0] (s-) {$s_{-}$};

\draw (s0) edge[-, solid, below] node{$a_1$} (s0+);
\draw (s0) edge[-, solid, above] node{$a_2$} (s0-);

\draw (s0+) edge[->, solid, below] node{} (s+);
\draw (s0-) edge[->, solid, above, out=90, in=90] node{$\scriptstyle 1/2$} (s0);
\draw (s0+) edge[->, solid, below, out=-90, in=-90] node[pos=0.8]{} (s0);
\draw (s0-) edge[->, solid, above right] node[pos=0.75]{$\scriptstyle 1/2$} (s+);

\draw (s+) edge[->, loop right, solid, double] node{$\scriptstyle 1$ } (s0);
%\draw (s-) edge[->, loop right, solid, double] node{$1$ } (s0);
\end{tikzpicture}
\caption{}\label{fig:counterEx2}
\end{subfigure}
\caption{The \MDPRtext employed in the examples of Section~\ref{sec:nonReg}. \begin{tikzpicture}
\node[ draw=none] (a) {}; 
\node[ draw=none, right= .5cm of a] (b) {}; 
\draw (a) edge[->, double] (b);
\end{tikzpicture} denotes a transition executed for multiple actions.}\label{fig:counterEx}
\end{figure}

\begin{restatable}[State-only reward $r_h(s)$]{example}{exampleA}\label{ex:stateOnly}
State-only reward functions have been widely considered in many IRL approaches~\citep[\eg][]{ng200algorithms,abbeel2004apprenticeship, syed2007game, komanduru2019correctness}.  We formalize the state-only feasible reward set as follows:
\begin{align*}
	\mathcal{R}_{\text{state}} = \mathcal{R} \cap \{ \forall (s,a,a',h) \,:\, r_{h}(s,a) = r_h(s,a') \}.
\end{align*}%
{\thinmuskip=2mu
\medmuskip=2mu
\thickmuskip=2mu%
Consider the \MDPRtext of Figure~\ref{fig:counterEx1} with $H=2$, $\pi^E_h(s_0) = \widehat{\pi}^E_h(s_0) = a_1$ with $h \in \{1,2\}$. Set $p_1(s_+|s_0,a_1) = 1/2+\epsilon/4$ and $\widehat{p}_1(s_+|s_0,a_1) = 1/2-\epsilon/4$ and, thus, $\|p_1(\cdot|s_0,a_1)-\widehat{p}_1(\cdot|s_0,a_1) \|_1 = \epsilon$. Let us set $r_2(s_+)=1$ and $r_2(s_-)=-1$, which makes $\pi^E$ optimal under $p$. We observe that $\widehat{\mathcal{R}}$ is defined by $\widehat{r}_2(s_-) \le \widehat{r}_2(s_+)$. Recalling that the rewards are bounded in $[-1,1]$, we have $\mathcal{H}_{d^{\text{G}}}(\mathcal{R}_{\text{state}} ,\widehat{\mathcal{R}}_{\text{state}} ) \ge  1$.}
%We observe that $\mathcal{R}$ is defined by $r_2(s_+) \ge r_2(s_-)$ and $\widehat{\mathcal{R}}$ by $\widehat{r}_2(s_-) \le \widehat{r}_2(s_+)$. Recalling that the rewards are bounded in $[-1,1]$, we have $\mathcal{H}_\infty(\mathcal{R},\widehat{\mathcal{R}}) \ge  1$.
\end{restatable}

\begin{restatable}[Time-homogeneous reward $r(s,a)$]{example}{exampleB}\label{ex:timeHomo}
Time-homogeneous reward functions have been employed in several RL~\citep[\eg][]{dann} and IRL settings~\citep[\eg][]{lindner2022active}. We formalize the time-homogeneous feasible reward set as follows:
\begin{align*}
	\mathcal{R}_{\text{hom}} = \mathcal{R} \cap \{ \forall (s,a,h,h') \,:\, r_{h}(s,a) = r_{h'}(s,a) \}.
\end{align*}%
{\thinmuskip=2mu
\medmuskip=2mu
\thickmuskip=2mu%
Consider the \MDPRtext of Figure~\ref{fig:counterEx2} with $H=2$, $\pi^E_1(s_0) = \widehat{\pi}^E_1(s_0) = a_1$ and $\pi^E_2(s_0) = \widehat{\pi}^E_2(s_0) = a_2$. For $h \in\{1,2\}$, 
we set $p_h(s_0|s_0,a_1) = 1/2+\epsilon/4$ and $\widehat{p}_h(s_0|s_0,a_1) = 1/2-\epsilon/4$, thus, $\|p_h(\cdot|s_0,a_1)-\widehat{p}_h(\cdot|s_0,a_1) \|_1 = \epsilon$. We set $r(s_0,a_1) = 1$, $r(s_0,a_2) = 1 - \epsilon/6$, and $r(s_1,a_1)=r(s_1,a_2)=1/2$ making $\pi^E$ optimal. We can prove that $\mathcal{H}_{d^{\text{G}}}(\mathcal{R}_{\text{hom}} ,\widehat{\mathcal{R}}_{\text{hom}}) \ge  1/4$.}
\end{restatable}

\begin{restatable}[$\beta$-margin reward]{example}{exampleC}\label{ex:betaSep}
A $\beta$-margin reward enforces a suboptimality gap of at least $\beta > 0$ ~\cite{ng200algorithms,komanduru2019correctness}. We formalize it in the finite-horizon case with a sequence $\beta = (\beta_h)_{h \in \dsb{H}}$, possibly different for every stage:
{\thinmuskip=2mu
\medmuskip=2mu
\thickmuskip=2mu
\begin{align*}
	\mathcal{R}_{\beta\text{-mar}} = \mathcal{R} \cap \{ \forall (s,a,h) \,:\, A^{\pi^E}_h(s,a;r) \in \{0\} \cup (-\infty,-\beta_h] \}.
\end{align*}
}%
Consider the \MDPRtext in Figure~\ref{fig:counterEx1} with $\pi^E_h(s_0) = \widehat{\pi}^E_h(s_0) = a_1$ for $h \in \{1,2\}$. We set $p_1(s_+|s_0,a_1) = 1/2+\epsilon$ and $\widehat{p}_1(s_+|s_0,a_1) = 1/2-\epsilon$. We set for \MDPRtext $\mathcal{M}$ the reward function as $r_1(s_0,a)=0$ and $r_h(s_+,a)=-r_h(s_-, a)=1$ for $a \in \{a_1,a_2\}$ and $h \in \dsb{2,H}$. In $(s_0,1)$ the suboptimality gap is $\beta_1 = 2 + 2 \epsilon (H-1)$. By selecting $H \ge 1+1/\epsilon$,  the feasible set $\widehat{\mathcal{R}}_{\beta\text{-mar}}$ is empty.
\end{restatable}

These examples show that, under certain classes of restrictions, the feasible reward set is not Lipschitz continuous \wrt the transition model and, more in general, \wrt the IRL problem. The generalization of these examples to more abstract conditions for guaranteeing the Lipschitz continuity of the feasible reward set is beyond the scope of the paper.

\section{PAC Framework for IRL with a Generative Model}\label{sec:pacFw}
In this section, we discuss the PAC (Probably Approximately Correct) requirements for estimating the feasible reward set with access to a \emph{generative model} of the environment. We first provide the notion of a learning algorithm estimating the feasible reward set with a generative model (Section~\ref{sec:irlAlg}). Then, we formally present the PAC requirement for the Hausdorff (pre)metric $\mathcal{H}_{d}$ (Section~\ref{sec:pac}). Finally, we discuss the relationships between the PAC requirements with different choices of (pre)metric $d$ (Section~\ref{sec:other}).

\subsection{Learning Algorithms with a Generative Model}\label{sec:irlAlg}
A learning algorithm for estimating the feasible reward set is a pair $\mathfrak{A} = (\mu,\tau)$, where $\mu = (\mu_t)_{t \in \Nat}$ is a \emph{sampling strategy} defined for every time step $t \in \Nat$ as $\mu_t \in \Delta^{\SAHs}_{\mathcal{D}_{t-1}}$ with $\mathcal{D}_t = (\SAHs\times \Ss \times \As)^t$ and $\tau$ is a stopping time \wrt a suitably defined filtration. At every step $t \in \Nat$, the learning algorithm query the environment in a triple $(s_t,a_t,h_t)$, selected based on the sampling strategy $\mu_t(\cdot|D_{t-1})$, where $D_{t-1} = ((s_l,a_l,h_l,s_{l}',a_{l}^E))_{l=1}^{t-1} \in\mathcal{D}_{t-1}$ is the dataset of past samples. Then, the algorithm observes the next state $s_{t}' \sim p_{h_t}(\cdot|s_t,a_t)$ and expert's action $a_t^E \sim \pi^E_{h_t}(\cdot|s_t)$ and updates the dataset $D_t = D_{t-1} \oplus (s_t,a_t,h_t,s_{t}',a_{t}^E)$. Based on the collected data $ D_{\tau}$, the algorithm computes the empirical IRL problem $(\widehat{M}^\tau,\widehat{\pi}^{E,\tau})$, based on Equation~\eqref{eq:empiricalMDP} and the empirical feasible reward set $\widehat{\mathcal{R}}^{\tau}$.

\subsection{PAC Requirement}\label{sec:pac}
We now introduce a general notion of a PAC requirement for estimating the feasible reward set of an IRL problem. To this end, we consider the Hausdorff (pre)metric introduced in Section~\ref{sec:regul} defined in terms of the reward (pre)metric $d(r,\widehat{r})$. We denote with $d$-IRL the problem of estimating the feasible reward set under the Hausdorff (pre)metric $\mathcal{H}_{d}$. 

\begin{defi}[PAC Algorithm for \textbf{$d$-IRL}]\label{defi:pac}
	Let $\epsilon \in (0,2)$ and $\delta \in (0,1)$. An algorithm $\mathfrak{A}=(\mu,\tau)$ is $(\epsilon,\delta)$-PAC for $d$-IRL if:
	\begin{align*}
		\Prob_{(\mathcal{M},\pi^E),\mathfrak{A}} \left( \mathcal{H}_d( \mathcal{R}, \widehat{\mathcal{R}}^\tau ) \le \epsilon \right) \ge 1-\delta,
\end{align*}	
where $\Prob_{(\mathcal{M},\pi^E),\mathfrak{A}}$ denotes the probability measure induced by executing the algorithm $\mathfrak{A}$ in the IRL problem $(\mathcal{M},\pi^E)$ and $\widehat{\mathcal{R}}^\tau$ is the feasible reward set induced by the empirical IRL problem $(\widehat{\mathcal{M}}^\tau,\widehat{\pi}^{E,\tau})$ estimated with the dataset ${D}_\tau$.
The \emph{sample complexity} is defined as $\tau \coloneqq |D_\tau|$.
\end{defi}

In the next section, we show the relationship between PAC requirements defined for notable choices of $d$. 

%allows recovering the requirements already present in the literature beyond $d^{\text{G}}$ as we shall show in the next section.

\subsection{Different Choices of $d$}\label{sec:other}
So far, we have evaluated the dissimilarity between the feasible reward sets by means of the Hausdorff induced by $d^{\text{G}}$, \ie the $L_\infty$-norm of between individual reward functions. In the literature, other (pre)metrics $d$ have been proposed~\citep[\eg][]{metelli2021provably, lindner2022active}.

\textbf{$d_{Q^*}^{\text{G}}$-IRL}~~Since the recovered reward functions are often used for performing forward RL, an index of interest is the dissimilarity between optimal Q-functions obtained with the reward $r \in \rSet$ and $\widehat{r} \in \widehat{\rSet}$ in the original \MDPRtext:
\begin{align*}
	& d_{Q^*}^{\text{G}}(r,\widehat{r})\coloneqq\max_{(s,a,h)\in \SAHs} \left| Q^*_h(s,a;r) - Q^*_h(s,a;\widehat{r}) \right|.
\end{align*}

\textbf{$d_{V^*}^{\text{G}}$-IRL}~~We are often interested in not just being accurate in estimating the optimal Q-function, but rather in the performance of an optimal policy $\widehat{\pi}^*$, learned with the recovered reward $\widehat{r} \in \widehat{\rSet}$, evaluated under the true reward $r \in \rSet$:
{\thinmuskip=2mu
\medmuskip=2mu
\thickmuskip=2mu
\begin{equation*}\resizebox{.99\linewidth}{!}{
$\displaystyle
	d_{V^*}^{\text{G}}(r,\widehat{r})\coloneqq \sup_{\widehat{\pi}^* \in {\Pi}^*(\widehat{r})} \max_{(s,h)\in \mathcal{S}\times\dsb{H}} \left| V^*_h(s;r) - V^{\widehat{\pi}^*}_h(s;r) \right|,
	$}
\end{equation*}}%
{\thinmuskip=1mu
\medmuskip=1mu
\thickmuskip=1mu
where $ {\Pi}^*(\widehat{r}) \coloneqq \{ \pi\,:\, \forall (s,a,h) \in \SAHs:\, A^{\pi}_h(s,a;\widehat{r})\le 0\}$ is the set of optimal policies under the recovered reward $\widehat{r}$.}

The following result formalizes the relationships between the presented $d$-IRL problems.

\begin{restatable}[Relationships between $d$-IRL problems]{thr}{connObj}\label{thr:connObj}
Let us introduce the graphical convention for $c > 0$:
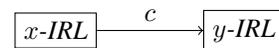
\begin{center}
\begin{tikzpicture}[node distance=2.5cm, square/.style={regular polygon,regular polygon sides=4}]
\node[rectangle, draw=black] (a) {$x$-IRL};
\node[rectangle, draw=black, right of=a] (b) {$y$-IRL};
\draw (a) edge[->, solid, above] node{$c$} (b);
\end{tikzpicture}
\end{center}
meaning that any $(\epsilon,\delta)$-PAC $x$-IRL algorithm is $(c\epsilon,\delta)$-PAC $y$-IRL. Then, the following statements hold:
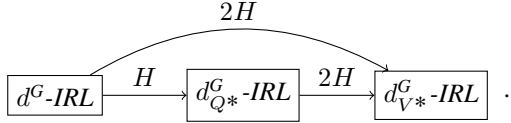
\begin{center}
\begin{tikzpicture}[node distance=2.5cm, square/.style={regular polygon,regular polygon sides=4}]
\node[rectangle, draw=black] (r1) {$d^{\text{G}}$-IRL};
\node[rectangle, draw=black, right of=r1] (q) {$d_{Q^*}^{\text{G}}$-IRL};
\node[rectangle, draw=black, right of=q] (v) {$d_{V^*}^{\text{G}}$-IRL};
\node[ draw=none, right=.1cm of v] (x) {.};

\draw (r1) edge[->, solid, above, bend left] node{$2H$} (v);
\draw (r1) edge[->, solid, above] node{$H$} (q);
\draw (q) edge[->, solid, above] node{$2H$} (v);
\end{tikzpicture}
\end{center}
\end{restatable}

Theorem~\ref{thr:connObj} shows that any $(\epsilon,\delta)$-PAC guarantee  on $d^{\text{G}}$, implies $(\epsilon',\delta)$-PAC guarantees on both $d_{Q^*}^{\text{G}}$ and $d_{V^*}^{\text{G}}$, where $\epsilon' = \Theta(H\epsilon)$ is linear in the horizon $H$.
% Thus, a sample complexity of order $\tau$ for the $d^{\text{G}}$-IRL problem translates into a sample complexity of order $H\tau$  (resp. $2H\tau$) for $d_{Q^*}^{\text{G}}$-IRL (resp. $d_{V^*}^{\text{G}}$-IRL).
This justifies why focusing on $d^{\text{G}}$-IRL, as in the following section where sample complexity lower bounds are derived. The lower bound analysis for $d_{Q^*}^{\text{G}}$-IRL and $d_{V^*}^{\text{G}}$-IRL is left to future works.

\section{Lower Bounds}\label{sec:lb}
In this section, we establish sample complexity lower bounds for the $d^{\text{G}}$-IRL problem based on the PAC requirement of Definition~\ref{defi:pac} in the generative model setting. We start presenting the general result (Section~\ref{sec:mainResLB}) and, then, we comment on its form and, subsequently, provide a sketch of the construction of the hard instances for obtaining the lower bound (Section~\ref{sec:constrLB}). For the sake of presentation, we assume that the expert's policy $\pi^E$ is known; the extension to the case of unknown $\pi^E$ is reported in Appendix~\ref{apx:notKnownPi}.

\subsection{Main Result}\label{sec:mainResLB}
In this section, we report the main result of the lower bound of the sample complexity of learning the feasible reward set.

\begin{restatable}[Lower Bound for $d^{\text{G}}$-IRL]{thr}{rInfLB}\label{thr:rInfLB}
Let $\mathfrak{A} = (\mu,\tau)$ be an $(\epsilon,\delta)$-PAC algorithm for $d^{\text{G}}$-IRL. Then, there exists an IRL problem $(\mathcal{M},\pi^E)$ such that, if $\delta \le 1/32$, $S \ge 9$, $A \ge 2$, and $H \ge 12$, the expected sample complexity is lower bounded by:
\begin{itemize}[topsep=0pt, noitemsep, leftmargin=*]
	\item if the transition model $p$ is time-inhomogeneous:
	\begin{align*}
		\E_{(\mathcal{M},\pi^E), \mathfrak{A}}\left[\tau\right] \ge \frac{1}{1024} \frac{H^3SA}{\epsilon^2} \left( \frac{1}{2} \log\left( \frac{1}{\delta} \right) + \frac{1}{5} S \right);
	\end{align*}
	\item if the transition model $p$ is time-homogeneous:
		\begin{align*}
		\E_{(\mathcal{M},\pi^E), \mathfrak{A}}\left[\tau\right] \ge\frac{1}{512} \frac{H^2SA}{\epsilon^2} \left( \frac{1}{2} \log\left( \frac{1}{\delta}\right) + \frac{1}{5} S \right),
	\end{align*}
\end{itemize}
where  $\E_{(\mathcal{M},\pi^E),\mathfrak{A}}$ denotes the expectation \wrt the probability measure $\Prob_{(\mathcal{M},\pi^E),\mathfrak{A}}$.
\end{restatable}

Some observations are in order. First, the derived lower bound displays a linear dependence on the number of actions $A$ and dependence on the horizon $H$ raised to a power $2$ or $3$, which depends on whether the underlying transition model is time-homogeneous, as common even for forward RL~\citep[\eg][]{dann, DominguesMKV21}. Second, we identify two different regimes visible inside the parenthesis related to the dependence on the number of states $S$ and the confidence $\delta$. Specifically, for small values of $\delta$ (\ie $\delta \approx 0$), the dominating part is $\log \left(\frac{1}{\delta} \right)$, leading to a sample complexity of order $\Omega \left( \frac{H^3SA}{\epsilon^2} \log  \left(\frac{1}{\delta} \right)\right)$. Instead, for large $\delta$ (\ie $\delta \approx 1/32$), the most relevant part is the one corresponding to $S$, leading to sample complexity of order $\Omega \left( \frac{H^3S^2A}{\epsilon^2} \right)$ (both for the time-inhomogeneous case). An analogous two-regime behavior has been previously observed in the reward-free exploration setting~\cite{JinKSY20, KaufmannMDJLV21, MenardDJKLV21}.

\begin{figure*}[t!]
\begin{subfigure}{.48\linewidth}
\centering
\begin{tikzpicture}[node distance=4cm]
\node[state] (s0) {$s_{*}$};
\node[draw=none, below right= .3cm and 2cm of s0, fill=black] (s0+) {};
\node[draw=none, above right= .3cm and 2cm of s0, fill=black] (s0-) {};
\node[state, below right= .6cm and 5cm of s0] (s+) {$s_{-}$};
\node[state, above right= .6cm and 5cm of s0] (s-) {$s_{+}$};

\draw (s0) edge[-, solid, below, vibrantBlue] node{$a_*$} (s0+);
\draw (s0) edge[-, double, above] node{$\neq a_*$} (s0-);

\draw (s0+) edge[->, solid, below, vibrantBlue] node{$\scriptstyle  1/2 - \epsilon'$} (s+);
\draw (s0-) edge[->, double, above] node{$\scriptstyle  1/2$} (s-);
\draw (s0+) edge[->, solid, right, vibrantBlue] node[pos=0.6]{$\scriptstyle 1/2 + \epsilon'$} (s-);
\draw (s0-) edge[->, double, above right] node[pos=0.75]{$\scriptstyle  1/2$} (s+);

\draw (s+) edge[->, loop right, solid, double] node{$1$ } (s0);
\draw (s-) edge[->, loop right, solid, double] node{$1$ } (s0);
\end{tikzpicture}
\caption{\MDPRtext used for the small-$\delta$ regime.}\label{fig:lbEx1}
\end{subfigure}
\begin{subfigure}{.48\linewidth}
\centering
\begin{tikzpicture}[node distance=4cm]
\node[state] (s0) {$s_{*}$};

\node[draw=none, below right= .5cm and 2cm of s0, fill=black] (s0+) {};
\node[draw=none, above right= .5cm and 2cm of s0, fill=black] (s0-) {};

\node[state, below right= 1cm and 5cm of s0] (s+) {$s_{\overline{S}}$};
\node[draw=white, state, right= 5cm of s0] (sss) {$\vdots$};
\node[state, above right= 1.5cm and 5cm of s0] (s-) {$s_{1}$};
\node[state, above right= .2cm and 5cm of s0] (s--) {$s_{2}$};

\draw (s0) edge[-, solid, below, vibrantRed] node{$a_0$} (s0+);
\draw (s0) edge[-, double, above, vibrantTeal] node{$a_j \neq a_0$} (s0-);

\draw (s0+) edge[->, solid, below, vibrantRed] node[pos=0.9]{$\scriptstyle 1/2$} (s-);
\draw (s0+) edge[->, solid, below, vibrantRed] node[pos=0.8]{$\scriptstyle 1/2$} (s--);
\draw (s0+) edge[->, solid, below right, vibrantRed] node[pos=0.5]{$\scriptstyle 1/2$} (s+);

\draw (s0-) edge[->, double, right, vibrantTeal] node[pos=0.55]{$\scriptstyle  (1+\epsilon' v_{\overline{S}}) /{\overline{S}} $} (s+);
\draw (s0-) edge[->, double, above left, vibrantTeal] node[pos=0.65]{$\scriptstyle  (1+\epsilon' v_1) /{\overline{S}}$} (s-);
\draw (s0-) edge[->, double, below, vibrantTeal] node[pos=0.6, ]{\setlength{\fboxsep}{0pt}\colorbox{white}{$\scriptstyle (1+\epsilon' v_{2}) /{\overline{S}}$} } (s--);

\draw (s+) edge[->, loop right, double] node{$1$ } (s0);
\draw (s-) edge[->, loop right, double] node{$1$ } (s0);
\draw (s--) edge[->, loop right, double] node{$1$ } (s0);
\end{tikzpicture}
\caption{\MDPRtext used for the large-$\delta$ regime.}\label{fig:lbEx2}
\end{subfigure}
\caption{The \MDPRtext employed in the constructions of the lower bounds of Section~\ref{sec:lb}. The expert's policy is $\pi^E(s) = a_0$. \begin{tikzpicture}
\node[ draw=none] (a) {}; 
\node[ draw=none, right= .5cm of a] (b) {}; 
\draw (a) edge[->, double] (b);
\end{tikzpicture} denotes a transition executed for multiple actions.}\label{fig:lbEx}
\end{figure*}

\subsection{Sketch of the Proof}\label{sec:constrLB}
In this section, we provide a sketch of the construction of the lower bounds of Theorem~\ref{thr:rInfLB}. The idea consists in deriving two separate bounds depending on the regime of $\delta$, which are based on two building blocks reported in Figure~\ref{fig:lbEx}. These instances are used to build lower bounds for a single state $s_*$ and the extension to multiple states and stages follows standard constructions~\citep[\eg][]{DominguesMKV21}.

\textbf{Small-$\delta$ regime}~~
Figure~\ref{fig:lbEx1} reports the instances employed in this regime. The expert's policy is $\pi^E(s) = a_0$. From state $s_*$, all actions bring the system to the absorbing states $s_+$ and $s_-$ with equal probability, except for action $a_* \neq a_0$ that increases by $\epsilon'>0$ the probability of reaching state $s_+$. The learner, in order to recover a correct feasible reward set, has to identify which is the action behaving like $a_*$ (among the $A$ available ones) to force action $a_0$ to be optimal. Considering $\Theta(A)$ instances, in which action $a_*$ changes, an application of \emph{Bretagnolle-Huber inequality}~\citep[][Theorem 14.2]{lattimore2020bandit} allows deriving a sample complexity lower bounded by $\Omega \left( \frac{AH^2}{\epsilon^2} \log \left( \frac{1}{\delta} \right) \right)$.

\textbf{Large-$\delta$ regime}~~
Figure~\ref{fig:lbEx2} depicts the instances used in this regime. The expert's policy is again $\pi^E(s) = a_0$. The system, instead, is made of $\overline{S} = \Theta(S)$ next states reachable with equal probability by playing action $a_0$. All other actions $a_j \neq a_0$ alter the probability distribution of the next state. Specifically, by playing the action $a_j\neq a_0$, the  probability of reaching the next state $s_k'$ is given by $(1+\epsilon' v^{(j)}_k)/{\overline{S}}$, where $v^{(j)} \in \{-1,1\}^{\overline{S}}$ is a vector such that $\sum_{k=1}^{\overline{S}} v^{(j)}_k = 0$. By varying $v_j$ in a suitable set, defined by means of a packing argument, we obtain $\Theta(2^{\overline{S}})$ instances each one separated by a finite dissimilarity, depending on $\epsilon'$. We obtain the lower bound by means of an application of the \emph{Fano's inequality}~\citep[][Proposition 4]{gerchinovitz2020fano} which results in order $\Omega \left( \frac{((1-\delta) - \log 2)S^2 A H^2}{\epsilon^2} \right)$.

\textbf{Extension to Multiple States and Stages}~~
At the beginning, the system randomly chooses a problem between Figure~\ref{fig:lbEx1} and Figure~\ref{fig:lbEx2}. Then, it transitions to the state in which the system may randomly remain for $\overline{H} < H$ stages after which it transitions with uniform probability to any of the $\Theta(S)$ states. $\overline{H}=\Theta(H)$ for the time-inhomogeneous (resp. $\overline{H}=O(1)$ for the time-homogeneous) case. In any state $s_*$ and stage $h_*$, the agent can face the problems shown in Figure~\ref{fig:lbEx}. By varying $s_*$ and $h_*$ among its possible $HS$ (resp. $S$) values, we get the bounds in Theorem~\ref{thr:rInfLB}. 

\begin{remark}[Generative vs Forward models]\label{remark}
This construction suffices for obtaining a bound for the generative model, but it can be easily extended to work with the \emph{forward model} of the environment (in which the agent interacts via
 trajectories only) by means of a standard \emph{tree-based construction}~\cite{JinKSY20, DominguesMKV21}. In such a case, the resulting PAC guarantee would no longer be expressed via the $L\infty$-norm distance $d^{\text{G}}$ between reward, but \emph{worst-case} over the visitation distributions induced by the policies: $d^{\text{F}}(r,\widehat{r}) \coloneqq \sup_{\pi} \E_{\mathcal{M},\pi}[|r_h(s,a) - \widehat{r}_h(s,a)|]$.
 \end{remark}
%
%{\color{red} Commentare sul forward model}
%
%It should be remarked that all our results can be easily extended to the infinite-horizon $\gamma$-discounted setting, where horizon becomes $\frac{1}{1-\gamma}\equiv H$.

\begin{figure}
\small
\noindent\fbox{%
    \parbox{\linewidth}{%
\begin{algorithmic}
   \STATE {\bfseries Input:} significance $\delta\in(0,1)$, $\epsilon$ target accuracy
   \STATE $t\gets 0$, $\epsilon_0\gets +\infty$
   \WHILE{$\epsilon_t>\epsilon$}
   \STATE $t \leftarrow t + SAH$
    \STATE Collect one sample from each $(s,a,h) \in \SAHs$
    \STATE Update $\widehat{p}^{t}$ according with \eqref{eq:empiricalMDP}
    \STATE Update $\epsilon_{t}=\max_{(s,a,h)\in \mathcal{S}\times\mathcal{A}\times\dsb{H}} \textcolor{vibrantBlue}{\mathcal{C}^{t}_h(s,a)}$ (resp. $\textcolor{vibrantRed}{\widetilde{\mathcal{C}}^{t}_h(s,a)}$)
   \ENDWHILE
\end{algorithmic}    
%    
%        {\bfseries Input:} significance $\delta\in(0,1)$, $\epsilon$ target accuracy\\
%        
%        Collect one sample from each $(s,a,h) \in \SAHs$ to obtain dataset $D_{t}$\\
%        Update transition model $\widehat{p}^{t}$ and 
%        NBB\\
%        \textbf{Stopping rule}: $\displaystyle \tau = \inf\left\{ t \in \mathbb{N}\,:\, \max_{(s,a,h) \in \SAHs} \mathcal{C}_{t}(s,a,h) \le \epsilon  \right\}$
    }%
}
\captionof{algorithm}{UniformSampling-IRL (\texttt{US-IRL}) for \textcolor{vibrantBlue}{time-inhomogeneous} (resp. \textcolor{vibrantRed}{time-homogeneous}) transition models.}\label{alg:UniformSampling-IRL}
\end{figure}

\section{Algorithm}\label{sec:algorithm}
In this section, we analyze the sample complexity of a uniform sampling strategy (UniformSampling-IRL, \texttt{US-IRL}) for the $d^{\text{G}}$-IRL problem (Algorithm~\ref{alg:UniformSampling-IRL}). We start presenting the sample complexity analysis (Section~\ref{sec:Algmain}) and, then, we provide a sketch of the proof (Section~\ref{sec:sketchAlg}).

\subsection{Main Result}\label{sec:Algmain}
The \texttt{US-IRL} algorithm was presented in \cite{metelli2021provably,lindner2022active} but analyzed for different IRL formulations (see Section~\ref{sec:RW}). We revise it since it matches our sample complexity lower bounds, provided that more sophisticated concentration tools \wrt those employed in~\cite{metelli2021provably,lindner2022active}. For the sake of presentation, we assume that the expert's policy $\pi^E$ is known; the extension to unknown $\pi^E$ is reported in Appendix~\ref{apx:notKnownPi}. At each iteration, the algorithm collects a sample from every $(s,a,h) \in \SAHs$ and, for time-inhomogeneous models, computes the confidence function:
\begin{align}\label{eq:algCI}
\textcolor{vibrantBlue}{\mathcal{C}^{t}_h(s,a)} \coloneqq 2\sqrt{2} (H-h+1)\sqrt{\frac{2\beta\bigl(n_h^t(s,a),\delta\bigl)}{n_h^t(s,a)}},
\end{align}
where $\beta\bigl(n,\delta\bigl)\coloneqq\log(SAH/\delta)+(S-1)\log\bigl(e(1+n/(S-1)\bigl)$.\footnote{In the time-homogeneous case, the algorithm merges the samples collected at different $h \in \dsb{H}$ for the estimation of the transition model and replaces the confidence function with:
\begin{align}\label{eq:algCI2}
\textcolor{vibrantRed}{\widetilde{\mathcal{C}}^{t}_h(s,a)} \coloneqq 2\sqrt{2} (H-h+1)\sqrt{\frac{2\widetilde{\beta}\bigl(n^t(s,a),\delta\bigl)}{n^t(s,a)}},
\end{align}
where $\widetilde{\beta}\bigl(n,\delta\bigl)\coloneqq\log(SA/\delta)+(S-1)\log\bigl(e(1+n/(S-1)\bigl)$ and $n^t(s,a) = \sum_{h=1}^H n^t_h(s,a)$.}
The algorithm stops as soon as all confidence functions fall below the threshold $\epsilon$. The following theorem provides the sample complexity of \texttt{US-IRL}.

\begin{restatable}[Sample Complexity of \texttt{US-IRL}]{thr}{scUSIRL}\label{thr:scUSIRL}
Let $\epsilon >0$ and $\delta \in (0,1)$, \texttt{US-IRL} is $(\epsilon,\delta)$-PAC for $d^{\text{G}}$-IRL and with probability at least $1-\delta$ it stops after $\tau$ samples with:
\begin{itemize}[topsep=0pt, noitemsep, leftmargin=*]
	\item if the transition model $p$ is time-inhomogeneous:
	\begin{align*}
		\tau \le \frac{8 H^3 SA}{\epsilon^2}  \left( \log \left( \frac{SAH}{\delta} \right) + (S-1) C \right),
	\end{align*}
	where $C = \log ( e/(S-1)+ (8 e H^2)/((S-1)\epsilon^2) (  \log (SAH/\delta) +  4e) )$;
	\item if the transition model $p$ is time-homogeneous and :
		\begin{align*}
		\tau \le \frac{8 H^2 SA}{\epsilon^2}  \left( \log \left( \frac{SA}{\delta} \right) + (S-1) C_2 \right),
	\end{align*}
	where $\widetilde{C} = \log (e/(S-1)+  (8 e H^2)/((S-1)\epsilon^2) (  \log (SA/\delta) +  4e) )$.
\end{itemize}
\end{restatable}
Thus,  time-inhomogeneous (resp. time-homogeneous) transition models, \texttt{US-IRL} suffers a sample complexity bound of order  $\widetilde{O} \left( \frac{H^3SA}{\epsilon^2}\left( \log\left( \frac{1}{\delta}   \right) +S \right)\right)$ (resp.  $\widetilde{O} \left( \frac{H^2SA}{\epsilon^2}\left( \log\left( \frac{1}{\delta}   \right) +S \right)\right)$) matching the lower bounds of Theorem~\ref{thr:rInfLB} up to logarithmic factors for both regimes of $\delta$.

\subsection{Sketch of the Proof}\label{sec:sketchAlg}
The idea of the proof is to exploit Theorem~\ref{thr:errorProp2} to reduce the Hausdorff distance to the $L_1$-norm between the transition model $\|\widehat{p}^t_h (\cdot|s,a) - p_h(\cdot|s,a)\|_1$. It is worth noting this term replaces  $|(\widehat{p}^t_h  - p_h) V_h|$ appearing in previous works~\cite{metelli2021provably,lindner2022active} that was comfortably bounded using H\"oeffding's inequality. In our case, the $L_1$-norm is unavoidable due to the Hausdorff distance that implies a worst-case choice of the reward function and, thus, of $V_h$. This term has to be carefully bounded using the stronger KL-divergence concentration result of~\citep[][Proposition 1]{JonssonKMDLV20} to get the $O(\log(1/\delta)+S)$ rate.\footnote{A more na\"ive application of the $L_1$-concentration of~\cite{weissman2003inequalities} would lead to the worse $O(S\log(1/\delta))$ rate.}

\section{Related Works}\label{sec:RW}
In this section, we discuss the related works about sample complexity analysis and lower bounds for IRL. Additional related works are reported in Appendix~\ref{apx:relwor}.

\textbf{Sample Complexity for Estimating the Feasible Reward Set}~~
The notion of feasible reward set $\rSet$ was introduced in~\cite{ng200algorithms} in an \emph{implicit} form in the infinite-horizon discounted case as a \emph{linear feasibility} problem and, subsequently, adapted to the finite-horizon case in~\cite{lindner2022active}. Furthermore, in~\cite{metelli2021provably,lindner2022active} an \emph{explicit} form of the reward functions belonging to the feasible region $\rSet$ was provided. In these works, the problem of estimating the feasible reward set is studied for the first time considering a \quotes{reference} pair of rewards $(\overline{r},\widecheck{r}) \in \mathcal{R \times \widehat{R}}$ against which to compare the rewards inside the recovered sets, leading to the (pre)metric:
\begin{align}\label{eq:guarOld}
	\widetilde{\mathcal{H}}_d(\mathcal{R},\mathcal{R},\overline{r},\widecheck{r}) \coloneqq \max \left\{ \inf_{\widehat{r} \in \widehat{\mathcal{R}}} d(\overline{r},\widehat{r}), \inf_{{r} \in {\mathcal{R}}} d({r},\widecheck{r}) \right\}.
\end{align}
Compared to the Hausdorff (pre)metric (Equation~\ref{eq:hausdorff}), in Equation~\eqref{eq:guarOld} there is no maximization over the choice of $(\overline{r},\widecheck{r})$, leading to a simpler problem.\footnote{In this sense, a PAC guarantee according to Definition~\ref{defi:pac}, implies a PAC guarantee defined \wrt (pre)metric of Equation~\eqref{eq:guarOld}.} In~\cite{metelli2021provably}, a uniform sampling approach (similar to Algorithm~\ref{alg:UniformSampling-IRL}) is proved to achieve a sample complexity of order $\widetilde{O}\left( \frac{\gamma^2SA}{(1-\gamma)^4\epsilon^2} \right)$ for the index of Equation~\eqref{eq:guarOld} with $d=d^{\text{G}}_{Q^*}$ in the discounted setting with generative model. For the forward model case, the \texttt{AceIRL} algorithm~\cite{lindner2022active} suffers a sample complexity of order $\widetilde{O}\left( \frac{H^5SA}{\epsilon^2} \right)$
for the index of Equation~\eqref{eq:guarOld} with $d=d^{\text{F}}_{V^*}$, in the finite-horizon case.\footnote{As discussed in Remark~\ref{remark}, in the forward model case, the dissimilarity is in expectation \wrt the worst-case policy.} Unfortunately, the reward recovered by \texttt{AceIRL} reward function is not guaranteed to be bounded by a predetermined constant (e.g., $[-1,1]$). Modified versions of these algorithms allow embedding problem-dependent features under a specific choice of a reward within the set.

\textbf{Sample Complexity Lower Bounds in IRL}~~To the best of our knowledge, the only work that proposes a sample complexity lower bound for IRL is~\cite{KomanduruH21}. The authors consider a finite state and action \MDPRtext and the IRL algorithm of~\cite{ng200algorithms} for $\beta$-strict separable IRL problems (\ie with suboptimality gap at least $\beta$) with state-only rewards in the discounted setting. When only two actions are available ($A=2$) and the samples are collected starting in each state with equal probability, by means of a geometric construction and Fano's inequality, the authors derive an $\Omega (S \log S)$ lower bound on the number of trajectories needed to identify a reward function. Note that this analysis limits to the \emph{identification} of a reward function within a finite set, rather than evaluating the accuracy of recovering the feasible reward set.

\section{Conclusions and Open Questions}\label{Conclusions}
In this paper, we provided contributions to the understanding of the complexity of the IRL problem. We conceived a lower bound of order $\Omega\left( \frac{H^3 SA}{\epsilon^2} \left( \log \left(\frac{1}{\delta}\right) + S \right)\right)$ on the number samples collected with a generative model in the finite-horizon setting. This result is of relevant interest since it sets, for the first time, the complexity of the IRL problem, defined as the problem of estimating the feasible reward set. Furthermore, we showed that a uniform sampling strategy matches the lower bound up to logarithmic factors. 
%This is reasonable since the Hausdorff distance defines no bias among the various regions of the problem. 
Nevertheless, the IRL problem is far from being closed. In the following, we outline a road map of open questions, hoping to inspire researchers to work in this appealing area.

\textbf{Forward Model}~~The most straightforward extension of our findings is moving to the \emph{forward model} setting, in which the agent can interact with the environment through trajectories only. As we already noted, our lower bounds can be comfortably extended to this setting.
%, by importing a \emph{tree-based construction}~\cite{lattimore2020bandit, JinKSY20, DominguesMKV21}. 
However, in this case, the PAC requirement has to be relaxed since controlling the $L_\infty$-norm between rewards is no longer a viable option (\eg for the possible presence of almost unreachable states). Which distance notion should be used for this setting? Will the Lipschitz regularity of Section~\ref{sec:regul} still hold?

%\textbf{Measuring distance}~~

\textbf{Problem-Dependent Analysis}~~Our analysis is \emph{worst-case} in the class of IRL problems. Would it be possible to obtain a \emph{problem-dependent} complexity results? Previous problem-dependent analyses provided results tightly connected to the properties of the specific reward selection procedure~\cite{metelli2021provably, lindner2022active}. Clearly, a currently open question, in all settings in which reward is missing, including reward-free exploration~\citep{JinKSY20} and IRL, is how to define a problem-dependent quantity in replacement of the suboptimality gaps. 

\textbf{Reward Selection}~~Our PAC guarantees concern with the complete feasible reward set. However, algorithmic solutions to IRL implement a specific criterion for selecting a reward (\eg maximum entropy, maximum margin). How the PAC guarantee based on the Hausdorff distance relates to guarantees on a single reward selected with a \emph{specific criterion} within $\mathcal{R}$?

\bibliographystyle{icml2023}
\bibliography{paper}

%%%%%%%%%%%%%%%%%%%%%%%%%%%%%%%%%%%%%%%%%%%%%%%%%%%%%%%%%%%%%%%%%%%%%%%%%%%%%%%
%%%%%%%%%%%%%%%%%%%%%%%%%%%%%%%%%%%%%%%%%%%%%%%%%%%%%%%%%%%%%%%%%%%%%%%%%%%%%%%
% APPENDIX
%%%%%%%%%%%%%%%%%%%%%%%%%%%%%%%%%%%%%%%%%%%%%%%%%%%%%%%%%%%%%%%%%%%%%%%%%%%%%%%
%%%%%%%%%%%%%%%%%%%%%%%%%%%%%%%%%%%%%%%%%%%%%%%%%%%%%%%%%%%%%%%%%%%%%%%%%%%%%%%
\newpage
\appendix
\onecolumn

\setlength{\abovedisplayskip}{8pt}
\setlength{\belowdisplayskip}{8pt}
\setlength{\textfloatsep}{16pt}

\section{Additional Related Works}\label{apx:relwor}
In this appendix, we report additional related works concerning sample complexity analysis for specific IRL algorithms and reward-free exploration.

\textbf{Sample Complexity of IRL Algorithms}~~Differently from forward RL, the theoretical understanding of the IRL problem is largely less established and the sample complexity analysis proposed in the literature often limit to specific algorithms. In the class of \emph{feature expectation} approaches, the seminal work~\cite{abbeel2004apprenticeship} propose IRL algorithms guaranteed to output an $\epsilon$-optimal policy (made of a mixture of Markov policies) after $\widetilde{O} \left( \frac{k}{\epsilon^2 (1-\gamma)^2}  \log \left( \frac{1}{\delta}\right)\right)$ trajectories (ideally of infinite length). The result holds in a discounted setting (being $\gamma$ the discount factor) under the assumption that the true reward function $r(s) = w^T \phi(s)$ is state-only and linear in some \emph{known} features $\phi$ of dimensionality $k$. In~\cite{syed2007game}, a game-theoretic approach to IRL, named \texttt{MWAL}, is proposed improving~\cite{abbeel2004apprenticeship} in terms of computational complexity and allowing the absence of an expert, preserving similar theoretical guarantees in the same setting. \texttt{Modular IRL}~\cite{vroman2014maximum}, that integrates supervised learning capabilities in the IRL algorithm, is guaranteed to produce an $\epsilon$-optimal policy after $\widetilde{O}\left( \frac{SA}{(1-\gamma)^2 \epsilon^2} \log \left( \frac{1}{\delta}\right) \right)$ trajectories. This class of algorithms, however, requires, as an inner step, to compute the optimal policy $\widehat{\pi}$ for every candidate reward function $\widehat{r}$. This step (and the corresponding sample complexity) is somehow hidden in the analysis since they either assume the knowledge of the transition model and apply dynamic programming~\citep[\eg][]{vroman2014maximum} or the access to a black-box RL algorithm~\citep[\eg][]{abbeel2004apprenticeship}. In the class of \emph{maximum entropy} approaches~\cite{ZiebartMBD08}, the \texttt{Maximum Likelihood IRL}~\cite{zeng2022maximumlikelihood} converges to a stationary solution with $\widetilde{O}(\epsilon^{-2})$ trajectories for \emph{non-linear} reward parametrization (with bounded gradient and Lipschitz smooth), when the underlying Markov chain is ergodic. Furthermore, the authors prove that, when the reward is linear in some features, the recovered solution corresponds to \texttt{Maximum Entropy IRL}~\cite{ZiebartMBD08}. Concerning the \emph{gradient-based} approaches, \cite{PirottaR16} and \cite{RamponiLMTR20} prove finite-sample convergence guarantee to the expert's weight under linear parametrization as a function of the accuracy of the gradient estimation. Surprisingly, a theoretical analysis of the IRL progenitor algorithm of~\cite{ng200algorithms} has been proposed only recently in~\cite{komanduru2019correctness}. A $\beta$-strict separability setting is enforced in which the rewards are assumed to lead to a suboptimality gap of at least $\beta>0$ when playing any non-optimal action. For finite MDPs, known expert's policy, under the demanding assumption that each state is reachable in one step with a minimum probability $\alpha>0$, and focusing on state-only reward, the authors prove that the algorithm outputs a $\beta$-strict separable feasible reward in at most $\widetilde{O}\left(  \frac{1+\gamma^2\Xi^2}{\alpha \beta^2 (1-\gamma)^4}  \log \left( \frac{1}{\delta}\right) \right)$ trajectories, where $\Xi \le S$ is the number of possible successor states. Recently, an approach with theoretical guarantees has been proposed for continuous states~\citep{DexterBH21}.

\textbf{Reward-Free Exploration}~~
Reward-free exploration~\citep[RFE,][]{JinKSY20, KaufmannMDJLV21, MenardDJKLV21} is a setting for pure exploration in MDPs composed of two phases: exploration and planning. In the exploration phase, the agent learns an estimated transition model $\widehat{p}$ without any reward feedback. In the planning phase, the agent is faced with a reward function $r$ and has to output an estimated optimal policy $\widehat{\pi}^*$, using $\widehat{p}$ since no further interaction with the environment is admitted. In this sense, RFE shares this two-phase procedure with our IRL problem, but, instead of the \emph{planning} phase, we face the \emph{computation} of the feasible reward set.\footnote{As shown in previous works, the computation of the feasible reward set can be formulated with a \emph{linear feasibility problem}~\cite{ng200algorithms}.} In RFE exploration, the sample complexity is computed against the performance of the learned policy $\widehat{\pi}^*$ under the reward $r$, \ie $V^*(\cdot;r) - V^{\widehat{\pi}^*}(\cdot;r)$, whose lower bound of the sample complexity has order $\Omega \left(  \frac{H^2 S A}{\epsilon^2} \left( H \log \left( \frac{1}{\delta} \right) + S \right)\right)$~\cite{JinKSY20, KaufmannMDJLV21}. The best known algorithm, \texttt{RF-Express}, proposed in~\cite{MenardDJKLV21} archives an almost-matching sample complexity of order $\Omega \left(  \frac{H^3 S A}{\epsilon^2} \left(  \log \left( \frac{1}{\delta} \right) + S \right)\right)$. The relevant connection with what we present in this paper is  the fact that the derivation of the lower bounds shares similarity especially in the construction of the instances. Nevertheless, in the time-inhomogeneous case, we achieve a higher lower bound of order  $\Omega \left(  \frac{H^3 S A}{\epsilon^2} \left(  \log \left( \frac{1}{\delta} \right) + S \right)\right)$. The connection between IRL and RFE should be investigated in future works, as also mentioned in~\cite{lindner2022active}.

\section{Proofs}\label{apx:proofs}
In this appendix, we report the proofs we omitted in the main paper.

\subsection{Proofs of Section~\ref{sec:regul}}

\begin{restatable}{lemma}{errorProp}\label{lemma:errorProp}
Let $r$ be feasible for the IRL problem $(\mathcal{M},\pi^E)$ bounded in $[-1,1]$ (\ie $\widehat{r} \in \mathcal{R}$) and defined according to Lemma~\ref{lemma:explicitRew} as $r_h(s,a)=-A_h(s,a) \indic_{\{\pi^E_h(a|s) = 0\}} + V_h(s) - p_h V_{h+1}(s,a)$. Let $(\widehat{\mathcal{M}},\widehat{\pi}^E)$ be an IRL problem and define for every $(s,a,h) \in \SAHs$:
\begin{align*}
\epsilon_h(s,a) & \coloneqq -A_h(s,a) \left( \indic_{\{\pi^E_h(a|s) = 0\}} - \indic_{\{\widehat{\pi}^E_h(a|s) = 0\}}\right) \\
& \quad + \left(\left(p_h - \widehat{p}_h \right) V_{h+1}\right)(s,a).
\end{align*}
Then, the reward function $\widehat{r}$ defined according to Lemma~\ref{lemma:explicitRew} as $\widehat{r}_h(s,a)=-\widehat{A}_h(s,a) \indic_{\{\widehat{\pi}^E_h(a|s) = 0\}} + \widehat{V}_h(s) - p_h \widehat{V}_{h+1}(s,a)$ for every $(s,a,h) \in \SAHs$ with:
\begin{align*}
	\widehat{A}_h(s,a) = \frac{A_h(s,a)}{1+ \epsilon}, \quad \widehat{V}_h(s) = \frac{V_h(s)}{1+ \epsilon}, \quad \widehat{V}_{H+1}(s) = 0.
\end{align*}
where $\epsilon \coloneqq \max_{(s,a,h) \in \SAHs} |\epsilon_h(s,a)|$, is feasible for the IRL problem $(\widehat{\mathcal{M}},\widehat{\pi}^E)$ and bounded in $[-1,1]$ (\ie $\widehat{r} \in \widehat{\mathcal{R}}$).
\end{restatable}

\begin{proof}
	Given the reward function ${r}_{h}(s,a) = -A_h(s,a) \indic_{\{\pi_h^E(a|s)=0\}} + V_h(s)- {p}_h V_{h+1}(s,a)$,  we define the reward function:
	\begin{align*}
		\widetilde{r}_{h}(s,a) = -A_h(s,a) \indic_{\{\widehat{\pi}_h^E(a|s)=0\}} + V_h(s) - \widehat{p}_h V_{h+1}(s,a),
	\end{align*}
	that, thanks to Lemma~\ref{lemma:explicitRew}, makes policy $\widehat{\pi}^E$ optimal. However, it is not guaranteed that $\widetilde{r} \in \widehat{\mathcal{R}}$ since it can take values larger than $1$. Thus, we define the reward:
	\begin{align*}
		\widehat{r}_{h}(s,a) = \frac{\widetilde{r}_{h}(s,a)}{1+\epsilon} = -\frac{A_h(s,a)}{1+\epsilon} \indic_{\{\widehat{\pi}_h^E(a|s)=0\}} +  \frac{V_h}{{1+\epsilon}}(s) - \widehat{p}_h \frac{V_{h+1}}{{1+\epsilon}}(s,a),
	\end{align*}
	which simply scales $\widetilde{r}_{h}$ and preserves the optimality of $\widehat{\pi}^E$. We now prove that $\widehat{r}_{h}(s,a)$ is bounded in $[-1,1]$. To do so, we prove that $\widetilde{r}_{h}(s,a)$ is bounded in $[-(1+\epsilon), (1+\epsilon)]$:
	\begin{align*}
		\left|\widetilde{r}_{h}(s,a)\right| & \le \left|{r}_{h}(s,a)\right| + \left|\widetilde{r}_{h}(s,a) - {r}_{h}(s,a)\right| \\
		& = 1 + \left| -A_h(s,a) \indic_{\{\widehat{\pi}_h^E(a|s)=0\}} + \widehat{p}_h V_{h+1}(s) - \left(-A_h(s,a) \indic_{\{\pi_h^E(a|s)=0\}} + {p}_h V_{h+1}(s) \right) \right|  \\
		&  = 1+|\epsilon_h(s,a)| \le 1+\epsilon.
	\end{align*}
\end{proof}

\errorPropA*

\begin{proof}
	Let $\widetilde{r}$ as defined in the proof of Lemma~\ref{lemma:errorProp}. Then, we have:
	\begin{align*}
		\left| r_h(s,a) - \widehat{r}_h(s,a) \right| & = \left| r_h(s,a) - \frac{\widetilde{r}_h(s,a)}{1+\epsilon} \right| \\
		& \le \frac{1}{1+\epsilon} \left( \left| r_h(s,a) - \widetilde{r}_h(s,a) \right| + \epsilon \left| r_h(s,a) \right| \right) \\
		&  \le \frac{2\epsilon}{1+\epsilon}.
	\end{align*}
	By recalling that $\frac{2\epsilon}{1+\epsilon}$ is a non-decreasing function of $\epsilon$, we bound it by replacing $\epsilon$ with an upper bound:
	\begin{align*}
		\epsilon & = \max_{(s,a,h)\in \SAHs} |\epsilon_h(s,a)| \\
		& \le \max_{(s,a,h)\in \SAHs} (H-h+1) \left[ \left| \indic_{\{\pi^E_h(a|s) = 0\}} - \indic_{\{\widehat{\pi}^E_h(a|s) = 0\}} \right|  + \left\| p_h(\cdot|s,a) - \widehat{p}_h(\cdot|s,a) \right\|_1  \right]\\
		&  \eqqcolon \rho^{\text{G}} ((\mathcal{M},\pi^E),(\widehat{\mathcal{M}},\widehat{\pi}^E)),
	\end{align*}
	where we used H\"older's inequality recalling that $|V_{h+1}(s,a)| \le H-h$ and $|A_h(s,a)| \le H-h+1$. 
	Clearly, $\rho^{\text{G}}  ((\mathcal{M},\pi^E),(\widehat{\mathcal{M}},\widehat{\pi}^E))$ is a (pre)metric.
\end{proof}

\begin{fact}\label{fact:largeRew}
There exist two \MDPRtext $\mathcal{M}$ and $\widehat{\mathcal{M}}$ with transition models $p$ and $\widehat{p}$ respectively, an expert's policy $\pi^E$ and a reward function $r_h(s,a) = -A_h(s,a) \indic_{\{\pi^E(a|s)=0\}} + V_h(s) - p_hV_{h+1}(s) $ feasible for the IRL problem $(\mathcal{M},\pi^E)$ bounded in $[-1,1]$ (\ie $r \in \mathcal{R}$) such that the reward function $\widehat{r}_h(s,a) = -A_h(s,a) \indic_{\{\pi^E(a|s)=0\}} + V_h(s) -\widehat{p}_hV_{h+1}(s) $ is feasible for the IRL problem $(\widehat{\mathcal{M}},\pi^E)$ not bounded in $[-1,1]$.
\end{fact}

\begin{proof}
We consider the \MDPRtext in Figure~\ref{fig:largeRew} with optimal policy and reward function defined for every $h \in \dsb{H}$ and $H=10$ as:
\begin{align*}
	& \pi^E_h(s_1) = a_1, \, \pi^E_h(s_2) = a_2, \\
	& r_h(s_1,a_1) = r_h(s_2,a_1)=0,\, r_h(s_1,a_2)=-1,\, r_h(s_2,a_2) = 1.
\end{align*}
Simple calculations lead to the V-function and advantage function values:
\begin{align*}
	& V^{\pi^E}_h(s_1) = 0,\, V^{\pi^E}_h(s_2) = H - h+1, \\
	& A^{\pi^E}_h(s_1,a_1) = 0,\,A^{\pi^E}_h(s_1,a_2) = - 1 + (H - h )/10 ,\,A^{\pi^E}_h(s_2,a_1) = - 1 - (H - h )/10,\,A^{\pi^E}_h(s_2,a_2) = 0 .
\end{align*}
We consider as alternative transition model $\widehat{p} = 1 -p$. After tedious calculations we obtain the alternative reward function:
\begin{align*}
\widehat{r}_h(s_1,a_1) = -(H-h),\,\widehat{r}_h(s_1,a_2) = - 1 + 8(H - h )/10 ,\,\widehat{r}_h(s_2,a_1) = 8 (H - h )/10,\,\widehat{r}_h(s_2,a_2) = H - h.
\end{align*}
It is simple to observe that for some $(s,a,h)$ we have $|\widehat{r}_h(s,a)| > 1$.

\begin{figure}[h!]
\centering
\begin{tikzpicture}[node distance=4cm]
\node[state] (s0) {$s_{1}$};
\node[state, right of=s0] (s1) {$s_{2}$};
\node[draw=none, below right= .3cm and 1cm of s0, fill=black] (s0+) {};

\node[draw=none, above left= .3cm and 1cm of s1, fill=black] (s0-) {};

\draw (s0) edge[->, solid, loop left] node{$a_1$} (s0);
\draw (s1) edge[->, solid, loop right] node{$a_2$} (s1);
\draw (s0) edge[-, solid, below] node{$a_2$} (s0+);
\draw (s1) edge[-, solid, above] node{$a_1$} (s0-);
\draw (s0+) edge[->, solid, below, out=-90, in=-90] node{$\scriptstyle 9/10$} (s0);
\draw (s0-) edge[->, solid, above, out=90, in=90] node{$\scriptstyle 1/10$} (s0);
\draw (s0+) edge[->, solid, below, out=-90, in=-90] node{$\scriptstyle 1/10$} (s1);
\draw (s0-) edge[->, solid, above, out=90, in=90] node{$\scriptstyle 9/10$} (s1);
%\draw (s0+) edge[->, solid, below] node[pos=0.8]{} (s-);
%\draw (s0-) edge[->, solid, above right] node[pos=0.75]{$\scriptstyle 1/2$} (s+);
%
%\draw (s+) edge[->, loop right, solid, double] node{$\scriptstyle 1$ } (s0);
%\draw (s-) edge[->, loop right, solid, double] node{$\scriptstyle 1$ } (s0);
\end{tikzpicture}
\caption{The \MDPRtext employed in Fact~\ref{fact:largeRew}.}\label{fig:largeRew}
\end{figure}
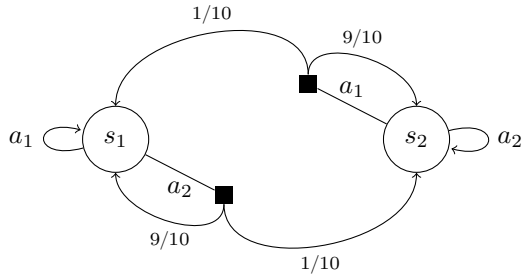
	
\end{proof}

\subsection{Proofs of Section~\ref{sec:pacFw}}

\connObj*
\begin{proof}
	Let $\mathfrak{A}$ be an $(\epsilon,\delta)$-PAC $d^{\text{G}}$-IRL algorithm. This means that with probability at least $1-\delta$, we have that for any IRL problem $\mathcal{H}_{d^{\text{G}}}(\mathcal{R},\widehat{\mathcal{R}}^\tau) \le \epsilon$. We introduce the following visitation distributions, defined for every $s,s' \in \Ss$, $h,l \in \dsb{H}$ with $l \ge h$, and $a,a' \in \As$:
	\begin{align*}
		\eta_{s,a,h,l}^{\pi}(s',a') = \Prob_{\mathcal{M},\pi} \left( s_l=s',a_l=a'|s_h=s,a_h=a\right), \qquad \eta_{s,h,l}^{\pi}(s',a') = \sum_{a \in \As} \pi_h(a|s) \eta_{s,a,h,l}^{\pi}(s',a').
	\end{align*}
	
	\textbf{$d^{\text{G}}$-IRL $\rightarrow$ $d^{\text{G}}_{Q^*}$-IRL}~~Let us consider the optimal Q-function difference and let $\pi^*$ an optimal policy under the reward function $r$, we have:
	\begin{align*}
		Q^*_h(s,a; r) - Q^*_h(s,a; \widehat{r})  & \le  Q^{\pi^*}_h(s,a; r) - Q^{\pi^*}_h(s,a; \widehat{r}) \\
		& = \sum_{l=h}^{H} \sum_{(s',a') \in \SAs} \eta_{s,a,h,l}^{\pi^*}(s',a')  (r_l(s',a') - \widehat{r}_l(s',a')) \\
		& \le \max_{(s,a,h) \in \SAHs} |r_h(s,a) - \widehat{r}_h(s,a) |\sum_{l=h}^{H} \underbrace{ \sum_{(s',a') \in \SAs} \eta_{s,a,h,l}^{\pi^*}(s',a') }_{= 1} \\
		& = (H-h+1)\max_{(s,a,h) \in \SAHs} |r_h(s,a) - \widehat{r}_h(s,a) | \\
		& \le H \max_{(s,a,h) \in \SAHs} |r_h(s,a) - \widehat{r}_h(s,a) | .
	\end{align*}
	As a consequence, we have:
	\begin{align*}
		\mathcal{H}_{d^{\text{G}}_{Q^*}}(\mathcal{R},\widehat{\mathcal{R}}^\tau) \le H \mathcal{H}_{d^{\text{G}}}(\mathcal{R},\widehat{\mathcal{R}}^\tau).
	\end{align*}
	
	\textbf{$d^{\text{G}}$-IRL $\rightarrow$ $d^{\text{G}}_{V^*}$-IRL}~~Let us consider the value functions and let $\pi^*$ (resp. $\widehat{\pi}^*$) be an optimal policy under reward function $r$ (resp. $\widehat{r}$), we have:
	\begin{align*}
		V^*_h(s;r) - V^{\widehat{\pi}^*}_h(s;r) & = V^{\pi^*}_h(s;r) - V^{\widehat{\pi}^*}_h(s;r) \pm V^{\widehat{\pi}^*}_h(s;\widehat{r}) \\
		& \le V^{\pi^*}_h(s;r)- V^{{\pi}^*}_h(s;\widehat{r}) + V^{\widehat{\pi}^*}_h(s;\widehat{r}) - V^{\widehat{\pi}^*}_h(s;r) \\
		& = \sum_{l=h}^{H} \sum_{(s',a') \in \SAs} \eta_{s,h,l}^{\pi^*}(s',a')(r_l(s',a') - \widehat{r}_l(s',a')) +  \sum_{l=h}^{H} \sum_{(s',a') \in \SAs} \eta_{s,h,l}^{\widehat{\pi}^*}(s',a')(r_l(s',a') - \widehat{r}_l(s',a')) \\
		& \le \max_{(s,a,h) \in \SAHs} |r_h(s,a) - \widehat{r}_h(s,a) | \left(\sum_{l=h}^{H} \sum_{(s',a') \in \SAs} \eta_{s,h,l}^{{\pi^*}}(s',a')+\sum_{l=h}^{H} \sum_{(s',a') \in \SAs} \eta_{s,h,l}^{\widehat{\pi}^*}(s',a') \right) \\
		& = 2(H-h+1) \max_{(s,a,h) \in \SAHs} |r_h(s,a) - \widehat{r}_h(s,a) |  \\
		& \le 2 H \max_{(s,a,h) \in \SAHs} |r_h(s,a) - \widehat{r}_h(s,a) | .
	\end{align*}
	Thus, it follows that:
	\begin{align*}
		\mathcal{H}_{d^{\text{G}}_{V^*}}(\mathcal{R},\widehat{\mathcal{R}}^\tau) \le 2H \mathcal{H}_{d^{\text{G}}}(\mathcal{R},\widehat{\mathcal{R}}^\tau).
	\end{align*}

	\textbf{$d^{\text{G}}_{Q^*}$-IRL $\rightarrow$ $d^{\text{G}}_{V^*}$-IRL}
	To prove this result, we need to introduce further tools. Specifically, we introduce the Bellman expectation operator and the Bellman optimal operator, defined for a reward function $r$, policy $\pi$, $(s,h) \in \Ss \times \dsb{H}$ and function $f: \Ss \rightarrow \Reals$:
	\begin{align*}
		T^*_{r,h} f (s) = \max_{a \in \As} \left\{ r_h(s,a) + p_h f(s,a) \right\}, \qquad T^{\pi}_{r,h} f (s) =\pi_h \left( r_h(s,a) + p_h f(s,a) \right).
	\end{align*}
	We recall the fixed-point properties: $T^*_{r,h} V^*_h =  V^*_h$ and $T^\pi_{r,h} V^\pi_h =  V^\pi_h$.	Let $\pi^*$ (resp. $\widehat{\pi}^*$) be an optimal policy under reward $r$ (resp. $\widehat{r}$). Let us consider the following derivation:
	\begin{align*}
	V^*_h(s;r) - V^{\widehat{\pi}^*}_h(s;r) & = T^*_{r,h} V^*_h(s;r) - T^{\widehat{\pi}^*}_{r,h} V^{\widehat{\pi}^*}_h(s;r) \pm T^{\pi^*}_{r,h} V_{h}^*(s;\widehat{r}) \pm T^{\pi^*}_{\widehat{r},h} V_{h}^*(s;\widehat{r}) \pm T^{*}_{\widehat{r},h} V_{h}^*(s;\widehat{r})  \pm T^{\widehat{\pi}^*}_{r,h} V^{\widehat{\pi}^*}_h(s;\widehat{r}) \\
	& = T^{\pi^*}_{r,h} V^*_{h}(s;r) - T^{\pi^*}_{r,h} V^*_{h}(s;\widehat{r}) + T^{\pi^*}_{r,h} V^*_h(s,\widehat{r}) - T^{\pi^*}_{\widehat{r},h} V^*_{h}(s;\widehat{r}) + \underbrace{T^{\pi^*}_{\widehat{r},h} V^*_{h}(s;\widehat{r}) - T^*_{\widehat{r},h} V^*_{h}(s;\widehat{r})}_{\le 0} \\
	& \quad + T^{\widehat{\pi}^*}_{\widehat{r},h} V^*_{h}(s;\widehat{r}) - T^{\widehat{\pi}^*}_{{r},h} V^*_{h}(s;\widehat{r}) + T^{\widehat{\pi}^*}_{{r},h} V^*_{h}(s;\widehat{r}) - T^{\widehat{\pi}^*}_{{r},h} V^*_{h}(s;{r})\\
	& = \pi^*_h p_h(V^*_{h+1}(\cdot;{r}) - V^*_{h+1}(\cdot;\widehat{r}))(s) + \pi^*_h(r_h - \widehat{r}_h)(s) \\
	& \quad + \widehat{\pi}^*_h (\widehat{r}_h - r_h)(s) + \widehat{\pi}^* p_h (V^*_{h+1}(\cdot;\widehat{r})-V^{\widehat{\pi}^*}_{h+1}(\cdot;{r}) )(s) \\
	& = (\pi^*_h - \widehat{\pi}^*_h) (Q^*_h(\cdot;r) - Q^*_h(\cdot;\widehat{r}))(s) + \widehat{\pi}^*p_h(V^*_{h+1}(\cdot;r) - V^{\widehat{\pi}^*}_{h+1}(\cdot;r)) (s).
	\end{align*}
	Let us apply the $L_\infty$-norm over the state space and the triangular inequality, we have:
	\begin{align*}
	\left\| V^*_h(\cdot;r) - V^{\widehat{\pi}^*}_h(\cdot;r) \right\|_\infty & \le \left\|(\pi^*_h - \widehat{\pi}^*_h) (Q^*_h(\cdot;r) - Q^*_h(\cdot;\widehat{r}))(\cdot) \right\|_\infty + \left\|\widehat{\pi}^*p_h(V^*_{h+1}(\cdot;r) - V^{\widehat{\pi}^*}_{h+1}(\cdot;r)) (\cdot)\right\|_\infty \\
	& \le 2 \left\|Q^*_h(\cdot;r) - Q^*_h(\cdot;\widehat{r}))(\cdot) \right\|_\infty + \left\|V^*_{h+1}(\cdot;r) - V^{\widehat{\pi}^*}_{h+1}(\cdot;r) \right\|_\infty.
	\end{align*}
	By unfolding the recursion over $h$, we obtain:
	\begin{align*}
	\left\| V^*_h(\cdot;r) - V^{\widehat{\pi}^*}_h(\cdot;r) \right\|_\infty \le 2\sum_{l={h+1}}^H \left\|Q^*_l(\cdot;r) - Q^*_l(\cdot;\widehat{r}))(\cdot) \right\|_\infty.
	\end{align*}
	Thus, we have:
	\begin{align*}
	\max_{(s,h) \in \Ss \times \dsb{H}} \left| V^*_h(s;r) - V^{\widehat{\pi}^*}_h(s;r) \right| \le 2 H \max_{(s,a,h) \in \SAHs} \left|Q^*_h(s,a;r) - Q^*_h(s,a;\widehat{r}) \right|.
	\end{align*}
	Since the derivation is carried out for arbitrary $\widehat{\pi}^*$, it follows that:
	\begin{align*}
		\mathcal{H}_{d^{\text{G}}_{V^*}}(\mathcal{R},\widehat{\mathcal{R}}^\tau) \le 2H \mathcal{H}_{d^{\text{G}}_{Q^*}}(\mathcal{R},\widehat{\mathcal{R}}^\tau).
	\end{align*}
\end{proof}

\subsection{Proofs of Section~\ref{sec:lb}}

\rInfLB*

\begin{proof}
	We put together the results of Theorem~\ref{thr:rInfLB1} and Theorem~\ref{thr:rInfLB2}, by recalling that $\max\{a,b\} \ge \frac{a+b}{2}$, or, equivalently, assuming to observe instances like the ones of Theorem~\ref{thr:rInfLB1} w.p. $1/2$ as well as those of Theorem~\ref{thr:rInfLB2}.
\end{proof}

\begin{restatable}{thr}{rInfLB1}\label{thr:rInfLB1}
Let $\mathfrak{A} = (\mu,\tau)$ be an $(\epsilon,\delta)$-PAC algorithm for $d^{\text{G}}$-IRL. Then, there exists an IRL problem $(\mathcal{M},\pi^E)$ such that, if $\epsilon \le 1$, $\delta < 1/16$, $S\ge 9$, $A\ge 2$, and $H \ge 12$, the expected sample complexity is lower bounded by:
\begin{itemize}[topsep=0pt, noitemsep, leftmargin=*]
	\item if the transition model $p$ is time-inhomogeneous:
	\begin{align*}
		\E_{(\mathcal{M},\pi^E),\mathfrak{A}}\left[\tau \right] \ge \frac{1}{2048} \frac{H^3SA}{\epsilon^2}  \log\left( \frac{1}{\delta} \right);
	\end{align*}
	\item if the transition model $p$ is time-homogeneous:
		\begin{align*}
		\E_{(\mathcal{M},\pi^E),\mathfrak{A}}\left[\tau \right] \ge \frac{1}{1024} \frac{H^2SA}{\epsilon^2}  \log\left( \frac{1}{\delta} \right).
	\end{align*}
\end{itemize}
\end{restatable}

\begin{proof}

\textbf{Step 1: Instances Construction}~~
The construction of the hard \MDPRtext instances follows similar steps as the ones presented in the constructions of lower bounds for policy learning~\citep{DominguesMKV21} and the hard instances are reported in Figure~\ref{fig:hardBB1} in a semi-formal way. The state space is given by $\Ss=\{s_{\text{start}}, s_{\text{root}}, s_-,s_+, s_1,\dots, s_{\overline{S}}\}$ and the action space is given by $\As = \{a_0,a_1,\dots,a_{\overline{A}}\}$. The transition model is described below and the horizon is $H \ge 3$. We introduce the constant $\overline{H} \in \dsb{H}$, whose value will be chosen later. Let us observe, for now, that if $\overline{H} = 1$, the transition model is time-homogeneous.

The agent begins in state $s_{\text{start}}$, where every action has the same effect. Specifically, if the stage $h < \overline{H}$, then there is  probability $1/2$ to remain in $s_{\text{start}}$ and a probability $1/2$ to transition to $s_{\text{root}}$. Instead, if $h \ge \overline{H}$, the state transitions to $s_{\text{root}}$ deterministically. From state $s_{\text{root}}$, every action has the same effect and the state transitions with equal probability $1/\overline{S}$ to a state $s_i$ with $i \in \dsb{\overline{S}}$. In all states $s_i$, apart from a specific one, i.e., state $s_*$, all actions have the same effect, i.e., transitioning to states $s_-$ and $s_+$ with equal probability $1/2$. State $s_*$ behaves as the other ones if the stage $h \neq h_*$, where $h_* \in \dsb{H}$ is a predefined stage. If, instead, $h=h_*$, all actions $a_j \neq a_*$ behave like in the other states, while for action $a_*$, we have a $1/2+\epsilon'$ probability of reaching $s_+$ (and consequently probability $1/2-\epsilon'$ of reaching $s_-$), with $\epsilon' \in [0,1/4]$. Notice that, having fixed $\overline{H}$, the possible values of $h^*$ are $\{3,\dots,2+\overline{H}\}$. States $s_+$ and $s_-$ are absorbing states. The expert's policy always plays action $a_0$.

Let us consider the base instance $\mathcal{M}_{0}$ in which there is no state behaving like $s_*$. Additionally, by varying the triple $\ell \coloneqq (s_*,a_*,h_*) \in \{s_1,\dots,s_{\overline{S}}\} \times \{a_1,\dots,a_{\overline{A}}\} \times \dsb{3,\overline{H}+2} \eqqcolon \mathcal{I}$, we can construct the  class of instances denoted by $\mathbb{M} = \{ \mathcal{M}_{\ell} : \ell \in \{0\} \cup \mathcal{I} \}$.

\begin{figure}[h!]
\centering
\begin{tikzpicture}[node distance=4cm]
\node[state] (s0) {$s_{\text{start}}$};
\node[state, below=2cm of s0] (s00) {$s_{\text{root}}$};

\node[state, below left of = s00, draw=none] (s1) {$\dots$};
\node[state, below right of = s00, draw=none] (sS) {$\dots$};
\node[state, below=2cm of s00, blue] (sj) {$s_*$};
\node[state, left of = s1] (s1) {$s_1$};
\node[state, right of = sS] (sS) {$s_{\overline{S}}$};
\node[state, below right=3cm and 2.5cm of sj] (plus) {$s_{+}$};
\node[state, below left=3cm and 2.5cm of sj] (minus) {$s_{-}$};
\draw (plus) edge[->, loop below, double] node{} (plus);
\draw (minus) edge[->, loop below, double] node{} (minus);
\draw (s0) edge[->, loop above, solid, double] node{$h < \overline{H}$ w.p. $\frac{1}{2}$ } (s0);
\draw (s0) edge[->, solid, right, double] node{w.p. $\frac{1}{2}$ or  $h \ge \overline{H}$} (s00);
\draw (s00) edge[->, solid, right, double] node{w.p. $\frac{1}{\overline{S}}$} (sj);
\draw (s00) edge[->, solid, above left, double] node{w.p. $\frac{1}{\overline{S}}$} (s1);
\draw (s00) edge[->, solid, above right, double] node{w.p. $\frac{1}{\overline{S}}$} (sS);

\draw (sS) edge[->, solid, below right, double] node{w.p. $\frac{1}{2}$} (plus);
\draw (sS) edge[->, solid, above, double] node[below, pos=0.3]{w.p. $\frac{1}{2}$} (minus);

\draw (s1) edge[->, solid,  double] node[below, pos=0.3]{w.p. $\frac{1}{2}$} (plus);
\draw (s1) edge[->, solid, below left, double] node{w.p. $\frac{1}{2}$} (minus);

\draw (sj) edge[->, solid, above right, bend left, blue, very thick] node[pos=0.3]{$h=h_*$ w.p. $\frac{1}{2}+\epsilon'$} (plus);
\draw (sj) edge[->, solid, above left, bend right, blue, very thick] node[pos=0.3]{$h=h_*$ w.p. $\frac{1}{2}-\epsilon'$} (minus);

\draw (sj) edge[->, solid, below, bend right, densely dashed, very thick] node[right, pos=0.6]{w.p. $\frac{1}{2}$} (plus);
\draw (sj) edge[->, solid, below, bend left, densely dashed, very thick] node[left, pos=0.6]{w.p. $\frac{1}{2}$} (minus);

\end{tikzpicture}	
\caption{Semi-formal representation of the the hard instances \MDPRtext used in the proof of Theorem~\ref{thr:rInfLB1}.}\label{fig:hardBB1}
\end{figure}

\textbf{Step 2: Feasible Set Computation}~~
Let us consider an instance $\mathcal{M}_\ell \in \mathbb{M}$, we now seek to provide a lower bound to the Hausdorff distance $\mathcal{H}_{d^{\text{G}}}\left( \mathcal{R}_{\mathcal{M}_0}, \mathcal{R}_{\mathcal{M}_{\ell}}\right)$. To this end, we focus on the triple $\ell = (s_*,a_*,h_*)$ and we enforce the convenience of action $a_0$ over action $a_*$. For the base \MDPRtext $\mathcal{M}_0$, let $r^{0} \in \mathcal{R}_{\mathcal{M}_0}$, we have:
\begin{align*}
	& r_{h_*}^0(s_*,a_0) + \frac{1}{2} \sum_{l=h_*+1}^H \left(r_l^0(s_-) + r_l^0(s_+)\right) \ge r_{h_*}^0(s_*,a_*) + \frac{1}{2} \sum_{l=h+1}^H \left(r_l^0(s_-) + r_l^0(s_+)\right) \\
	& \quad  \implies r_{h_*}^0(s_*,a_0) \ge r_{h_*}^0(s_*,a_*) ,
\end{align*}
For the alternative \MDPRtext $\mathcal{M}_{\ell}$, let $r^{\ell} \in \mathcal{R}_{\mathcal{M}_\ell}$, we have:
\begin{align*}
	& r_{h_*}^{\ell}(s_*,a_0) + \frac{1}{2} \sum_{l=h_*+1}^H \left(r_l^{\ell}(s_-) + r_l^{\ell}(s_+)\right) \ge r_{h_*}^{\ell}(s_*,a_*) + \sum_{l=h_*+1}^H \left(\left(\frac{1}{2} - \epsilon' \right) r_l^{\ell}(s_-) + \left(\frac{1}{2} + \epsilon' \right) r_l^{\ell}(s_+)\right) \\
	& \quad \implies r_{h_*}^{\ell}(s_*,a_0) \ge r_{h_*}^{\ell}(s_*,a_*)  - \epsilon' \sum_{l=h_*+1}^H \left( r_l^{\ell}(s_-)  - r_l^{\ell}(s_+)\right).
\end{align*}

In order to lower bound the Hausdorff distance $\mathcal{H}_{d^{\text{G}}}\left( \mathcal{R}_{\mathcal{M}_0}, \mathcal{R}_{\mathcal{M}_{\ell}} \right) $, we set for $\mathcal{M}_{\ell}$:

\begin{align*}
	r_l^{\ell}(s_-) = -r_l^{\ell}(s_+) = 1, \; r_{h_*}^{\ell}(s_*,a_*) = 1, \;r_{h_*}^{\ell}(s_*,a_0) = 1 - 2\epsilon'(H-h_*).
\end{align*}
Then, for notational convenience, for the \MDPRtext $\mathcal{M}_0$, we set $x \coloneqq r_{h_*}^0(s_*,a_0)$ and $y \coloneqq r_{h_*}^0(s_*,a_*)$:
\begin{align*}
	\mathcal{H}_{d^{\text{G}}}\left( \mathcal{R}_{\mathcal{M}_0}, \mathcal{R}_{\mathcal{M}_{\ell}}\right) \ge \min_{\substack{x,y \in [-1,1] \\ y \ge x}} \max \left\{ \left| x-1 \right|, \left| y-1+2\epsilon'(H-h_*)\right| \right\} = \epsilon'(H-h_*).
\end{align*}

We enforce the following constraint on this quantity:
\begin{align}\label{eq:ineqEps1}
\forall h^* \in \dsb{3,\overline{H}+2} \; : \; (H-h^*) \epsilon' \ge 2\epsilon \implies \epsilon' \ge \max_{ h^* \in \dsb{3,\overline{H}+2}} \frac{2\color{black}\epsilon}{(H-h^*)} = \frac{2\color{black}\epsilon}{(H-\overline{H}-2)}.
\end{align}
Notice that $\epsilon' \le 1/4$ whenever $H \ge \overline{H}+10$.

\textbf{Step 3: Lower bounding Probability}~~
Let us consider an $(\epsilon,\delta)$-correct algorithm $\mathfrak{A}$ that outputs the estimated feasible set $\widehat{\mathcal{R}}$. Thus, for every $\imath \in \mathcal{I}$, we can lower bound the error probability:
\begin{align*}
\delta & \ge \sup_{\text{all }\mathcal{M}\text{ \MDPRtext} \text{ and expert policies } \pi} \Prob_{(\mathcal{M},\pib),\mathfrak{A}} \left( \mathcal{H}_{d^{\text{G}}}\left( \mathcal{R}_{\mathcal{M}}, \widehat{\mathcal{R}} \right) \ge \epsilon \right) \\
& \ge \sup_{\mathcal{M} \in \mathbb{M}} \Prob_{(\mathcal{M},\pib),\mathfrak{A}} \left( \mathcal{H}_{d^{\text{G}}}\left( \mathcal{R}_{\mathcal{M}}, \widehat{\mathcal{R}} \right) \ge \epsilon \right) \\
& \ge \max_{\ell \in \{0,\imath\} }  \Prob_{(\mathcal{M}_\ell,\pib),\mathfrak{A}} \left( \mathcal{H}_{d^{\text{G}}}\left( \mathcal{R}_{\mathcal{M}_\ell}, \widehat{\mathcal{R}} \right) \ge \epsilon \right).
\end{align*}
For every $\imath \in \mathcal{I}$, let us define the \emph{identification function}:
\begin{align*}
	\Psi_{\imath} \coloneqq \argmin_{\ell \in \{0,\imath\}} \mathcal{H}_{d^{\text{G}}}\left( \mathcal{R}_{\mathcal{M}_\ell}, \widehat{\mathcal{R}} \right).
\end{align*}
Let $\jmath \in \{0,\imath\}$. If $\Psi_{\imath} = \jmath$, then, $\mathcal{H}_{d^{\text{G}}}(\mathcal{R}_{\mathcal{M}_{\Psi_{\imath}}}, \mathcal{R}_{\mathcal{M}_\jmath}) =0$. Otherwise, if $\Psi_{\imath} \neq \jmath$, we have:
\begin{align*}
	\mathcal{H}_{d^{\text{G}}}\left(\mathcal{R}_{\mathcal{M}_{\Psi_{\imath}}}, \mathcal{R}_{\mathcal{M}_\jmath} \right)& \le \mathcal{H}_{d^{\text{G}}}\left(\mathcal{R}_{\mathcal{M}_{\Psi_{\imath}}}, \widehat{\mathcal{R}}\right) + \mathcal{H}_{d^{\text{G}}}\left(\widehat{\mathcal{R}}, \mathcal{R}_{\mathcal{M}_\jmath} \right) \le 2\mathcal{H}_{d^{\text{G}}}\left(\widehat{\mathcal{R}}, \mathcal{R}_{\mathcal{M}_\jmath} \right) ,
\end{align*}
where the first inequality follows from triangular inequality and the second one from the definition of identification function $\Psi_{\imath}$. From Equation~\eqref{eq:ineqEps1}, we have that $\mathcal{H}_{d^{\text{G}}}\left(\mathcal{R}_{\mathcal{M}_{\Psi_{\imath}}}, \mathcal{R}_{\mathcal{M}_\jmath} \right) \ge 2\epsilon$. Thus, it follows that $\mathcal{H}_{d^{\text{G}}}\left(\widehat{\mathcal{R}}, \mathcal{R}_{\mathcal{M}_\jmath} \right) \ge \epsilon$. This implies the following inclusion of events for $\jmath \in \{0,\imath\}$:
\begin{align*}
 \left\{\mathcal{H}_{d^{\text{G}}}\left(\widehat{\mathcal{R}}, \mathcal{R}_{\mathcal{M}_\jmath} \right) \ge \epsilon \right\} \supseteq \left\{ \Psi_{\imath} \neq \jmath \right\}.
\end{align*}
Thus, we can proceed by lower bounding the probability:
\begin{align*}
\max_{\ell \in \{0,\imath\} }  \Prob_{(\mathcal{M}_\ell,\pib),\mathfrak{A}} \left( \mathcal{H}_{d^{\text{G}}}\left( \mathcal{R}_{\mathcal{M}_\ell}, \widehat{\mathcal{R}} \right) \ge \epsilon \right) & \ge \max_{\ell \in \{0,\imath\} } \Prob_{(\mathcal{M}_\ell,\pib),\mathfrak{A}} \left( \Psi_{\imath} \neq \ell \right) \\
& \ge \frac{1}{2} \left[ \Prob_{(\mathcal{M}_0,\pib),\mathfrak{A}}\left(  \Psi_{\imath} \neq 0 \right) + \Prob_{(\mathcal{M}_\imath,\pib),\mathfrak{A}}\left( \Psi_{\imath} \neq \imath \right) \right] \\
& =  \frac{1}{2} \left[ \Prob_{(\mathcal{M}_0,\pib),\mathfrak{A}}\left(  \Psi_{\imath} \neq 0 \right) + \Prob_{(\mathcal{M}_\imath,\pib),\mathfrak{A}}\left( \Psi_{\imath} = 0 \right) \right],
\end{align*}
where the second inequality follows from the observation that $\max\{a,b\} \ge \frac{1}{2} (a + b)$ and the equality from observing that $\Psi_{\imath} \in \{0,\imath\}$. The intuition behind this derivation is that we lower bound the probability of making a mistake $\ge \epsilon$ with the probability of failing in identifying the true underlying problem. We can now apply the Bretagnolle-Huber inequality~\citep[][Theorem 14.2]{lattimore2020bandit} (also reported in Theorem~\ref{thr:BHIneq} for completeness) with $\mathbb{P} = \Prob_{(\mathcal{M}_0,\pib),\mathfrak{A}}$, $\mathbb{Q} = \Prob_{(\mathcal{M}_0,\pib)\mathfrak{A}}$, and $\mathcal{A} = \{\Psi_{\imath} \neq 0\}$:
\begin{align*}
\Prob_{(\mathcal{M}_0,\pib),\mathfrak{A}}\left(  \Psi_{\imath} \neq 0 \right) + \Prob_{(\mathcal{M}_\imath,\pib),\mathfrak{A}}\left( \Psi_{\imath} = 0 \right) \ge \frac{1}{2} \exp \left( - D_{\text{KL}} \left( \Prob_{(\mathcal{M}_0,\pib),\mathfrak{A}}, \Prob_{(\mathcal{M}_\imath,\pib),\mathfrak{A}}\right) \right).
\end{align*}

\textbf{Step 4: KL-divergence Computation}~~
Let $\mathcal{M}\in \mathbb{M}$, we denote with $\mathbb{P}_{\mathfrak{A},\mathcal{M},\pib}$ the joint probability distribution of all events realized by the execution of the algorithm in the \MDPRtext (the presence of $\pib$ is irrelevant as we assume it known):
\begin{align*}
\Prob_{(\mathcal{M},\pib),\mathfrak{A}} = \prod_{t=1}^\tau \rho_t(s_t, a_t, h_t |H_{t-1}) p_{h_t}(s_{t}'|s_t,a_t).
\end{align*}
where $H_{t-1} = (s_1,a_1,h_1,s_1',\dots,s_{t-1},a_{t-1},h_{t-1},s_{t-1}')$ is the history. Let $\imath \in \mathcal{I}$ and denote with $p^0$ and $p^\imath$ the transition models associated with $\mathcal{M}_0$ and $\mathcal{M}_\imath$. Let us now move to the KL-divergence:
\begin{align*}
	D_{\text{KL}}\left( \mathbb{P}_{(\mathcal{M}_0,\pib),\mathfrak{A}}, \mathbb{P}_{(\mathcal{M}_\imath,\pib),\mathfrak{A}} \right) & = \E_{(\mathcal{M}_0,\pib),\mathfrak{A}} \left[ \sum_{t=1}^\tau D_{\text{KL}} \left(p_{h_t}^0(\cdot|s_t,a_t),p_{h_t}^\imath(\cdot|s_t,a_t) \right) \right] \\
	& \le \E_{(\mathcal{M}_0,\pib),\mathfrak{A}} \left[ N^\tau_{h_*}(s_*,a_*) \right] D_{\text{KL}} \left(p_{h_*}^0(\cdot|s_*,a_*),p_{h_*}^\imath(\cdot|s_*,a_*) \right) \\
	& \le 8(\epsilon')^2 \E_{(\mathcal{M}_0,\pib),\mathfrak{A}} \left[ N^\tau_{h_*}(s_*,a_*) \right].
	%
	%\E_{\mathcal{M}_0} \left[ N_{h_*}^\tau(s_*,a_*) \right] d_{\text{KL}}\left( \frac{1}{2}, \frac{1}{2}+\epsilon' \right) \le.
\end{align*}
having observed that the transition models differ in $\imath = (s_*,a_*,h_*)$ and defined $N^\tau_{h_*}(s_*,a_*) = \sum_{t=1}^\tau \indic\{(s_t,a_t,h_t) = (s_*,a_*,h_*)\}$ and the last passage is obtained by Lemma~\ref{lemma:KLBound} with $D=2$ (and $\epsilon=2\epsilon'$). Putting all together, we have:
\begin{align*}
	\delta \ge \frac{1}{4} \exp \left( - 8\E_{(\mathcal{M}_0,\pib),\mathfrak{A}} \left[  N_{h_*}^\tau(s_*,a_*)  \right] (\epsilon')^2 \right) \implies \E_{(\mathcal{M}_0,\pib),\mathfrak{A}} \left[  N_{h_*}^\tau(s_*,a_*)  \right] \ge \frac{\log \frac{1}{4\delta}}{8(\epsilon')^2} = \frac{(H-\overline{H}-2)^2\log \frac{1}{4\delta}}{32\epsilon^2}.
\end{align*}

Thus, summing over $(s_*,a_*,h_*) \in \mathcal{I}$, we have:
\begin{align*}
	\E_{(\mathcal{M}_0,\pib),\mathfrak{A}} \left[\tau\right] & \ge \sum_{(s_*,a_*,h_*) \in \mathcal{I}}\E_{(\mathcal{M}_0,\pib),\mathfrak{A}} \left[  N_{h_*}^\tau(s_*,a_*)  \right] \\
	& = \sum_{(s_*,a_*,h_*) \in \mathcal{I}} \frac{(H-\overline{H}-2 )^2\log \frac{1}{4\delta}}{32\epsilon^2} \\
	& =  \frac{\overline{S}\overline{A}\overline{H}(H-\overline{H} -2)^2}{32\epsilon^2}  \log \frac{1}{4\delta}.
\end{align*}
The number of states is given by $S = |\Ss| = \overline{S}+4$, the number of actions is given by $A = |\As| = \overline{A}+1$. Let us first consider the time-homogeneous case, \ie $\overline{H} = 1$:
\begin{align*}
\E_{(\mathcal{M}_0,\pib),\mathfrak{A}} \left[\tau\right] \ge \frac{(S-4)(A-1)(H-3)^2}{32\epsilon^2}  \log \frac{1}{4\delta}.
\end{align*}
For $\delta < 1/16$, $S\ge 9$, $A\ge 2$, $H \ge 10$, we obtain:
\begin{align*}
\E_{(\mathcal{M}_0,\pib),\mathfrak{A}} \left[\tau\right] \ge \frac{SAH^2}{1024\epsilon^2}  \log \frac{1}{\delta}.
\end{align*}
For the time-inhomogeneous case, instead, we select $\overline{H} = H/2$, to get:
\begin{align*}
\E_{(\mathcal{M}_0,\pib),\mathfrak{A}} \left[\tau\right] \ge \frac{(S-4)(A-1)(H/2)(H-H/2-2)^2}{\epsilon^2}  \log \frac{1}{4\delta}.
\end{align*}
For $\delta < 1/16$, $S\ge 9$, $A\ge 2$, $H \ge 12$, we obtain:
\begin{align*}
\E_{(\mathcal{M}_0,\pib),\mathfrak{A}} \left[\tau\right] \ge \frac{SAH^3}{2048\epsilon^2}  \log \frac{1}{\delta}.
\end{align*}

\end{proof}

\clearpage

\begin{restatable}{thr}{rInfLB2}\label{thr:rInfLB2}
Let $\mathfrak{A} = (\mu,\tau)$ be an $(\epsilon,\delta)$-PAC algorithm for $d^{\text{G}}$-IRL. Then, there exists an IRL problem $(\mathcal{M},\pi^E)$ such that, if $\epsilon \le 1$, $\delta \le 1/2$, $S \ge 16$, $A \ge 2$, $H \ge 131$, the expected sample complexity is lower bounded by:
\begin{itemize}[topsep=0pt, noitemsep, leftmargin=*]
	\item if the transition model $p$ is time-inhomogeneous:
	\begin{align*}
		\E_{(\mathcal{M},\pi^E),\mathfrak{A}}\left[\tau \right] \ge \frac{1}{5120} \frac{S^2AH^3}{\epsilon^2};
	\end{align*}
	\item if the transition model $p$ is time-homogeneous:
		\begin{align*}
		\E_{(\mathcal{M},\pi^E),\mathfrak{A}}\left[\tau \right]  \ge \frac{1}{2560} \frac{S^2AH^2}{\epsilon^2}.
	\end{align*}
\end{itemize}
\end{restatable}

\begin{proof}
\textbf{Step 1: Instances Construction}~~
The construction of the hard \MDPRtext instances for this second bound follows steps similar to those of reward free exploration~\citep{JinKSY20} and the instances are reported in Figure~\ref{fig:hardBB2} in a semi-formal way. The state space is given by $\Ss=\{s_{\text{start}}, s_{\text{root}}, s_1,\dots, s_{\overline{S}},s_1',\dots, s_{\overline{S}}'\}$ and the action space is given by $\As = \{a_0,a_1,\dots,a_{\overline{A}}\}$. We assume $\overline{S}$ to be divisible by $16$. The transition model is described below and the horizon is $H \ge 3$. 

The agent begins in state $s_{\text{start}}$, where every action has the same effect. Specifically, if the stage $h < \overline{H}$ ($\overline{H} \in \dsb{H}$, whose value will be chosen later), then there is  probability $1/2$ to remain in $s_{\text{start}}$ and a probability $1/2$ to transition to $s_{\text{root}}$. Instead, if $h \ge \overline{H}$, the state transitions to $s_{\text{root}}$ deterministically. From state $s_{\text{root}}$, every action has the same effect and the state transitions with equal probability $1/\overline{S}$ to a state $s_i$ with $i \in \dsb{\overline{S}}$. In every state $s_i$ and every stage $h$, action $a_0$ allows reaching states $s_1',\dots, s_{\overline{S}}'$ with equal probability $1/\overline{S}$. Instead, by playing the other actions $a_j$ with $j \ge 1$ at stage $h$, the probability distribution of the next state is given by $p_h(s_k'|s_i,a_j) = (1+\epsilon' v_k^{(s_i,a_j,h)}) / \overline{S}$ where the vector $v^{(s_i,a_j,h)} = (v_1^{(s_i,a_j,h)},\dots,v_{\overline{S}}^{(s_i,a_j,h)}) \in \mathcal{V}$, where $\mathcal{V} \coloneqq \{ \{-1,1\}^{\overline{S}} : \sum_{j=1}^{{\overline{S}}} v_j = 0\}$ and $\epsilon' \in [0,1/2]$. Notice that, having fixed $\overline{H}$, the possible values of $h$ are $\{3,\dots,2+\overline{H}\}$. States $s'_1,\dots,s'_{\overline{S}}$ are absorbing states. The expert's policy always plays action $a_0$.

Let us introduce the set $\mathcal{I} \coloneqq \{s_1,\dots,s_{\overline{S}}\} \times \{a_1,\dots,a_{\overline{A}}\} \times \dsb{3,\overline{H}+2}$. Let $\bm{v} = (v^\imath)_{\imath \in \mathcal{I}} \in \mathcal{V}^{\mathcal{I}}$ which is the set of vectors having as components the elements $v^{\imath}$ determining the probability distribution of the next state starting from the triple $\imath \in \mathcal{I}$. We denote with $\mathcal{M}_{\bm{v}}$ the \MDPRtext induced by $\bm{v}$.  We can construct the  class of instances denote by $\mathbb{M} = \{ \mathcal{M}_{\bm{v}} : \bm{v} \in \mathcal{V}^{\mathcal{I}} \}$. Moreover, we denoted with  $\mathcal{M}_{\bm{v} \stackrel{\imath}{\leftarrow} w}$ the instance in which we replace the $\imath$ component of $\bm{v}$, \ie $v^{\imath}$, with $w \in \mathcal{V}$ and $\mathcal{M}_{\bm{v} \stackrel{\imath}{\leftarrow} 0}$ the instance in which we replace the $\imath$ component of $\bm{v}$, \ie $v^{\imath}$, with the zero vector.

\begin{figure}[t]
\centering
\begin{tikzpicture}[node distance=3.5cm]
\node[state] (s0) {$s_{\text{start}}$};
\node[state, below=2cm of s0] (s00) {$s_{\text{root}}$};

\node[state, below left of = s00, draw=none] (s1xx) {$\dots$};
\node[state, below right of = s00, draw=none] (sSxx) {$\dots$};
%\node[state, right=3cm of s1] (sj) {$s_2$};
\node[state, left of = s1xx, blue] (s1) {$s_1$};
\node[state, right of = sSxx, red] (sS) {$s_{\overline{S}}$};
\node[state, below=3cm of s1] (s1') {$s_{1}'$};
\node[state, right=3cm of s1'] (s2') {$s_{2}'$};
\node[state, below=3cm of sS] (sS') {$s_{\overline{S}}'$};
\draw (s1') edge[->, loop below, double] node{} (s1');
\draw (s2') edge[->, loop below, double] node{} (s2');
\draw (sS') edge[->, loop below, double] node{} (sS');

\node[state, below of = s00, draw=none] (s1yy) {};
\node[state, below right of = s1yy, draw=none] (s1yy) {$\dots$};
\draw (s0) edge[->, loop above, solid, double] node{$h < \overline{H}$ w.p. $\frac{1}{2}$ } (s0);
\draw (s0) edge[->, solid, right, double] node{w.p. $\frac{1}{2}$ or  $h \ge \overline{H}$} (s00);
%\draw (s00) edge[->, solid, right, double] node{w.p. $\frac{1}{\overline{S}}$} (sj);
\draw (s00) edge[->, solid, above left, double] node{w.p. $\frac{1}{\overline{S}}$} (s1);
\draw (s00) edge[->, solid, above right, double] node{w.p. $\frac{1}{\overline{S}}$} (sS);

\draw (sS) edge[->, solid, above left, very thick, dashdotted, red] node[pos=0.3]{$a_j$ w.p. $\frac{1+\epsilon' v_{1}^{(s_{\overline{S}},a_j,h)}}{\overline{S}}$} (s1');
\draw (sS) edge[->, solid, below right,  very thick, dashdotted, red] node[pos=0.4]{$a_j$ w.p. $\frac{1+\epsilon' v_{2}^{(s_{\overline{S}},a_j,h)}}{\overline{S}}$} (s2');
\draw (sS) edge[->, solid, right,  very thick, dashdotted, red] node{$a_j$ w.p. $\frac{1+\epsilon' v_{\overline{S}}^{(s_{\overline{S}},a_j,h)}}{\overline{S}}$} (sS');

\draw (s1) edge[->, solid, left, very thick, dashdotted, blue] node{$a_j$ w.p. $\frac{1+\epsilon' v_1^{(s_1,a_j,h)}}{\overline{S}}$} (s1');
\draw (s1) edge[->, solid, above, very thick, dashdotted, blue] node{$a_j$ w.p. $\frac{1+\epsilon' v_2^{(s_1,a_j,h)}}{\overline{S}}$} (s2');
\draw (s1) edge[->, solid, below left, very thick, dashdotted, blue] node[pos=0.8]{$a_j$ w.p. $\frac{1+\epsilon' v_{\overline{S}}^{(s_1,a_j,h)}}{\overline{S}}$} (sS');

%\draw (sj) edge[->, solid, above right, bend left, blue, very thick] node{$h=h_*$ w.p. $\frac{1+\epsilon' v_{1}}{2}$} (s1');
%\draw (sj) edge[->, solid, above left, bend right, blue, very thick] node{$h=h_*$ w.p. $\frac{1+\epsilon' v_{2}}{2}$} (s2');
%\draw (sj) edge[->, solid, above left, bend right, blue, very thick] node{$h=h_*$ w.p. $\frac{1+\epsilon' v_{\overline{S}}}{2}$} (sS');
%
%\draw (sj) edge[->, solid, below, bend right, densely dashed, very thick] node{w.p. $\frac{1}{\overline{S}}$} (s1');
%\draw (sj) edge[->, solid, below, bend left, densely dashed, very thick] node{w.p. $\frac{1}{\overline{S}}$} (s2');
%\draw (sj) edge[->, solid, below, bend left, densely dashed, very thick] node{w.p. $\frac{1}{\overline{S}}$} (sS');

\end{tikzpicture}	
\caption{Semi-formal representation of the the hard instances \MDPRtext used in the proof of Theorem~\ref{thr:rInfLB2}.}\label{fig:hardBB2}
\end{figure}
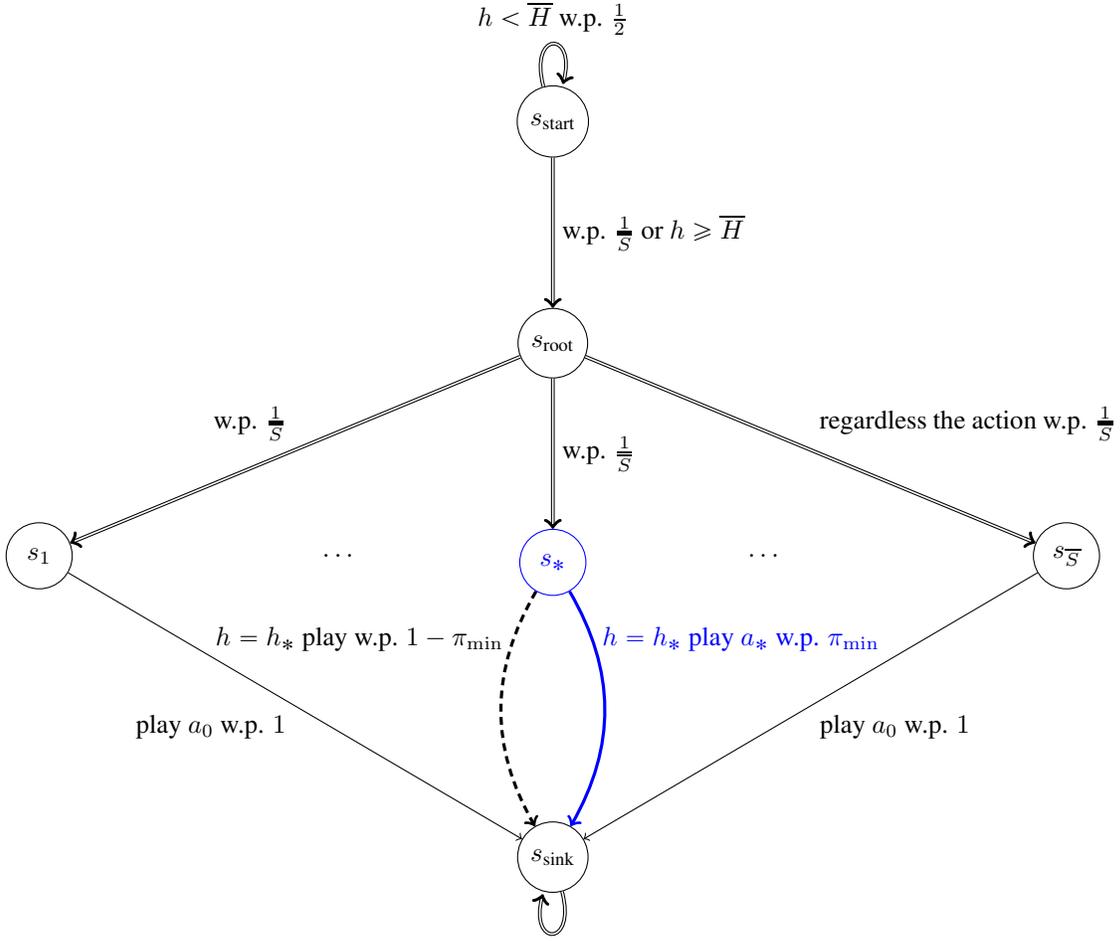

\textbf{Step 2: Feasible Set Computation}~~Thanks to Lemma~\ref{lemma:packing}, we know that there exists a subset $\overline{\mathcal{V}} \subset \mathcal{V}$ of cardinality at least $|\overline{\mathcal{V}} | \ge 2^{\overline{S}/5}$ such that for every $v,w \in \overline{\mathcal{V}} $ with $v \neq w$ we have $\sum_{j=1}^{\overline{S}} |v_j - w_j| \ge \overline{S}/16$. Thus, we consider the set $\overline{\mathcal{V}}^{\mathcal{I}} \subset {\mathcal{V}}^{\mathcal{I}}$ and to build the instances $\bm{v} \in \overline{\mathcal{V}}^{\mathcal{I}}$ and $v,w \in \overline{\mathcal{V}}$ with $v \neq w$. Let $\imath \in \mathcal{I}$, the induced instances are denoted by $\mathcal{M}_{\bm{v} \stackrel{\imath}{\leftarrow} v}, \mathcal{M}_{\bm{v} \stackrel{\imath}{\leftarrow} w} \in \mathbb{M}$.
%
%Let us consider $\imath \in \mathcal{I}$ and $\bm{v} \in $
%
%
%Let us consider two instances $\mathcal{M}_\ell^{v}, \mathcal{M}_\ell^{v'}\in \mathbb{M}$ with $v \neq v'$, we now seek to provide a lower bound to the Hausdorff distance $\mathcal{H}_{d^{\text{G}}}\left( \mathcal{R}_{\mathcal{M}_\ell^{v}}, \mathcal{R}_{\mathcal{M}_\ell^{v'}}\right)$.  Thus, we select $v,v' \in \overline{\mathcal{V}}$ to be used to construct the instances $\mathcal{M}_\ell^{v} $ and $\mathcal{M}_\ell^{v'}$.

To lower bound the Hausdorff distance, we focus on the triple $\imath = (s_*,a_*,h_*)$ and we enforce the convenience of action $a_0$ over action $a_*$. For both \MDPRtext $\mathcal{M}_{\bm{v} \stackrel{\imath}{\leftarrow} v}$ and $ \mathcal{M}_{\bm{v} \stackrel{\imath}{\leftarrow} w}$, let $r^{v} \in \mathcal{R}_{\mathcal{M}_{\bm{v} \stackrel{\imath}{\leftarrow} v}}$ and $r^{w} \in \mathcal{R}_{\mathcal{M}_{\bm{v} \stackrel{\imath}{\leftarrow} w}}$, we have:
%\begin{align*}
%	& r_{h_*}^0(s_*,a_0) + \frac{1}{\overline{S}} \sum_{l=h_*+1}^H \sum_{j =1}^{\overline{S}}r_l^0(s'_j) \ge r_{h_*}^0(s_*,a_*) + \sum_{l=h_*+1}^H \sum_{j =1}^{\overline{S}}r_l^0(s'_j) \\
%	& \quad  \implies r_{h_*}^0(s_*,a_0) \ge r_{h_*}^0(s_*,a_*) ,
%\end{align*}
%For the alternative \MDPRtext $\mathcal{M}_{\ell}^v$, let $r^{\ell,v} \in \mathcal{R}_{\mathcal{M}_\ell}$, we have:
\begin{align*}
	& r_{h_*}^{v}(s_*,a_0) + \frac{1}{\overline{S}} \sum_{l=h_*+1}^H \sum_{j =1}^{\overline{S}}r_l^{v}(s'_j) \ge r_{h_*}^{v}(s_*,a_*) + \sum_{l=h_*+1}^H \sum_{j =1}^{\overline{S}}  \frac{1+\epsilon' v_j}{\overline{S}}r_l^{v}(s'_j) \\
	& \quad \implies r_{h_*}^{v}(s_*,a_0) \ge r_{h_*}^{v}(s_*,a_*)  + \frac{\epsilon'}{\overline{S}}  \sum_{j =1}^{\overline{S}}   v_j \sum_{l=h_*+1}^H r_l^{v}(s'_j).
\end{align*}
\begin{align}
	& r_{h_*}^{w}(s_*,a_0) + \frac{1}{\overline{S}} \sum_{l=h_*+1}^H \sum_{j =1}^{\overline{S}}r_l^{w}(s'_j) \ge r_{h_*}^{w}(s_*,a_*) + \sum_{l=h_*+1}^H \sum_{j =1}^{\overline{S}}  \frac{1+\epsilon' w_j'}{\overline{S}}r_l^{w}(s'_j) \notag \\
	& \quad \implies r_{h_*}^{w}(s_*,a_0) \ge r_{h_*}^{w}(s_*,a_*)  + \frac{\epsilon'}{\overline{S}} \sum_{j =1}^{\overline{S}}   w_j \sum_{l=h_*+1}^H r_l^{w}(s'_j). \label{eq:constD}
\end{align}
In order to lower bound the Hausdorff distance $\mathcal{H}_{d^{\text{G}}}\left( \mathcal{M}_{\bm{v} \stackrel{\imath}{\leftarrow} v},\mathcal{M}_{\bm{v} \stackrel{\imath}{\leftarrow} w} \right) $, we set for $\mathcal{M}_{\bm{v} \stackrel{\imath}{\leftarrow} v}$:
\begin{align*}
	r_l^{v}(s'_j) = -v_j, \; r_{h_*}^{v}(s_*,a_*) = 1, \;r_{h_*}^{v}(s_*,a_0) = 1 - \epsilon'(H-h_*).
\end{align*}
We now want to find the closest reward function $r^w$ for the instance $\mathcal{M}_{\bm{v} \stackrel{\imath}{\leftarrow} w}$, recalling that there are at least $\overline{S}/16$ components of the vectors $v$ and $w$ that are different. Clearly, we can set $r_l^w(s_j')  = r_l^v(s_j') = -v_j$ for all $j \in \dsb{\overline{S}}$ in which $v_j= w_j$ since this will not increase the Hausdorff distance and make the constraint in Equation~\eqref{eq:constD} less restrictive. For symmetry reasons, we can limit our reasoning to the case in which  $v_j=-1$ and $w_j=1$ for the $j$ terms in which they are different. This, way, the constraint becomes:
\begin{align*}
	\underbrace{r_{h_*}^{w}(s_*,a_0)}_{\eqqcolon x} & \ge \underbrace{r_{h_*}^{w}(s_*,a_*)}_{\eqqcolon y}  - \frac{N_{v,w}}{\overline{S}} {\epsilon'} (H - h^*)\\
	& \qquad  + \left(1 -  \frac{N_{v,w}}{\overline{S}} \right) \epsilon' (H - h^*) \underbrace{\frac{1}{(H - h^*) \left(1 -  \frac{N_{v,w}}{\overline{S}} \right)}  \sum_{j : v_j \neq w_j}^{\overline{S}} \sum_{l=h_*+1}^H r_l^{w}(s'_j)}_{\eqqcolon z},
\end{align*}
where $N_{v,w} =\sum_{j=1}^{\overline{S}} \indic \{ v_j=w_j\}$. Notice that $z \in [-1,1]$. Let $\alpha = \frac{N_{v,w}}{\overline{S}}$, the Hausdorff distance can be lower bounded by:
\begin{align*}
	\mathcal{H}_{d^{\text{G}}}\left( \mathcal{M}_{\bm{v} \stackrel{\imath}{\leftarrow} v},\mathcal{M}_{\bm{v} \stackrel{\imath}{\leftarrow} w} \right)  & = \min_{\substack{x,y,z \in [-1,1] \\ y \ge x -\alpha \epsilon' (H-h^*)  + (1-\alpha)\epsilon' (H-h^*) z}} \max\left\{ |x-1|, |y - (1-\epsilon'(H-h^*))|, |z + 1| \right\} \\
	& \ge\min_{\substack{x,y \in [-1,1] \\ y \ge x -\alpha \epsilon' (H-h^*)}} \max\left\{ |x-1|, |y - (1-\epsilon'(H-h^*))| \right\} \\
	& = \frac{1}{2} (1-\alpha) \epsilon' (H-h^*) \ge \frac{\epsilon'}{32}(H-h^*) ,
\end{align*}
where the first inequality derives from the fact that to have a Hausdorff distance smaller than $1$, we must take $z < 0$ at least and the second inequality is obtained by recalling that $1-\alpha \ge \frac{1}{16}$ for the packing argument.

We enforce the following constraint on this quantity:
\begin{align}\label{eq:ineqEps}
\forall h^* \in \dsb{3,\overline{H}+2} \; : \; \frac{\epsilon'}{32}(H-h^*) \ge 2 \epsilon \implies \epsilon' \ge \max_{ h^* \in \dsb{3,\overline{H}+2}} \frac{\epsilon}{64(H-h^*)} = \frac{64\epsilon}{(H-\overline{H}-2)}.
\end{align}
Notice that $\epsilon' \le 1/2$ whenever $H \ge \overline{H}+ 130$.

\textbf{Step 3: Lower bounding Probability}~~Let us consider an $(\epsilon,\delta)$-correct algorithm $\mathfrak{A}$ that outputs the estimated feasible set $\widehat{\mathcal{R}}$. Thus, consider $\imath \in \mathcal{I}$ and $\bm{v} \in \overline{\mathcal{V}}^{\mathcal{I}}$, we can lower bound the error probability:
\begin{align*}
\delta & \ge \sup_{\text{all }\mathcal{M}\text{ \MDPRtext} \text{ and expert policies } \pi} \Prob_{(\mathcal{M},\pib),\mathfrak{A}} \left( \mathcal{H}_{d^{\text{G}}}\left( \mathcal{R}_{\mathcal{M}}, \widehat{\mathcal{R}} \right) \ge \epsilon \right) \\
& \ge \sup_{\mathcal{M} \in \mathbb{M}} \Prob_{(\mathcal{M},\pib),\mathfrak{A}} \left( \mathcal{H}_{d^{\text{G}}}\left( \mathcal{R}_{\mathcal{M}}, \widehat{\mathcal{R}} \right) \ge \epsilon \right) \\
& \ge \max_{w \in \overline{\mathcal{V}}}  \Prob_{(\mathcal{M}_{\bm{v} \stackrel{\imath}{\leftarrow} w },\pib),\mathfrak{A}} \left( \mathcal{H}_{d^{\text{G}}}\left( \mathcal{R}_{\mathcal{M}_{\bm{v} \stackrel{\imath}{\leftarrow} w }}, \widehat{\mathcal{R}} \right) \ge \epsilon \right).
\end{align*}
For every $\imath \in \mathcal{I}$ and $\bm{v} \in \overline{\mathcal{V}}^{\mathcal{I}}$, let us define the \emph{identification function}:
\begin{align*}
	\Psi_{\imath, \bm{v} } \coloneqq \argmin_{w \in \overline{\mathcal{V}}} \mathcal{H}_{d^{\text{G}}} \left( \mathcal{R}_{\mathcal{M}_{\bm{v} \stackrel{\imath}{\leftarrow} w }}, \widehat{\mathcal{R}} \right).
\end{align*}
Let $w \in \overline{\mathcal{V}}$. If $\Psi_{\imath, \bm{v} } = w$, then, $\mathcal{H}_{d^{\text{G}}}(\mathcal{R}_{\mathcal{M}_{\bm{v} \stackrel{\imath}{\leftarrow} \Psi_{\imath, \bm{v} } }}, \mathcal{R}_{\mathcal{M}_{\bm{v} \stackrel{\imath}{\leftarrow} w} }) =0$. Otherwise, if $\Psi_{\imath, \bm{v}} \neq w$, we have:
\begin{align*}
	\mathcal{H}_{d^{\text{G}}}(\mathcal{R}_{\mathcal{M}_{\bm{v} \stackrel{\imath}{\leftarrow} \Psi_{\imath, \bm{v} } }}, \mathcal{R}_{\mathcal{M}_{\bm{v} \stackrel{\imath}{\leftarrow} w} }) & \le \mathcal{H}_{d^{\text{G}}}(\mathcal{R}_{\mathcal{M}_{\bm{v} \stackrel{\imath}{\leftarrow} \Psi_{\imath, \bm{v} } }}, \widehat{\mathcal{R}} )+\mathcal{H}_{d^{\text{G}}}(\widehat{\mathcal{R}}, \mathcal{R}_{\mathcal{M}_{\bm{v} \stackrel{\imath}{\leftarrow} w} }) \le 2\mathcal{H}_{d^{\text{G}}}(\widehat{\mathcal{R}}, \mathcal{R}_{\mathcal{M}_{\bm{v} \stackrel{\imath}{\leftarrow} w} }) ,
\end{align*}
where the first inequality follows from triangular inequality and the second one from the definition of identification function $\Psi_{\imath,\bm{v}}$. From Equation~\eqref{eq:ineqEps}, we have that $\mathcal{H}_{d^{\text{G}}}\left(\mathcal{R}_{\mathcal{M}^{\Psi_{\imath}}_{\imath}}, \mathcal{R}_{\mathcal{M}_\imath^v} \right) \ge 2\epsilon$. Thus, it follows that $\mathcal{H}_{d^{\text{G}}}(\widehat{\mathcal{R}}, \mathcal{R}_{\mathcal{M}_{\bm{v} \stackrel{\imath}{\leftarrow} w} }) \ge \epsilon$. This implies the following inclusion of events for $w \in\overline{\mathcal{V}}$:
\begin{align*}
 \left\{\mathcal{H}_{d^{\text{G}}}(\widehat{\mathcal{R}}, \mathcal{R}_{\mathcal{M}_{\bm{v} \stackrel{\imath}{\leftarrow} w} })  \ge \epsilon \right\} \supseteq \left\{\Psi_{\imath, \bm{v} } \neq w \right\}.
\end{align*}
Thus, we can proceed by lower bounding the probability:
\begin{align*}
\max_{w \in \overline{\mathcal{V}}}  \Prob_{(\mathcal{M}_{\bm{v} \stackrel{\imath}{\leftarrow} w },\pib),\mathfrak{A}} \left( \mathcal{H}_{d^{\text{G}}}\left( \mathcal{R}_{\mathcal{M}_{\bm{v} \stackrel{\imath}{\leftarrow} w }}, \widehat{\mathcal{R}} \right) \ge \epsilon \right) & \ge \max_{w \in \overline{\mathcal{V}}}  \Prob_{(\mathcal{M}_{\bm{v} \stackrel{\imath}{\leftarrow} w },\pib),\mathfrak{A}} \left( \Psi_{\imath, \bm{v} } \neq w \right) \\
& \ge \frac{1}{|\mathcal{\overline{V}}|} \sum_{w \in \mathcal{\overline{V}}} \Prob_{(\mathcal{M}_{\bm{v} \stackrel{\imath}{\leftarrow} w },\pib),\mathfrak{A}} \left( \Psi_{\imath, \bm{v} } \neq w \right),
\end{align*}
where the second inequality follows from bounding the maximum of probability with the average. We can now apply the Fano's inequality (Theorem~\ref{thr:Fano}) with reference probability $\mathbb{P}_0 = \Prob_{(\mathcal{M}_{\bm{v} \stackrel{\imath}{\leftarrow} 0},\pib),\mathfrak{A}}$, $\mathbb{P}_w = \Prob_{(\mathcal{M}_{\bm{v} \stackrel{\imath}{\leftarrow} w },\pib),\mathfrak{A}}$, and $\mathcal{A}_w = \{\Psi_{\imath, \bm{v} } \neq w\}$:
\begin{align}\label{eq:FanoAppl}
\frac{1}{|\mathcal{\overline{V}}|} \sum_{w \in \mathcal{\overline{V}}} \Prob_{(\mathcal{M}_{\bm{v} \stackrel{\imath}{\leftarrow} w },\pib),\mathfrak{A}} \left( \Psi_{\imath, \bm{v} } \neq w \right) \ge 1 - \frac{1}{\log |\overline{\mathcal{V}}|} \left( \frac{1}{|\overline{\mathcal{V}}|} \sum_{w \in \overline{\mathcal{V}}} D_{\text{KL}} \left(\Prob_{(\mathcal{M}_{\bm{v} \stackrel{\imath}{\leftarrow} w },\pib),\mathfrak{A}},\Prob_{(\mathcal{M}_{\bm{v} \stackrel{\imath}{\leftarrow} 0},\pib),\mathfrak{A}}\right) - \log 2 \right).
\end{align}

\textbf{Step 4: KL-divergence Computation}~~
Let $\mathcal{M}$ be an instance, we denote with $\mathbb{P}_{\mathfrak{A},\mathcal{M},\pib}$ the joint probability distribution of all events realized by the execution of the algorithm in the \MDPRtext (the presence of $\pib$ is irrelevant as we assume it known):
\begin{align*}
\Prob_{(\mathcal{M},\pib),\mathfrak{A}} = \prod_{t=1}^\tau \rho_t(s_t, a_t, h_t |H_{t-1}) p_{h_t}(s_{t}'|s_t,a_t).
\end{align*}
where $H_{t-1} = (s_1,a_1,h_1,s_1',\dots,s_{t-1},a_{t-1},h_{t-1},s_{t-1}')$ is the history up to time $t-1$. Let $\imath \in \mathcal{I}$ and $v \in \mathcal{\overline{V}}$ and denote with $p^{\bm{v} \stackrel{\imath}{\leftarrow} 0}$ and $p^{\bm{v} \stackrel{\imath}{\leftarrow} w}$ the transition models associated with $\mathcal{M}_{\bm{v} \stackrel{\imath}{\leftarrow} 0}$ and $\mathcal{M}_{\bm{v} \stackrel{\imath}{\leftarrow} w }$. Let us now move to the KL-divergence and denoting $\imath = (s_*,a_*,h_*)$:
Thus, we have:
\begin{align*}
	 D_{\text{KL}} \left(\Prob_{(\mathcal{M}_{\bm{v} \stackrel{\imath}{\leftarrow} w },\pib),\mathfrak{A}},\Prob_{(\mathcal{M}_{\bm{v} \stackrel{\imath}{\leftarrow} 0},\pib),\mathfrak{A}}\right) & = \E_{(\mathcal{M}_{\bm{v} \stackrel{\imath}{\leftarrow} w },\pib),\mathfrak{A}} \left[ \sum_{t=1}^\tau D_{\text{KL}} \left(p_{h_t}^{\bm{v} \stackrel{\imath}{\leftarrow} w}(\cdot|s_t,a_t), p_{h_t}^{\bm{v} \stackrel{\imath}{\leftarrow} 0}(\cdot|s_t,a_t) \right) \right] \\
	& \le \E_{(\mathcal{M}_{\bm{v} \stackrel{\imath}{\leftarrow} w },\pib),\mathfrak{A}}  \left[ N^\tau_{h_*}(s_*,a_*) \right] D_{\text{KL}} \left(p_{h_*}^{\bm{v} \stackrel{\imath}{\leftarrow} w}(\cdot|s_*,a_*),p_{h_*}^{\bm{v} \stackrel{\imath}{\leftarrow} 0}(\cdot|s_*,a_*) \right) \\
	& \le  2 (\epsilon')^2\E_{(\mathcal{M}_{\bm{v} \stackrel{\imath}{\leftarrow} w },\pib),\mathfrak{A}}   \left[ N^\tau_{h_*}(s_*,a_*) \right],
	%
	%\E_{\mathcal{M}_0} \left[ N_{h_*}^\tau(s_*,a_*) \right] d_{\text{KL}}\left( \frac{1}{2}, \frac{1}{2}+\epsilon' \right) \le.
\end{align*}
having observed that the transition models differ in $\imath = (s_*,a_*,h_*)$ and defined $N^\tau_{h_*}(s_*,a_*) = \sum_{t=1}^\tau \indic\{(s_t,a_t,h_t) = (s_*,a_*,h_*)\}$ and the last passage is obtained by Lemma~\ref{lemma:KLBound} with $D=\overline{S}$. 
Plugging into Equation~\eqref{eq:FanoAppl}, we obtain:
\begin{align*}
\delta \ge \frac{1}{|\mathcal{\overline{V}}|} \sum_{w \in \mathcal{\overline{V}}} \Prob_{(\mathcal{M}_{\bm{v} \stackrel{\imath}{\leftarrow} w },\pib),\mathfrak{A}} \left( \Psi_{\imath, \bm{v} } \neq w \right)  \implies \frac{1}{|\mathcal{\overline{V}}|} \sum_{w \in \overline{\mathcal{V}}} \E_{(\mathcal{M}_{\bm{v} \stackrel{\imath}{\leftarrow} w },\pib),\mathfrak{A}}   \left[ N^\tau_{h_*}(s_*,a_*) \right] \ge \frac{(1-\delta) \log|\overline{\mathcal{V}}| - \log 2}{2(\epsilon')^2}.
\end{align*}
Since the derivation is carried out for every $\imath \in \mathcal{I}$ and $\bm{v} \in \overline{\mathcal{V}}^{\mathcal{I}}$, we can perform the summation over $\imath$ and the average over $\bm{v}$:
\begin{align*}
\sum_{\imath \in \mathcal{I}} \frac{1}{|\mathcal{\overline{V}}|^{|\mathcal{I}|}} \sum_{\bm{v} \in \overline{\mathcal{V}}^{\mathcal{I}}} \frac{1}{|\mathcal{\overline{V}}|} \sum_{w \in \overline{\mathcal{V}}} \E_{(\mathcal{M}_{\bm{v} \stackrel{\imath}{\leftarrow} w },\pib),\mathfrak{A}}   \left[ N^\tau_{h_*}(s_*,a_*) \right] & =  \frac{1}{|\mathcal{\overline{V}}|^{|\mathcal{I}|}} \sum_{\bm{v} \in \overline{\mathcal{V}}^{\mathcal{I}}}\sum_{\imath \in \mathcal{I}} \E_{(\mathcal{M}_{\bm{v}},\pib),\mathfrak{A}}   \left[ N^\tau_{h_*}(s_*,a_*) \right] \\
& \ge \overline{SAH}\frac{(1-\delta) \log|\overline{\mathcal{V}}| - \log 2}{2(\epsilon')^2}.
\end{align*}
Notice that we get a guarantee on a mean under the uniform distribution of the instances of the sample complexity. Thus, there must exist one $\bm{v}^{\text{hard}} \in \mathcal{\overline{V}}$ such that:
\begin{align*}
\E_{(\mathcal{M}_{\bm{v}^{\text{hard}},\pib}),\mathfrak{A}}  [{\tau}] \ge \sum_{\imath \in \mathcal{I}}  \E_{(\mathcal{M}_{\bm{v}^{\text{hard}},\pib} ),\mathfrak{A} } \left[ N^\tau_{h_*}(s_*,a_*) \right] \ge \overline{SAH}\frac{(1-\delta) \log|\overline{\mathcal{V}}| - \log 2}{2(\epsilon')^2}.
\end{align*}
Then, we select $\delta \le 1/2$, recall that $|\overline{\mathcal{V}}| \ge 2^{ \overline{S}/5}$, we get:
\begin{align*}
\E_{(\mathcal{M}_{\bm{v}^{\text{hard}},\pib}),\mathfrak{A}}  [{\tau}] \ge \overline{SAH}\frac{ \overline{S}/10 - \log 2}{2(\epsilon')^2} =   \overline{SAH}\frac{ (H - \overline{H} - 2)^2 (\overline{S}/10 - \log 2)}{8912\epsilon^2} 
\end{align*}
The number of states is given by $S = |\Ss| = 2\overline{S}+2$, the number of actions is given by $A = |\As| = \overline{A}+1$. Let us first consider the time-homogeneous case, \ie $\overline{H} = 1$, for $S \ge 16$, $A \ge 2$, $H \ge 130$, we have:
\begin{align*}
\E_{(\mathcal{M}_{\bm{v}^{\text{hard}},\pib}),\mathfrak{A}}  [{\tau}] \ge  \frac{S^2AH^2}{2560 \epsilon^2}.
\end{align*}
For the time inhomogeneous case, we select $\overline{H} = H/2$, to get, under the same conditions:
\begin{align*}
	\E_{(\mathcal{M}_{\bm{v}^{\text{hard}},\pib}),\mathfrak{A}}  [{\tau}] \ge  \frac{S^2AH^3}{5120\epsilon^2}.
\end{align*}
\end{proof}

\clearpage

\subsection{Proofs of Section~\ref{sec:algorithm}}

\scUSIRL*

\begin{proof}
We start with the case in which the transition model is time-inhomogeneous. In this case, we introduce the following good event:
\begin{align*}
\mathcal{E}\coloneqq \left\{\forall t \in \Nat,\,\forall (s,a,h) \in \SAHs :\,D_{\text{KL}}\Bigl(\widehat{p}_h^t(\cdot|s,a),p_h(\cdot|s,a)\Bigl)\le\frac{\beta\bigl(n_h^t(s,a),\delta\bigl)}{n_h^t(s,a)}\right\}   , 
\end{align*}
where $p_h$ is the true transition model and $\widehat{p}_h^t$ is its estimate via Equation~\eqref{eq:empiricalMDP} at time $t$. Thanks to Lemma~\ref{lemma:bound_good_event}, we have that $\Prob_{(\mathcal{M},\pi^E),\mathfrak{A}}(\mathcal{E}) \ge 1-\delta$. Thus, under the good event $\mathcal{E}$, we apply Theorem~\ref{thr:errorProp2}:
\begin{align*}
	\mathcal{H}_{d^{\text{G}}}(\mathcal{R},\widehat{\mathcal{R}}^\tau) & \le \frac{ 2 \rho^{\text{G}}((\mathcal{M},\pi^E),(\widehat{\mathcal{M}}^t,\widehat{\pi}^{E,t}))}{1+\rho^{\text{G}}((\mathcal{M},\pi^E),(\widehat{\mathcal{M}}^t,\widehat{\pi}^{E,t}))} \\
	& \le 2 \rho^{\text{G}}((\mathcal{M},\pi^E),(\widehat{\mathcal{M}},\widehat{\pi}^E)) \\
	& \le  2 \max_{(s,a,h)\in \SAHs} (H-h+1) \left( \left| \indic_{\{\pi^E_h(a|s) = 0\}} - \indic_{\{\widehat{\pi}^E_h(a|s) = 0\}} \right| + \left\| p_h(\cdot|s,a) - \widehat{p}_h(\cdot|s,a) \right\|_1  \right) \\
	& \le 2 \max_{(s,a,h)\in \SAHs} (H-h+1) \left\| p_h(\cdot|s,a) - \widehat{p}_h(\cdot|s,a) \right\|_1  \\
	& \le 2\sqrt{2} \max_{(s,a,h)\in \SAHs} (H-h+1)\sqrt{D_{\text{KL}}\Bigl(\widehat{p}_h^t(\cdot|s,a),p_h(\cdot|s,a)\Bigl)} = \max_{(s,a,h)\in \SAHs} \mathcal{C}_h^t(s,a),
\end{align*}
where we exploited the fact that the expert's policy is known in the last but one passage and used Pinsker's inequality in the last passage. When the \texttt{US-IRL} stops we have that $\max_{(s,a,h)\in \SAHs} \mathcal{C}_h^t(s,a) \le \epsilon$ and, consequently, for all $(s,a,h) \in \SAHs$ we have:
\begin{align*}
\max_{(s,a,h)\in \SAHs} \mathcal{C}_h^t(s,a) = \max\limits_{(s,a,h)\in \mathcal{S}\times\mathcal{A}\times\dsb{H}} 2 \sqrt{2}(H-h+1)\sqrt{\frac{\beta\bigl(n_h^t(s,a),\delta\bigl)}{n_h^t(s,a)}} \le \epsilon.
\end{align*}
Thus, the algorithm stops at the smallest $t$ such that:
\begin{align*}
\implies n_h^t(s,a)  \ge \frac{8(H-h+1)^2\beta\bigl(n_h^t(s,a),\delta\bigl)}{\epsilon^2}  = \frac{8(H-h+1)^2}{\epsilon^2} \left(  \log (SAH/\delta) + (S-1) \log (e (1+n_h^t(s,a)/(S-1)) \right).
\end{align*}
Thus, by applying Lemma 15 of~\cite{KaufmannMDJLV21}, we obtain:
\begin{align*}
	 n_h^\tau(s,a) \le  \frac{8(H-h+1)^2}{\epsilon^2} \left(\log \left(\frac{SAH}{\delta}\right) + (S-1) \log \left( \frac{8e(H-h+1)^2}{(S-1)\epsilon^2} \left(  \log \left(\frac{SAH}{\delta}\right) +  4e\right) \right)  \right).
\end{align*}
By recalling that $\tau = SAH n_h^\tau(s,a)$, and bounding $H-h+1\le H$, we obtain:
\begin{align*}
	\tau \le \frac{8 H^3 SA}{\epsilon^2}  \left(\log \left(\frac{SAH}{\delta}\right) + (S-1) \log \left(\frac{e}{S-1}+ \frac{8 e H^2}{(S-1)\epsilon^2} \left(  \log \left(\frac{SAH}{\delta}\right) +  4e\right) \right)  \right).
\end{align*}

If the transition model is time-homogeneous, we suppress the subscript $h$ and the algorithm \texttt{US-IRL}, will merge together all the samples collected at different stages $h$. Let us define $n^t(s,a) = \sum_{h=1}^H n^t_h(s,a)$ and $n^t(s,a,s') = \sum_{h=1}^H n^t_h(s,a,s')$. Now the transition model will be estimated straightforwardly as follows:
\begin{align*}
\widehat{p}^t(s'|s,a) \coloneqq \begin{cases}
									\frac{n^t(s,a,s')}{n^t(s,a)} & \text{if } n^t(s,a) > 0 \\
									\frac{1}{S} & \text{otherwise}								\end{cases}.
\end{align*}
Let us consider now the following good event:
\begin{align*}
\widetilde{\mathcal{E}} \coloneqq \left\{\forall t \in \Nat,\,\forall (s,a) \in \SAs :\,D_{\text{KL}}\Bigl(\widehat{p}^t(\cdot|s,a),p(\cdot|s,a)\Bigl)\le\frac{\widetilde{\beta}\bigl(n^t(s,a),\delta\bigl)}{n^t(s,a)}\right\}.
\end{align*}
Thanks to Lemma~\ref{lemma:bound_good_event}, we have that $\Prob_{(\mathcal{M},\pi^E),\mathfrak{A}}(\widetilde{\mathcal{E}}) \ge 1-\delta$. Thus, in such a case, thanks to Theorem~\ref{thr:errorProp2}, we have:
\begin{align*}
\mathcal{H}_{d^{\text{G}}}(\mathcal{R},\widehat{\mathcal{R}}^\tau) \le 2\sqrt{2} \max_{(s,a,h)\in \SAHs} (H-h+1)\sqrt{D_{\text{KL}}\Bigl(\widehat{p}_h^t(\cdot|s,a),p_h(\cdot|s,a)\Bigl)} = \max_{(s,a,h)\in \SAHs} \widetilde{\mathcal{C}}_h^t(s,a).
\end{align*}
The algorithm, therefore, stops as soon as:
\begin{align*}
\max_{(s,a,h)\in \SAHs} \widetilde{\mathcal{C}}_h^t(s,a) = \max_{(s,a,h)\in \SAHs} 2\sqrt{2} (H-h+1)\sqrt{\frac{\widetilde{\beta}\bigl(n^t(s,a),\delta\bigl)}{n^t(s,a)}} = \max_{(s,a)\in \SAs} 2\sqrt{2} H \sqrt{\frac{\widetilde{\beta}\bigl(n^t(s,a),\delta\bigl)}{n^t(s,a)}} \le \epsilon.
\end{align*}
This allows us to compute the maximum value of $n^\tau(s,a)$:
\begin{align*}
n^\tau(s,a) \le \frac{8H^2}{\epsilon^2} \left(\log \left(\frac{SA}{\delta}\right) + (S-1) \log \left(\frac{e}{S-1}+ \frac{8eH^2}{(S-1)\epsilon^2} \left(  \log \left(\frac{SA}{\delta}\right) +  4e\right) \right)  \right).
\end{align*}
Recalling that $\tau = SA n^\tau(s,a)$, we obtain:
\begin{align*}
	\tau \le \frac{8 H^2 SA}{\epsilon^2}  \left(\log \left(\frac{SA}{\delta}\right) + (S-1) \log \left( \frac{8 e H^2}{(S-1)\epsilon^2} \left(  \log \left(\frac{SA}{\delta}\right) +  4e\right) \right)  \right).
\end{align*}
\end{proof}

\begin{lemma}\label{lemma:bound_good_event}
The following statements hold:
\begin{itemize}
	\item for $\beta\bigl(n,\delta\bigl)=\log(SAH/\delta)+(S-1)\log\bigl(e(1+n/(S-1)\bigl)$, we have that $\mathbb{P}(\mathcal{E})\ge 1-\delta$;
	\item for $\widetilde{\beta}\bigl(n,\delta\bigl)=\log(SA/\delta)+(S-1)\log\bigl(e(1+n/(S-1)\bigl)$, we have that $\mathbb{P}(\widetilde{\mathcal{E}})\ge 1-\delta$.
\end{itemize}
\end{lemma}

\begin{proof}
%\citep[][Proposition 1]{JonssonKMDLV20}\label{lemma:jonsson}
Let us start with the first statement.
Similarly to Lemma 10 of ~\cite{KaufmannMDJLV21}, we apply first a union bound and then technical Proposition 1 of~\cite{JonssonKMDLV20} (also reported as Lemma~\ref{lemma:jonsson} for completeness) to concentrate the KL-divergence:
\begin{align*}
\mathbb{P}(\mathcal{E}^c)=&\mathbb{P}\Biggl(\exists t \in \Nat,\,\exists (s,a,h) \in \SAHs:\,D_{\text{KL}}\Bigl(\widehat{p}_h^t(\cdot|s,a),p_h(\cdot|s,a)\Bigl)\ge\frac{\beta\bigl(n_h^t(s,a),\delta\bigl)}{n_h^t(s,a)}\Biggl)\\
\le&\sum\limits_{h\in\dsb{H}}\sum\limits_{(s,a)\in\mathcal{S}\times\mathcal{A}}\mathbb{P}\Biggl(\exists t \in \Nat:\,D_{\text{KL}}\Bigl(\widehat{p}_h^t(\cdot|s,a),p_h(\cdot|s,a)\Bigl)\ge\frac{\beta\bigl(n_h^t(s,a),\delta\bigl)}{n_h^t(s,a)}\Biggl)\\
\le&\sum\limits_{h\in\dsb{H}}\sum\limits_{(s,a)\in\mathcal{S}\times\mathcal{A}}\frac{\delta}{SAH}=\delta.
\end{align*}
The proof of the second statement is analogous having simply observed that the union bound has to be performed over $\SAs$ only.
\end{proof}

\section{Examples of Section~\ref{sec:nonReg}}\label{apx:examples}
In this appendix, we provide a detailed derivations of the examples presented in Section~\ref{sec:nonReg}.

\exampleA*
\begin{proof}
For the \MDPRtext $\mathcal{M}$, in order to make $\pi^E_1(s_0)=a_1$ optimal, we have to enforce:
\begin{align*}
	& r_1(s_0) + \frac{2+\epsilon}{4} r_2(s_+) + \frac{2-\epsilon}{4} r_2(s_-) \ge r_1(s_0) + \frac{1}{2} r_2(s_+) + \frac{1}{2} r_2(s_-) \\
	& \implies r_2(s_+) \ge r_2(s_-).
\end{align*}
Similarly, to make $\widehat{\pi}^E_1(s_0)=a_1$, we have for $\widehat{\mathcal{M}}$:
\begin{align*}
	& \widehat{r}_1(s_0) + \frac{2-\epsilon}{4} \widehat{r}_2(s_+) + \frac{2+\epsilon}{4} \widehat{r}_2(s_-) \ge \widehat{r}_1(s_0) + \frac{1}{2} \widehat{r}_2(s_+) + \frac{1}{2} \widehat{r}_2(s_-) \\
	& \implies \widehat{r}_2(s_+) \le \widehat{r}_2(s_-).
\end{align*}
Thus, suppose, we set $r_2(s_-) = 1$ and $r_2(s_+) =-1$, we have:
\begin{align*}
\mathcal{H}_{d^{\text{G}}}(\mathcal{R}_{\text{state}},\widehat{\mathcal{R}}_{\text{state}}) \ge \min_{\substack{\widehat{r}_2(s_-), \widehat{r}_2(s_+) \in [-1,1] \\ \widehat{r}_2(s_+) \le \widehat{r}_2(s_-)}} \max \left\{|1-\widehat{r}_2(s_-) |,|-1-\widehat{r}_2(s_+) | \right\} = 1.
\end{align*}
\end{proof}

\exampleB*
\begin{proof}
Consider the \MDPRtext $\mathcal{M}$ and we set $r(s_0,a_1) = 1$, $r(s_0,a_2) = 1 - \epsilon/12$, and $r(s_1,a)=1/2$ for $a \in \{a_1,a_2\}$. We immediately observe that $\pi^E$ is optimal since for $h=2$, $r(s_0,a_1) \ge r(s_0,a_2)$ and for $h=1$:
\begin{align*}
	&r(s_0,a_2) + \frac{2+\epsilon}{4} r(s_0,a_1) + \frac{2-\epsilon}{4} r(s_1,a) \ge r(s_0,a_1) + \frac{1}{2}r(s_0,a_1) + \frac{1}{2}r(s_1,a) \\
	&\iff r(s_0,a_2) + \left(\frac{\epsilon}{4}-1\right) r(s_0,a_1) - \frac{\epsilon}{4} r(s_1,a) \ge 0\\
	& \iff 1 - \frac{\epsilon}{12} + \frac{\epsilon}{4} -1 - \frac{\epsilon}{8} \ge 0.
\end{align*}
Consider now the alternative \MDPRtext $\widehat{\mathcal{M}}$, we have to enforce the following two conditions:
\begin{align}
	& \widehat{r}(s_0,a_1) \ge \widehat{r}(s_0,a_2), \label{eq:p002}\\
	& \widehat{r}(s_0,a_2) + \frac{2-\epsilon}{4} \widehat{r}(s_0,a_1) + \frac{2+\epsilon}{4} \widehat{r}(s_1,a) \ge \widehat{r}(s_0,a_1) + \frac{1}{2}\widehat{r}(s_0,a_1) + \frac{1}{2}\widehat{r}(s_1,a) \notag\\
	& \iff \widehat{r}(s_0,a_2) - \left(\frac{\epsilon}{4} + 1 \right)\widehat{r}(s_0,a_1) + \frac{\epsilon}{4} \widehat{r}(s_1,a) \ge 0. \label{eq:p001}
\end{align}
The way of enforcing Equation~\eqref{eq:p002} that is less constraining for Equation~\eqref{eq:p001} is setting $\widehat{r}(s_0,a_1) = \widehat{r}(s_0,a_2)$, to get:
\begin{align*}
- \frac{\epsilon}{4} \widehat{r}(s_0,a_1) + \frac{\epsilon}{4} \widehat{r}(s_1,a) \ge 0 \iff   \widehat{r}(s_1,a) \ge \widehat{r}(s_0,a_1) .
\end{align*}
This implies:
\begin{align*}
	\mathcal{H}_{d^{\text{G}}}(\mathcal{R}_{\text{hom}},\widehat{\mathcal{R}}_{\text{hom}}) \ge \min_{\substack{\widehat{r}(s_1,a) , \widehat{r}(s_0,a_1) \in [-1,1]\\ \widehat{r}(s_1,a) \ge \widehat{r}(s_0,a_1)}} \max\left\{\left| 1- \widehat{r}(s_0,a_1) \right|,  \left| \frac{1}{2}- \widehat{r}(s_1,a) \right| \right\} \ge \frac{1}{4},
\end{align*}
by setting $\widehat{r}(s_0,a_1)=\widehat{r}(s_1,a)=1/4$.
\end{proof}

\exampleC*
\begin{proof}
	Concerning the \MDPRtext $\mathcal{M}$, we observe that by setting $r_1(s_0,a_1)=1$, $r_1(s_0,a_2)=-1$, and $r_h(s_+,a)=-r_h(s_-, a)=1$ for $a \in \{a_1,a_2\}$ and $h \in \dsb{2,H}$, the policy $\pi^E$ is optimal. In particular, in state-stage pair $(s_0,1)$ the suboptimality gap is given by $\beta_1 = 2 + 2\epsilon (H-1)$. To enforce the optimality of $\widehat{\pi}^E = \pi^E$ in the \MDPRtext $\widehat{\mathcal{M}}$, we have:
	\begin{align*}
	& \widehat{r}_1(s_0,a_1) +  \sum_{h=2}^H \frac{1}{2} \widehat{r}_h(s_+,a_1) + \frac{1}{2} \widehat{r}_h(s_-,a_1) \ge \widehat{r}_1(s_0,a_2) +  \sum_{h=2}^H \frac{1}{2} \widehat{r}_h(s_+,a_1) + \frac{1}{2} \widehat{r}_h(s_-,a_1) + \beta_1\\
& \iff \widehat{r}_1(s_0,a_1)- \widehat{r}_1(s_0,a_2) \ge \beta_1.
	\end{align*}
	Thus, if $\beta_1 \ge 2$, we have that the feasible set $\widehat{\mathcal{R}}_{\beta\text{-sep}}$ is empty. Thus, we select $H \ge 1+1/\epsilon$ to have $\beta_1 \ge 4$.
\end{proof}

\section{Unknown Expert's Policy $\pi^E$}\label{apx:notKnownPi}
In this appendix, we extend the lower bounds and the algorithm for the case in which the expert's policy is unknown. Clearly, if the expert's policy is deterministic, under the generative model setting, its estimation is trivial as it suffices to query every state and stage (resp. state) exactly once for time-inhomogeneous (resp. time-homogeneous) policies, leading to $\E_{(\mathcal{M},\pi^E),\mathfrak{A}} \left[\tau\right] = HS$ (resp. $\E_{(\mathcal{M},\pi^E),\mathfrak{A}} \left[\tau\right] = S$). Thus, we consider a more general setting in which the expert's policy can be stochastic (still being optimal). Specifically, we consider the following assumption.

\begin{ass}\label{ass:policy}
There exists a known constant $\pi_{\min} \in (0,1]$ such that every action played by the expert's policy $\pi^E$ is played with at least probability $\pi_{\min}$:
\begin{align*}
	\forall (s,a,h) \in \SAHs \,:\, \pi^E_h(a|s) \in \{0\} \cup [\pi_{\min},1].
\end{align*}
\end{ass}

Intuitively, Assumption~\ref{ass:policy} formalizes a form of identifiability for the policy. As already mentioned in Section~\ref{sec:regul}, what matters for learning the feasible reward set is whether an action is played by the agent (not the corresponding probability). Assumption~\ref{ass:policy} enforces that every optimal action must be played with a minimum (known) non-null probability $\pi_{\min}$. We shall show that if this assumption is violated, the problem becomes non-learnable.

\subsection{Lower Bound}
The following result provides a lower bound for learning the feasible reward set according to the PAC requirement of Definition~\eqref{defi:pac} when the expert's policy is unknown, but the transition model is known. Clearly, one can combine this result with the ones of Section~\ref{sec:lb} to address the setting in which both the expert's policy and the transition model are unknown.

\begin{restatable}{thr}{rInfLB1}\label{thr:rInfLBpol}
Let $\mathfrak{A} = (\mu,\tau)$ be an $(\epsilon,\delta)$-PAC algorithm for $d^{\text{G}}$-IRL. Then, there exists an IRL problem $(\mathcal{M},\pi^E)$ where $\pi^E$ fulfills Assumption~\ref{ass:policy} such that, if $\epsilon \le 1/2$, $\delta < 1/16$, $S\ge 7$, $A\ge 2$, and $H \ge 3$, the number of samples $N$ is lower bounded in expectation by:
\begin{itemize}[topsep=0pt, noitemsep, leftmargin=*]
	\item if the expert's policy $\pi^E$ is time-inhomogeneous:
		\begin{align*}
		\E_{(\mathcal{M},\pi^E),\mathfrak{A}} \left[\tau\right] \ge \frac{SH}{8\log \frac{1}{1-\pi_{\min}} } \log \left(\frac{1}{\delta}\right).
	\end{align*}
	\item if the expert's policy $\pi^E$ is time-homogeneous:
 	\begin{align*}
\E_{(\mathcal{M},\pi^E),\mathfrak{A}} \left[\tau\right] \ge \frac{S}{4\log \frac{1}{1-\pi_{\min}} } \log \left(\frac{1}{\delta} \right);
	\end{align*}

\end{itemize}
\end{restatable}

Before presenting the proof, let us comment the result. We observe that when Assumption~\ref{ass:policy} is violated, \ie $\pi_{\min} \rightarrow 0$, the sample complexity lower bound degenerates to infinity, proving that the problem become non-learnable.

\begin{proof}
\textbf{Step 1: Instances Construction}~~
The hard \MDPRtext instances are depicted in Figure~\ref{fig:hardBB3} in a semi-formal way. The state space is given by $\Ss=\{s_{\text{start}}, s_{\text{root}}, s_1,\dots, s_{\overline{S}}, s_{\text{sink}}\}$ and the action space is given by $\As = \{a_0,a_1,\dots,a_{\overline{A}}\}$. The transition model is described below and the horizon is $H \ge 3$. We introduce the constant $\overline{H} \in \dsb{H}$, whose value will be chosen later. Let us observe, for now, that if $\overline{H} = 1$, the transition model is time-homogeneous.

The agent begins in state $s_{\text{start}}$, where every action has the same effect. Specifically, if the stage $h < \overline{H}$, then there is  probability $1/2$ to remain in $s_{\text{start}}$ and a probability $1/2$ to transition to $s_{\text{root}}$. Instead, if $h \ge \overline{H}$, the state transitions to $s_{\text{root}}$ deterministically. From state $s_{\text{root}}$, every action has the same effect and the state transitions with equal probability $1/\overline{S}$ to a state $s_i$ with $i \in \dsb{\overline{S}}$. In all states $s_i$, apart from a specific one, i.e., state $s_*$, the expert's policy plays action $a_0$ deterministically, \ie $\pi^E_h(a_0|s_i)=1$ and the state transitions deterministically to $s_{\text{sink}}$. In state $s_*$ the expert's policy plays $a_0$ as the other ones if the stage $h \neq h_*$, where $h_* \in \dsb{H}$ is a predefined stage. If, instead, $h=h_*$, the expert's action plays $a_0$ w.p. $1-\pi_{\min}$ and a specific action $a_*$ w.p. $\pi_{\min} \in [0,1/2]$. Then, the transition is deterministic to state $s_{\text{sink}}$.
Notice that, having fixed $\overline{H}$, the possible values of $h^*$ are $\{3,\dots,2+\overline{H}\}$. State $s_{\text{sink}}$ is an absorbing state. 

Let us consider the base instance ${\pi}_{0}$ in which the expert's policy always plays action $a_0$ deterministically.\footnote{In this construction, the \MDPRtext does not change across the instances, but what changes is the expert's policy. Thus, we parametrize the instances through the policy rather than the \MDPRtext.} Additionally, by varying the pair $\ell \coloneqq (s_*,h_*) \in \{s_1,\dots,s_{\overline{S}}\} \times \dsb{3,\overline{H}+2} \eqqcolon \mathcal{J}$, we can construct the  class of instances denoted by $\mathbb{M} = \{ \pi_{\ell} : \ell \in \{0\} \cup \mathcal{J} \}$.

\begin{figure}[h!]
\centering
\begin{tikzpicture}[node distance=4cm]
\node[state] (s0) {$s_{\text{start}}$};
\node[state, below=2cm of s0] (s00) {$s_{\text{root}}$};

\node[state, below left of = s00, draw=none] (s1) {$\dots$};
\node[state, below right of = s00, draw=none] (sS) {$\dots$};
\node[state, below=2cm of s00, blue] (sj) {$s_*$};
\node[state, left of = s1] (s1) {$s_1$};
\node[state, right of = sS] (sS) {$s_{\overline{S}}$};

\node[state, below=3cm of sj] (sink) {$s_{\text{sink}}$};

\draw (sink) edge[->, loop below, double] node{} (sink);

\draw (s0) edge[->, loop above, solid, double] node{$h < \overline{H}$ w.p. $\frac{1}{2}$ } (s0);
\draw (s0) edge[->, solid, right, double] node{w.p. $\frac{1}{\overline{S}}$ or  $h \ge \overline{H}$} (s00);
\draw (s00) edge[->, solid, right, double] node{w.p. $\frac{1}{\overline{S}}$} (sj);
\draw (s00) edge[->, solid, above left, double] node{w.p. $\frac{1}{\overline{S}}$} (s1);
\draw (s00) edge[->, solid, above right, double] node{regardless the action w.p. $\frac{1}{\overline{S}}$} (sS);

\draw (sS) edge[->, solid, below right,] node{play $a_0$ w.p. $1$} (sink);

\draw (s1) edge[->, solid, below left] node{play $a_0$ w.p. $1$} (sink);

\draw (sj) edge[->, solid, right, bend left, blue, very thick] node[pos=0.2]{$h=h_*$ play $a_*$ w.p. $\pi_{\min}$} (sink);

\draw (sj) edge[->, solid, left, bend right, densely dashed, very thick] node[pos=0.2]{$h=h_*$ play w.p. $1-\pi_{\min}$} (sink);

\end{tikzpicture}	
\caption{Semi-formal representation of the the hard instances \MDPRtext used in the proof of Theorem~\ref{thr:rInfLBpol}.}\label{fig:hardBB3}
\end{figure}

\textbf{Step 2: Feasible Set Computation}~~
Let us consider an instance $\pi_\ell \in \mathbb{M}$, we now seek to provide a lower bound to the Hausdorff distance $\mathcal{H}_{d^{\text{G}}}\left( \mathcal{R}_{\pi_0}, \mathcal{R}_{\pi_{\ell}}\right)$. To this end, we focus on the pair $\ell = (s_*,h_*)$ and we enforce the convenience of both actions $a_0$ and $a_*$ over the other actions. Since both actions are played with non-zero probability by the expert's policy, their value function must be the same. Let us denote with $r^{\ell} \in \mathcal{R}_{\pi_{\ell}}$, we must have for all $a_j \not\in \{a_0,a_*\}$:
\begin{align*}
	&  r_{h_*}^{\ell}(s_*,a_0)  + \sum_{l=h_*+1}^H r_l^{\ell}(s_{\text{sink}}) \ge r_{h_*}^{\ell}(s_*,a_j)  + \sum_{l=h_*+1}^H r_l^{\ell}(s_{\text{sink}}) \\
	& \implies r_{h_*}^{\ell}(s_*,a_0) \ge r_{h_*}^{\ell}(s_*,a_j),\\
	& r_{h_*}^{\ell}(s_*,a_0)  + \sum_{l=h_*+1}^H r_l^{\ell}(s_{\text{sink}}) = r_{h_*}^{\ell}(s_*,a_*)  + \sum_{l=h_*+1}^H r_l^{\ell}(s_{\text{sink}}) \\
	& \implies r_{h_*}^{\ell}(s_*,a_0) = r_{h_*}^{\ell}(s_*,a_*) .
\end{align*}
Consider now the base instance $\pi_{0}$ and denote with $r^{0} \in \mathcal{R}_{\pi_0}$. Here we have to enforce the convenience of action $a_0$ over all the others, including $a_*$:
\begin{align*}
	&  r_{h_*}^{0}(s_*,a_0)  + \sum_{l=h_*+1}^H r_l^{\ell}(s_{\text{sink}}) \ge r_{h_*}^{0}(s_*,a_j)  + \sum_{l=h_*+1}^H r_l^{\ell}(s_{\text{sink}}) \\
	& \implies r_{h_*}^{0}(s_*,a_0) \ge r_{h_*}^{0}(s_*,a_j),\\
	& r_{h_*}^{0}(s_*,a_0)  + \sum_{l=h_*+1}^H r_l^{0}(s_{\text{sink}}) \ge r_{h_*}^{0}(s_*,a_*)  + \sum_{l=h_*+1}^H r_l^{0}(s_{\text{sink}}) \\
	& \implies r_{h_*}^{0}(s_*,a_0) \ge r_{h_*}^{0}(s_*,a_*) .
\end{align*}
In order to lower bound the Hausdorff distance, we perform a valid assignment of the rewards for the base instance:
\begin{align*}
	 r_{h_*}^{0}(s_*,a_0) = 1, \, r_{h_*}^{0}(s_*,a_*) = -1, \, r_{h_*}^{0}(s_*,a_j) = -1.
\end{align*}
Thus, the Hausdorff distance can be bounded as follows, having renamed, for convenience $x = r_{h_*}^{\ell}(s_*,a_0)$ and $y =  r_{h_*}^{\ell}(s_*,a_*)$:
\begin{align*}
	\mathcal{H}_{d^{\text{G}}}(\mathcal{R}_{\pi_0}, \mathcal{R}_{\pi_{\ell}}) \ge \min_{\substack{x,y \in [-1,1] \\ x = y}} \max\{|x-1|, |y+1|\} = 1.
\end{align*}

\textbf{Step 3: Lower bounding Probability}~~
Let us consider an $(\epsilon,\delta)$-correct algorithm $\mathfrak{A}$ that outputs the estimated feasible set $\widehat{\mathcal{R}}$. Thus, for every $\imath \in \mathcal{J}$, we can lower bound the error probability:
\begin{align*}
\delta & \ge \sup_{\text{all }\mathcal{M}\text{ \MDPRtext} \text{ and expert policies } \pi} \Prob_{(\mathcal{M},\pib),\mathfrak{A}} \left( \mathcal{H}_{d^{\text{G}}}\left( \mathcal{R}_{\pi}, \widehat{\mathcal{R}} \right) \ge \frac{1}{2} \right) \\
& \ge \sup_{\pi \in \mathbb{M}} \Prob_{(\mathcal{M},\pib),\mathfrak{A}} \left( \mathcal{H}_{d^{\text{G}}}\left( \mathcal{R}_{\pi}, \widehat{\mathcal{R}} \right) \ge \frac{1}{2} \right) \\
& \ge \max_{\ell \in \{0,\imath\} }  \Prob_{(\mathcal{M},\pi_\ell),\mathfrak{A}} \left( \mathcal{H}_{d^{\text{G}}}\left( \mathcal{R}_{\pi_\ell}, \widehat{\mathcal{R}} \right) \ge \frac{1}{2} \right).
\end{align*}
For every $\imath \in \mathcal{J}$, let us define the \emph{identification function}:
\begin{align*}
	\Psi_{\imath} \coloneqq \argmin_{\ell \in \{0,\imath\}} \mathcal{H}_{d^{\text{G}}}\left( \mathcal{R}_{\pi_\ell}, \widehat{\mathcal{R}} \right).
\end{align*}
Let $\jmath \in \{0,\imath\}$. If $\Psi_{\imath} = \jmath$, then, $\mathcal{H}_{d^{\text{G}}}(\mathcal{R}_{\pi_{\Psi_{\imath}}}, \mathcal{R}_{\pi_\jmath}) =0$. Otherwise, if $\Psi_{\imath} \neq \jmath$, we have:
\begin{align*}
	\mathcal{H}_{d^{\text{G}}}\left(\mathcal{R}_{\pi_{\Psi_{\imath}}}, \mathcal{R}_{\pi_\jmath} \right)& \le \mathcal{H}_{d^{\text{G}}}\left(\mathcal{R}_{\pi_{\Psi_{\imath}}}, \widehat{\mathcal{R}}\right) + \mathcal{H}_{d^{\text{G}}}\left(\widehat{\mathcal{R}}, \mathcal{R}_{\pi_\jmath} \right) \le 2\mathcal{H}_{d^{\text{G}}}\left(\widehat{\mathcal{R}}, \mathcal{R}_{\pi_\jmath} \right) ,
\end{align*}
where the first inequality follows from triangular inequality and the second one from the definition of identification function $\Psi_{\imath}$. From Equation~\eqref{eq:ineqEps}, we have that $\mathcal{H}_{d^{\text{G}}}\left(\mathcal{R}_{\pi_{\Psi_{\imath}}}, \mathcal{R}_{\pi_\jmath} \right) \ge 1$. Thus, it follows that $\mathcal{H}_{d^{\text{G}}}\left(\widehat{\mathcal{R}}, \mathcal{R}_{\pi_\jmath} \right) \ge \frac{1}{2}$. This implies the following inclusion of events for $\jmath \in \{0,\imath\}$:
\begin{align*}
 \left\{\mathcal{H}_{d^{\text{G}}}\left(\widehat{\mathcal{R}}, \mathcal{R}_{\pi_\jmath} \right) \ge \frac{1}{2} \right\} \supseteq \left\{ \Psi_{\imath} \neq \jmath \right\}.
\end{align*}
Thus, we can proceed by lower bounding the probability:
\begin{align*}
\max_{\ell \in \{0,\imath\} }  \Prob_{(\mathcal{M}_\ell,\pib),\mathfrak{A}} \left( \mathcal{H}_{d^{\text{G}}}\left( \mathcal{R}_{\pi_\ell}, \widehat{\mathcal{R}} \right) \ge \frac{1}{2} \right) & \ge \max_{\ell \in \{0,\imath\} } \Prob_{(\mathcal{M}_\ell,\pib),\mathfrak{A}} \left( \Psi_{\imath} \neq \ell \right) \\
& \ge \frac{1}{2} \left[ \Prob_{(\mathcal{M}_0,\pib),\mathfrak{A}}\left(  \Psi_{\imath} \neq 0 \right) + \Prob_{(\mathcal{M}_\imath,\pib),\mathfrak{A}}\left( \Psi_{\imath} \neq \imath \right) \right] \\
& =  \frac{1}{2} \left[ \Prob_{(\mathcal{M}_0,\pib),\mathfrak{A}}\left(  \Psi_{\imath} \neq 0 \right) + \Prob_{(\mathcal{M}_\imath,\pib),\mathfrak{A}}\left( \Psi_{\imath} = 0 \right) \right],
\end{align*}
where the second inequality follows from the observation that $\max\{a,b\} \ge \frac{1}{2} (a + b)$ and the equality from observing that $\Psi_{\imath} \in \{0,\imath\}$. We can now apply the Bretagnolle-Huber inequality~\citep[][Theorem 14.2]{lattimore2020bandit} (also reported in Theorem~\ref{thr:BHIneq} for completeness) with $\mathbb{P} = \Prob_{(\mathcal{M}_0,\pib),\mathfrak{A}}$, $\mathbb{Q} = \Prob_{(\mathcal{M}_0,\pib),\mathfrak{A}}$, and $\mathcal{A} = \{\Psi_{\imath} \neq 0\}$:
\begin{align*}
\Prob_{(\mathcal{M}_0,\pib),\mathfrak{A}}\left(  \Psi_{\imath} \neq 0 \right) + \Prob_{(\mathcal{M}_\imath,\pib),\mathfrak{A}}\left( \Psi_{\imath} = 0 \right) \ge \frac{1}{2} \exp \left( - D_{\text{KL}} \left( \Prob_{(\mathcal{M}_0,\pib),\mathfrak{A}}, \Prob_{(\mathcal{M}_\imath,\pib),\mathfrak{A}}\right) \right).
\end{align*}

\textbf{Step 4: KL-divergence Computation}~~
Let $\mathcal{M}\in \mathbb{M}$, we denote with $\mathbb{P}_{\mathfrak{A},\mathcal{M},\pib}$ the joint probability distribution of all events realized by the execution of the algorithm in the \MDPRtext (the presence of $p$ is irrelevant as it does not change across the different instances):
\begin{align*}
\Prob_{(\mathcal{M},\pib),\mathfrak{A}} = \prod_{t=1}^\tau \rho_t(s_t, a_t, h_t |H_{t-1}) p_{h_t}(s_{t}'|s_t,a_t) \pi_{h_t}^E(a_t^{E}|s_t).
\end{align*}
where $H_{t-1} = (s_1,a_1,h_1,s_1',a_1^E,\dots,s_{t-1},a_{t-1},h_{t-1},s_{t-1}',a_{t-1}^E)$ is the history. Let $\imath \in \mathcal{I}$. Let us now move to the KL-divergence between the instances $\pi_0$ and $\pi_\imath$ for some $\imath =(s_*,h_*) \in \mathcal{J}$:
\begin{align*}
	D_{\text{KL}}\left( \mathbb{P}_{(\mathcal{M}_0,\pib),\mathfrak{A}}, \mathbb{P}_{(\mathcal{M}_\imath,\pib),\mathfrak{A}} \right) & = \E_{(\mathcal{M}_0,\pib),\mathfrak{A}} \left[ \sum_{t=1}^\tau D_{\text{KL}} \left(\pi_{h_t}^0(\cdot|s_t),\pi_{h_t}^\imath(\cdot|s_t) \right) \right] \\
	& \le \E_{(\mathcal{M}_0,\pib),\mathfrak{A}} \left[ N^\tau_{h_*}(s_*) \right] D_{\text{KL}} \left(\pi_{h_*}^0(\cdot|s_*),\pi_{h_*}^\imath(\cdot|s_*) \right) \\
	& \le \log \frac{1}{1-\pi_{\min}} \E_{(\mathcal{M}_0,\pib),\mathfrak{A}} \left[ N^\tau_{h_*}(s_*,a_*) \right].
	%
	%\E_{\mathcal{M}_0} \left[ N_{h_*}^\tau(s_*,a_*) \right] d_{\text{KL}}\left( \frac{1}{2}, \frac{1}{2}+\epsilon' \right) \le.
\end{align*}
having observed that the transition models differ in $\imath = (s_*,h_*)$ and defined $N^\tau_{h_*}(s_*) = \sum_{t=1}^\tau \indic\{(s_t,h_t) = (s_*,h_*)\}$ and the last passage is obtained by explicitly computing the KL-divergence:
\begin{align*}
 D_{\text{KL}} \left(\pi_{h_*}^0(\cdot|s_*),\pi_{h_*}^\imath(\cdot|s_*) \right)= \sum_{a \in \As} \pi_{h_*}^0(a|s_*) \log \left(\frac{\pi_{h_*}^0(a|s_*)}{\pi_{h_*}^\imath(a|s_*)} \right) = \pi_{h_*}^0(a_0|s_*) \log \left(\frac{\pi_{h_*}^0(a_0|s_*)}{\pi_{h_*}^\imath(a_0|s_*)} \right) = \log \frac{1}{1-\pi_{\min}}.
\end{align*}
Putting all together, we have:
\begin{align*}
	\delta \ge \frac{1}{4} \exp \left( - \log \frac{1}{1-\pi_{\min}} \E_{(\mathcal{M}_0,\pib),\mathfrak{A}} \left[ N^\tau_{h_*}(s_*) \right]  \right) \implies \E_{(\mathcal{M}_0,\pib),\mathfrak{A}} \left[  N_{h_*}^\tau(s_*)  \right] \ge \frac{\log \frac{1}{4\delta}}{\log \frac{1}{1-\pi_{\min}}}.
\end{align*}

Thus, summing over $(s_*,a_*) \in \mathcal{J}$, we have:
\begin{align*}
	\E_{(\mathcal{M}_0,\pib),\mathfrak{A}} \left[\tau\right] & \ge \sum_{(s_*,a_*) \in \mathcal{J}}\E_{(\mathcal{M}_0,\pib),\mathfrak{A}} \left[  N_{h_*}^\tau(s_*,a_*)  \right] \\
	& = \sum_{(s_*,a_*,h_*) \in \mathcal{I}} \frac{(H-\overline{H}-2 )^2\log \frac{1}{4\delta}}{2\epsilon^2} \\
	& =  \overline{S}\overline{H}\frac{\log \frac{1}{4\delta}}{\log \frac{1}{1-\pi_{\min}}}.
\end{align*}
The number of states is given by $S = |\Ss| = \overline{S}+3$. Let us first consider the time-homogeneous case, \ie $\overline{H} = 1$:
\begin{align*}
\E_{(\mathcal{M}_0,\pib),\mathfrak{A}} \left[\tau\right] \ge (S-3) \frac{\log \frac{1}{4\delta}}{\log \frac{1}{1-\pi_{\min}}}.
\end{align*}
For $\delta < 1/16$, $S\ge 7$, $A\ge 2$, $H \ge 2$, we obtain:
\begin{align*}
\E_{(\mathcal{M}_0,\pib),\mathfrak{A}} \left[\tau\right] \ge \frac{S}{4\log \frac{1}{1-\pi_{\min}} } \log \frac{1}{\delta}.
\end{align*}
For the time-inhomogeneous case, instead, we select $\overline{H} = H/2$, to get:
\begin{align*}
\E_{(\mathcal{M}_0,\pib),\mathfrak{A}} \left[\tau\right] \ge \frac{(S-3)(H/2)}{\epsilon^2}  \frac{\log \frac{1}{4\delta}}{\log \frac{1}{1-\pi_{\min}}}.
\end{align*}
For $\delta < 1/16$, $S\ge 7$, $A\ge 2$, $H \ge 2$, we obtain:
\begin{align*}
\E_{(\mathcal{M}_0,\pib),\mathfrak{A}} \left[\tau\right] \ge \frac{SH}{8\log \frac{1}{1-\pi_{\min}} } \log \frac{1}{\delta}.
\end{align*}

\end{proof}

\subsection{Algorithm}

\begin{figure}
\small
\noindent\fbox{%
    \parbox{\linewidth}{%
\begin{algorithmic}
   \STATE {\bfseries Input:} significance $\delta\in(0,1)$, $\epsilon$ target accuracy
   \STATE $t\gets 0$, $\epsilon_0\gets +\infty$
   \WHILE{$\epsilon_t>\epsilon$}
   \STATE $t \leftarrow t + SAH$
    \STATE Collect one sample from each $(s,a,h) \in \SAHs$
    \STATE Update $\widehat{p}^{t}$ and $\widehat{\pi}^{E,t}$ according to \eqref{eq:empiricalMDP}
    \STATE Update $\epsilon_{t}=\max_{(s,a,h)\in \mathcal{S}\times\mathcal{A}\times\dsb{H}} \textcolor{vibrantBlue}{\overline{\mathcal{C}}^{t}_h(s,a)}$ (resp. $\textcolor{vibrantRed}{\widetilde{\overline{\mathcal{C}}}^{t}_h(s,a)}$)
   \ENDWHILE
\end{algorithmic}    
%    
%        {\bfseries Input:} significance $\delta\in(0,1)$, $\epsilon$ target accuracy\\
%        
%        Collect one sample from each $(s,a,h) \in \SAHs$ to obtain dataset $D_{t}$\\
%        Update transition model $\widehat{p}^{t}$ and 
%        NBB\\
%        \textbf{Stopping rule}: $\displaystyle \tau = \inf\left\{ t \in \mathbb{N}\,:\, \max_{(s,a,h) \in \SAHs} \mathcal{C}_{t}(s,a,h) \le \epsilon  \right\}$
    }%
}
\captionof{algorithm}{UniformSampling-IRL (\texttt{US-IRL}) for \textcolor{vibrantBlue}{time-inhomogeneous} (resp. \textcolor{vibrantRed}{time-homogeneous}) transition models and expert's policies.}\label{alg:us2}
\end{figure}
%
%\begin{figure}
%\small
%\noindent\fbox{%
%    \parbox{\linewidth}{%
%\begin{algorithmic}
%   \STATE {\bfseries Input:} significance $\delta\in(0,1)$, $\epsilon$ target accuracy
%   \STATE $t\gets 0$, $\epsilon_0\gets +\infty$
%   \WHILE{$\epsilon_t>\epsilon$}
%   \STATE $t \leftarrow t + SAH$
%    \STATE Collect one sample from each $(s,a,h) \in \mathcal{S}\times\mathcal{A}\times\dsb{H}$
%    \STATE Update $\widehat{p}^{t}$ and $\widehat{\pi}^{E,t}$ according to \eqref{eq:empiricalMDP}
%    \STATE Update $\epsilon_{t}=\max_{(s,a,h)\in \mathcal{S}\times\mathcal{A}\times\dsb{H}}\mathcal{\overline{C}}^{t}_h(s,a)$
%   \ENDWHILE
%\end{algorithmic}    
%%    
%%        {\bfseries Input:} significance $\delta\in(0,1)$, $\epsilon$ target accuracy\\
%%        
%%        Collect one sample from each $(s,a,h) \in \SAHs$ to obtain dataset $D_{t}$\\
%%        Update transition model $\widehat{p}^{t}$ and 
%%        NBB\\
%%        \textbf{Stopping rule}: $\displaystyle \tau = \inf\left\{ t \in \mathbb{N}\,:\, \max_{(s,a,h) \in \SAHs} \mathcal{C}_{t}(s,a,h) \le \epsilon  \right\}$
%    }%
%}
%\captionof{algorithm}{UniformSampling-IRL (\texttt{US-IRL}) for time-inhomogeneous transition models and expert's policies.}\label{alg:us2}
%\end{figure}

In this appendix, we extend \texttt{US-IRL} to the expert's policy estimation under Assumption~\ref{ass:policy}. The pseudocode is reported in Algorithm~\ref{alg:us2}. The interaction protocol follows the same principles of Algorithm~\ref{alg:UniformSampling-IRL}, with the only difference that the confidence function, now, must account for the policy estimation, leading to the following function for every $(s,a,h) \in \SAHs$:\footnote{As for the transition model, one can adapt the confidence function for the case of stationary policy in straightforward way:
\begin{align}
\textcolor{vibrantRed}{\widetilde{\overline{C}}_h^t(s,a) }\coloneqq 2(H-h+1) \left(\indic_{\left\{ n_h^t(s) \ge \max\left\{1, \widetilde{\xi}(n^t(s), \delta/2)\right\}\right\}} +  \sqrt{\frac{2\widetilde{\beta}\bigl(n^t(s,a),\delta/2\bigl)}{n^t(s,a)}} \right),
\end{align}
where:
\begin{align*}
\widetilde{\xi}(n,\delta) \coloneqq \frac{\log (2 SA n^2/\delta)}{\log(1/(1-\pi_{\min}))}.
\end{align*}
In principle, one can also consider the case of a time-homogeneous transition model and time-inhomogeneous expert's policy. We omit it because it adds nothing to the characteristics of the problem and of the algorithms.
}
\begin{align}\label{eq:algCI2}
\textcolor{vibrantBlue}{\overline{C}_h^t(s,a) }\coloneqq 2(H-h+1) \left(\indic_{\left\{ n_h^t(s) \ge \max\left\{1, \xi(n_h^t(s), \delta/2)\right\}\right\}} +  \sqrt{\frac{2\beta\bigl(n_h^t(s,a),\delta/2\bigl)}{n_h^t(s,a)}} \right).
\end{align}
where:
\begin{align*}
\xi(n,\delta) \coloneqq \frac{\log (2 SAH n^2/\delta)}{\log(1/(1-\pi_{\min}))} .
\end{align*} 
It is worth noting that we have distributed the confidence $\delta$ equally between the problem estimating the policy and that of estimating the transition model. The following theorem provides the sample complexity of \texttt{US-IRL}.

\begin{restatable}[Sample Complexity of \texttt{US-IRL}]{thr}{scUSIRL}\label{thr:scUSIRL}
Let $\epsilon >0$ and $\delta \in (0,1)$, under Assumption~\ref{ass:policy}, \texttt{US-IRL} is $(\epsilon,\delta)$-PAC for $d^{\text{G}}$-IRL and with probability at least $1-\delta$ it stops after $\tau$ samples with:
\begin{itemize}[topsep=0pt, noitemsep, leftmargin=*]
	\item if the transition model $p$ and the expert's policy $\pi^E$ are time-inhomogeneous:
	\begin{align*}
		\tau & \le \frac{8 H^3 SA}{\epsilon^2}  \left( \log \left( \frac{SAH}{\delta} \right) + (S-1) \overline{C}_1 \right) + SH + \frac{SH}{\log(1/(1-\pi_{\min}))} \left(\log \left(\frac{4 SAH}{\delta} \right)+ \overline{C}_2 \right),
	\end{align*}
	where $\overline{C}_1 = \log ( e/(S-1)+(8 e H^2)/((S-1)\epsilon^2) (  \log (2SAH/\delta) +  4e) )$ and $\overline{C}_2 = 2 \log \left( \frac{\log (4 SAH/\delta) + 2 }{\log(1/(1-\pi_{\min}))}  \right)$.
	\item if the transition model $p$ and the expert's policy $\pi^E$ are time-homogeneous:
	\begin{align*}
		\tau & \le \frac{8 H^2 SA}{\epsilon^2}  \left( \log \left( \frac{SA}{\delta} \right) + (S-1) \overline{C}_1 \right) + SH + \frac{S}{\log(1/(1-\pi_{\min}))} \left(\log \left(\frac{4 SA}{\delta} \right)+ \overline{C}_2 \right),
	\end{align*}
	where $\widetilde{\overline{C}}_1 = \log ( e/(S-1)+(8 e H^2)/((S-1)\epsilon^2) (  \log (2SA/\delta) +  4e) )$ and $\widetilde{\overline{C}}_2 = 2 \log \left( \frac{\log (4 SA/\delta) + 2 }{\log(1/(1-\pi_{\min}))}  \right)$.
\end{itemize}
	
\end{restatable}

Before moving to the proof, let us observe that the result matches the rate of the lower bound of Theorem~\ref{thr:rInfLBpol} up to logarithmic terms.

\begin{proof}
We make use of the notation of the proof of Theorem~\ref{thr:scUSIRL}.
We start with the case in which the transition model is time-inhomogeneous. In addition to the good event $\mathcal{E}$ related to the transition model, we introduce the following one:
\begin{align*}
\mathcal{E}_{\pi} \coloneqq \left\{\forall t \in \Nat,\,\forall (s,a,h) \in \SAHs : \left| \indic_{\pi^E_h(a|s)=0} - \indic_{\widehat{\pi}^{E,t}_h(a|s)=0} \right| \le \indic_{\left\{ n_h^t(s) \ge \max\left\{1, \xi(n_h^t(s), \delta/2)\right\}\right\}}\right\}   , 
\end{align*}
where $\pi_h^E$ is the true expert's policy and $\widehat{\pi}^{E,t}$ is its estimate via Equation~\eqref{eq:empiricalMDP} at time $t$. Thanks to Lemma~\ref{lemma:bound_good_event} and Lemma~\ref{lemma:bound_good_event2}, we have that $\Prob(\mathcal{E} \cap \mathcal{E}_{\pi}) \ge 1-\delta$. Thus, under the good event $\mathcal{E} \cap \mathcal{E}_{\pi}$, we apply Theorem~\ref{thr:errorProp2} to obtain  $\mathcal{H}_{d^{\text{G}}}(\mathcal{R},\widehat{\mathcal{R}}^\tau)  \le \max_{(s,a,h)\in \SAHs} \overline{\mathcal{C}}_h^t(s,a)$. A sufficient condition to make this term $\le \epsilon$ is to request the following ones:
\begin{align*}
& \max_{(s,a,h)\in \SAHs} 2(H-h+1) \indic_{\left\{ n_h^t(s) \ge \max\left\{1, \xi(n_h^t(s), \delta/2)\right\}\right\}} = 0,\\
& \max_{(s,a,h)\in \SAHs}2\sqrt{2}(H-h+1)  \sqrt{\frac{\beta\bigl(n_h^t(s,a),\delta/2\bigl)}{n_h^t(s,a)}} \le \epsilon.
\end{align*}
For the first one, we first enforce the condition:
\begin{align*}
	n^{t}_h(s) \ge \xi(n_h^t(s),\delta/2) =  \frac{\log (4 SAH (n^{t}_h)^2/\delta)}{\log(1/(1-\pi_{\min}))} = \frac{\log (4 SAH/\delta)}{\log(1/(1-\pi_{\min}))} + \frac{2\log  n^{t}_h}{\log(1/(1-\pi_{\min}))} .
\end{align*}
Using Lemma 15 of~\cite{KaufmannMDJLV21} and enforcing $n^{t}_h(s) \ge 1$, we obtain:
\begin{align*}
	n^{\tau}_h(s) \le 1 + \frac{1}{\log(1/(1-\pi_{\min}))} \left(\log (4 SAH/\delta) + 2 \log \left( \frac{\log (4 SAH/\delta) + 2 }{\log(1/(1-\pi_{\min}))}  \right) \right).
\end{align*}
Combining this result with that of Theorem~\ref{thr:scUSIRL} for what concerns the transition model, we obtain:
\begin{align*}
	\tau &\le \frac{8 H^3 SA}{\epsilon^2}  \left(\log \left(\frac{2SAH}{\delta}\right) + (S-1) \log \left( \frac{8 e H^2}{(S-1)\epsilon^2} \left(  \log \left(\frac{2SAH}{\delta}\right) +  4e\right) \right)  \right) \\
	& \quad + SH + \frac{SH}{\log(1/(1-\pi_{\min}))} \left(\log (4 SAH/\delta) + 2 \log \left( \frac{\log (4 SAH/\delta) + 2 }{\log(1/(1-\pi_{\min}))}  \right) \right).
\end{align*}
Analogous derivations can be carried out for the case of time-homogenous policy using the good event:
\begin{align*}
\widetilde{\mathcal{E}}_{\pi} \coloneqq \left\{\forall t \in \Nat,\,\forall (s,a) \in \SAs : \left| \indic_{\pi^E(a|s)=0} - \indic_{\widehat{\pi}^{E,t}(a|s)=0} \right| \le \indic_{\left\{ n^t(s) \ge \max\left\{1, \widetilde{\xi}(n^t(s), \delta/2)\right\}\right\}}\right\}, 
\end{align*}
where $\widetilde{\xi}(n,\delta) \coloneqq  \frac{\log (2 SA n^2/\delta)}{\log(1/(1-\pi_{\min}))}$. We omit the tedious but straightforward derivation.
\end{proof}

\begin{lemma}\label{lemma:bound_good_event2}
Under Assumption~\ref{ass:policy}, the following statements hold:
\begin{itemize}
	\item for $\xi(n,\delta) \coloneqq \frac{\log (2SAH n^2/\delta)}{\log(1/(1-\pi_{\min}))}$, we have that $\mathbb{P}(\mathcal{E}_{\pi})\ge 1-\delta$;
	\item for $\widetilde{\xi}(n,\delta) \coloneqq \frac{\log (2SA n^2/\delta)}{\log(1/(1-\pi_{\min}))}$, we have that $\mathbb{P}(\widetilde{\mathcal{E}}_{\pi})\ge 1-\delta$.
\end{itemize}
\end{lemma}

\begin{proof}
%\citep[][Proposition 1]{JonssonKMDLV20}\label{lemma:jonsson}
Let us start with the first statement.
We apply first a union bound and, then, Lemma~\ref{lemma:conc1} to perform the concentration:
\begin{align*}
\mathbb{P}(\mathcal{E}^c_{\pi})& = \mathbb{P}\Biggl(\exists t \in \Nat,\,\exists (s,a,h) \in \SAHs :\,\left| \indic_{\pi^E_h(a|s)=0} - \indic_{\widehat{\pi}^{E,t}_h(a|s)=0} \right| \le \indic_{\left\{ n_h^t(s) > \max\left\{1, \xi(n_h^t(s), \delta)\right\}\right\}}\Biggl)\\
& = \mathbb{P}\Biggl(\exists n \in \Nat,\,\exists (s,a,h) \in \SAHs :\,\left| \indic_{\pi^E_h(a|s)=0} - \indic_{\widehat{\pi}^{E,[n]}_h(a|s)=0} \right| > \indic_{\left\{ n \ge \max\left\{1, \xi(n, \delta)\right\}\right\}}\Biggl)\\
& \le \sum_{h \in \dsb{H}} \sum_{(s,a) \in \SAs} \sum_{n \ge 0} \mathbb{P}\Biggl( \left| \indic_{\pi^E_h(a|s)=0} - \indic_{\widehat{\pi}^{E,[n]}_h(a|s)=0} \right| \le \indic_{\left\{ n > \max\left\{1, \xi(n, \delta)\right\}\right\}}\Biggl)\\
& \le \sum_{h \in \dsb{H}} \sum_{(s,a) \in \SAs} \sum_{n \ge 1} \mathbb{P}\Biggl( \left| \indic_{\pi^E_h(a|s)=0} - \indic_{\widehat{\pi}^{E,[n]}_h(a|s)=0} \right| \le \indic_{\left\{ n > \max\left\{1, \xi(n, \delta)\right\}\right\}}\Biggl) \\
& \le \sum_{h \in \dsb{H}} \sum_{(s,a) \in \SAs} \frac{\delta}{2SAHn^2} = \frac{\pi^2}{6} \frac{\delta}{2} \le \delta.
\end{align*}
where on the first passage we enforced the condition on the time instants in which the policy estimate changes (\ie when $(s,h)$ is visited) and we denotes such an estimate as $\widehat{\pi}^{E,[n]}_h$. Then, after a union bound, we apply Lemma~\ref{lemma:conc1}.
The proof of the second statement is analogous having simply observed that the union bound has to be performed over $\SAs$ only.
\end{proof}

%%%%%%%%%%%%%%%%%%%%%%%%%%%%%%%%%%%%%%%%%%%%%%%%%%%%%%%%%%%%%%%%%%%%%%%%%%%%%%%
%%%%%%%%%%%%%%%%%%%%%%%%%%%%%%%%%%%%%%%%%%%%%%%%%%%%%%%%%%%%%%%%%%%%%%%%%%%%%%%
\section{Technical Lemmas}
\begin{thr}(\emph{Bretagnolle-Huber inequality}~\citep[][Theorem 14.2]{lattimore2020bandit})\label{thr:BHIneq}
Let $\mathbb{P}$ and $\mathbb{Q}$ be probability measures on the same measurable space $(\Omega, \mathcal{F})$, and let $\mathcal{A} \in \mathcal{F}$ be an arbitrary event. Then,
\begin{align*}
	\mathbb{P}(\mathcal{A}) + \mathbb{Q}(\mathcal{A}^c) \ge \frac{1}{2} \exp \left( - D_{\text{KL}}(\mathbb{P},\mathbb{Q}) \right),
\end{align*}
where $\mathcal{A}^c = \Omega \setminus \mathcal{A}$ is the complement of $\mathcal{A}$.
\end{thr}

\begin{thr}(\emph{Fano inequality}~\citep[][Proposition 4]{gerchinovitz2020fano})\label{thr:Fano}
Let $\mathbb{P}_0,\mathbb{P}_1,\dots,\mathbb{P}_M$ be probability measures on the same measurable space $(\Omega, \mathcal{F})$, and let $\mathcal{A}_1,\dots,\mathcal{A}_M \in \mathcal{F}$ be a partition of $\Omega$. Then,
\begin{align*}
	\frac{1}{M} \sum_{i=1}^M \mathbb{P}_i(\mathcal{A}_i^c) \ge 1 - \frac{ \frac{1}{M} \sum_{i=1}^M D_{\text{KL}}(\mathbb{P}_i,\mathbb{P}_0)  - \log 2}{\log M},
\end{align*}
where $\mathcal{A}^c = \Omega \setminus \mathcal{A}$ is the complement of $\mathcal{A}$.
\end{thr}

\begin{lemma}\citep[][Proposition 1]{JonssonKMDLV20}\label{lemma:jonsson}
Let $\mathbb{P} = (p_1,\dots,p_D)$ be a categorical probability measure on the support $\dsb{D}$. Let$\mathbb{P}_n = (\widehat{p}_1,\dots,\widehat{p}_D)$ be the maximum likelihood estimate of $\mathbb{P}$  obtained with $n \ge 1$ independent samples. Then, for every $\delta \in (0,1)$ it holds that:
\begin{align*}
	\Prob \left( \exists n \ge 1 \,:\, n D_{\text{KL}}\left(\mathbb{P}_n, \mathbb{P}\right) > \log(1/\delta) + (D-1) \log \left( e(1+n/(D-1)) \right)  \right) \le \delta
\end{align*}
\end{lemma}

\begin{lemma}\label{lemma:KLBound}
Let $\epsilon \in [0,1/2]$ and $\mathbf{v} \in \{-\epsilon,\epsilon\}^D$ such that $\sum_{i=1}^d v_i = 0$. Consider the two categorical distributions  $\mathbb{P} = \left( \frac{1}{D},\frac{1}{D},\dots, \frac{1}{D}\right)$ and $\mathbb{P} = \left( \frac{1+v_1}{D},\frac{1+v_2}{D},\dots, \frac{1+v_D}{D}\right)$. Then, it holds that:
\begin{align*}
 D_{\text{KL}}(\mathbb{P},\mathbb{Q}) \le 2 \epsilon^2 \qquad \text{and} \qquad 
 D_{\text{KL}}(\mathbb{Q},\mathbb{P}) \le 2 \epsilon^2.
\end{align*}
\end{lemma}

\begin{proof}
First of all we recall that since $\sum_{i=1}^M v_i = 0$, we have $|\{i \in \dsb{D}: v_i=\epsilon\}|=|i \in \dsb{D}: v_i=-\epsilon|=D/2$. Let us compute the KL-divergence $D_{\text{KL}}(\mathbb{P},\mathbb{Q})$:
\begin{align*}
D_{\text{KL}}(\mathbb{P},\mathbb{Q}) & = \sum_{i=1}^D \frac{1+v_i}{D} \log \frac{\frac{1+v_i}{D}}{\frac{1}{D}} \\
& = \sum_{i\in \dsb{D}:v_i=\epsilon}  \frac{1+\epsilon}{D} \log (1+\epsilon) + \sum_{i\in \dsb{D}:v_i=-\epsilon}  \frac{1-\epsilon}{D} \log (1-\epsilon) \\
& = \frac{1+\epsilon}{2} \log (1+\epsilon) + \frac{1-\epsilon}{2} \log (1-\epsilon) \\
& = \underbrace{\frac{1}{2} \log (1-\epsilon^2)}_{\le 0} + \frac{\epsilon}{2} \log (1+\epsilon) - \frac{\epsilon}{2} \log (1-\epsilon) \\
& \le \frac{\epsilon^2}{2} + \frac{\epsilon}{2} \left(\frac{1}{1-\epsilon} - 1\right) = \epsilon^2 \frac{2-\epsilon}{2(1-\epsilon)} \le \frac{3}{2} \epsilon^2 \le 2 \epsilon^2.
\end{align*}
where we used the inequality $\log(1+x) \le x$ for $x \ge 0$ and $-\log(1-x) \le \frac{1}{1-x} - 1$ for $0<x<1$ and exploited that  $\epsilon \le \frac{1}{{2}}$.
Let us now move to the second KL-divergence $D_{\text{KL}}(\mathbb{Q},\mathbb{P}) $:
\begin{align*}
D_{\text{KL}}(\mathbb{Q},\mathbb{P}) & = \sum_{i=1}^D \frac{1}{D} \log \frac{\frac{1}{D}}{\frac{1+v_i}{D}} \\
& = \sum_{i\in \dsb{D}:v_i=\epsilon}  \frac{1}{D} \log \frac{1}{1+\epsilon} + \sum_{i\in \dsb{D}:v_i=-\epsilon}  \frac{1}{D} \log \frac{1}{1-\epsilon} \\
& = - \frac{1}{2} \log (1-\epsilon^2) \\
& \le \frac{1}{2} \left( \frac{1}{1-\epsilon^2} - 1\right) = \frac{\epsilon^2}{2(1-\epsilon^2)} \le \frac{2}{3}\epsilon^2 \le 2 \epsilon^2,
\end{align*}
where we used the inequality $-\log(1-x) \le \frac{1}{1-x} - 1$ for $0<x<1$ and observed that $\epsilon \le \frac{1}{{2}}$.
\end{proof}

\begin{lemma}\label{lemma:conc1}
	Let $\mathbb{P} = (p_1,\dots,p_D)$ be a categorical probability measure on the support $\dsb{D}$. Let$\mathbb{P}_n = (\widehat{p}_1,\dots,\widehat{p}_D)$ be the maximum likelihood estimate of $\mathbb{P}$  obtained with $n \ge 1$ independent samples. Then, if $p_i \in \{0\} \cup [p_{\min},1]$ for some $p_{\min} \in (0,1]$. Then, for every $i \in \dsb{D}$ individually, for every $\delta \in (0,1)$, it holds that:
	\begin{align*}
		\left| \indic_{\{ p_i = 0\}} -  \indic_{\{ \widehat{p}_i = 0\}}\right|	\le \indic_{\left\{ n \ge \max\left\{1, \frac{\log \left(\frac{1}{\delta} \right)}{\log\left( \frac{1}{1-p_{min}} \right)} \right\} \right\}}.
	\end{align*}	 
\end{lemma}

\begin{proof}
	Let $i \in \dsb{D}$ such that $p_i > 0$ and, thus, $\indic_{\{ p_i = 0\}} = 0$. By assumption, it must be that $p_i \ge p_{\min}$. To make a mistake, we must have that $\indic_{\{ \widehat{p}_i = 0\}} = 1$, and, thus, $\widehat{p}_i = 0$. Thus, we compute the probability that no sample $i$ is observed among the $n$ ones:
	\begin{align*}
		\mathbb{P}\left( \bigcap_{j \in \dsb{n}} X_j \neq i \right) = \prod_{j \in \dsb{n}} \mathbb{P} \left( X_j \neq i \right) = \mathbb{P} \left( X_1 \neq i \right)^n = (1-p_i)^n \le (1-p_{\min})^n,
\end{align*}	 
where we exploited the fact that the random variables $X_j$ are i.i.d.. If $n=0$ the latter expression is $1$. If, instead, $n \ge 1$, by setting the last expression equal to $\delta$, we get:
\begin{align*}
	(1-p_{\min})^n \le \delta \implies n \ge \frac{\log \left(\frac{1}{\delta} \right)}{\log \left(\frac{1}{1-p_{min}} \right)}.
\end{align*}
The result follows.
\end{proof}

\begin{lemma}\label{lemma:packing}
	Let $\Vs = \{v \in \{-1,1\}^D : \sum_{j=1}^D v_j = 0\}$. Then, the $\frac{D}{16}$-packing number of $\mathcal{V}$ \wrt the metric $d(v,v')=\sum_{j=1}^D |v_j-v'_j|$ is lower bounded by $2^\frac{D}{5}$.
\end{lemma}
	
\begin{proof}
	Let us denote the packing number with $M(\epsilon;\mathcal{V},d)$ and the covering number with $N(\epsilon;\mathcal{V},d)$. It is well known that $N(\epsilon;\mathcal{V},d) \le M(\epsilon;\mathcal{V},d)$~\cite{gyorfi2022adistribution}. Thus, a lower bound to the covering number is a lower bound to the packing number. Let us consider the (pseudo)metric $d'(v,v') = \sum_{j=1}^{D/2} |v_j-v'_j|$ that considers the first half of the components only. Clearly, we have that $d'(v,v') \le d(v,v')$. Therefore, any $\epsilon$-cover \wrt $d(v,v')$ is an  $\epsilon$-cover \wrt $d'(v,v')$ and, consequently, $N(\epsilon;\mathcal{V},d') \le N(\epsilon;\mathcal{V},d)$. Since the (pseudo)metric $d'$ considers only the first half of the components, constructing an $\epsilon$-cover of $\mathcal{V}$ \wrt $d'$ is equivalent to constructing an  $\epsilon$-cover of $\mathcal{V}'$ \wrt $d'$, where $\mathcal{V}' = \{-1,1 \}^{D/2}$. $\mathcal{V}'$ considers the first half of the components of vectors of $\mathcal{V}$, that can be freely chosen, disregarding the summation constraint.\footnote{From an algebraic perspective, $\mathcal{V}'$ can be considered the quotient set obtained from $\mathcal{V}$ by means of the equivalence relation $v \sim v' \iff v_{j} = v_{j'} $ for all $j \in \dsb{D/2}$.} Thus, $N(\epsilon;\mathcal{V},d') = N(\epsilon;\mathcal{V}',d')$. Notice that $d'$ is now a proper metric on $\mathcal{V}'=\{-1,1\}^{D/2}$. Now, we reduce the problem to constructing cover on the Hamming space $\mathcal{H} = \{0,1\}^{D/2}$. Indeed, we can always map an $(\epsilon/2)$-cover for the Hamming space $\mathcal{H}$ to an $\epsilon$-cover for the space $\mathcal{V}'$. Specifically, let $(h_l)_{l}$ an $(\epsilon/2)$-cover for the Hamming space, we construct $(v'_l)_l$ by applying the following transformation: 
\begin{align*}
v'_{j} = \begin{cases} 
		-1 & \text{if } h_{j} = 0 \\
		1 & \text{if } h_{j} = 1 
	\end{cases},
\end{align*}
or, in more convenient way, $v' = 2h-1$. Let $v' \in \mathcal{V}'$:
\begin{align*}
	\min_{l} d'(v',v'_{l,j}) = \min_{l} \sum_{j=1}^{D/2} |v'_j -  v'_{l,j}| = 2 \min_{l}  \sum_{j=1}^{D/2} |h'_j -  h'_{l,j}| \le  \epsilon.
\end{align*}
The covering number of a Hamming space has been lower bounded in~\cite{cohen1985good} for $\epsilon \in \dsb{D/2}$ as:
\begin{align*}
	\log_2 N(\epsilon;\mathcal{H},d') \ge \frac{D}{2} - \log_{2} \sum_{k=0}^{\epsilon} \binom{D/2}{k}.
\end{align*}
We take $\epsilon = D/16$, and we use the known bound $\sum_{i=0}^k \binom{n}{k} \le \left(\frac{en}{k}\right)^k$:
\begin{align*}
\sum_{k=0}^{D/16} \binom{D/2}{k} \le (8e)^{D/16}.
\end{align*}
From, which, we get:
\begin{align*}
\log_2 N(\epsilon;\mathcal{H},d') \ge \frac{D}{2} - \log_{2} \sum_{k=0}^{\epsilon} \binom{D/2}{k} \ge \frac{D}{2} - \frac{D}{16} \log_2 (8e) \ge \frac{D}{5}.
\end{align*}

\end{proof}

\end{document}